\documentclass{article}

\usepackage{microtype}
\usepackage{graphicx}
\usepackage{subfigure}
\usepackage{booktabs} % for professional tables

\usepackage{hyperref}

\usepackage[colorinlistoftodos,color=green!30]{todonotes}
\newcommand{\mytexttt}[1]{\texttt{\color{blue}#1}} % Blue color example

\usepackage[accepted]{icml2024}

\usepackage{float}

\usepackage{amsmath}
\usepackage{amssymb}
\usepackage{mathtools}
\usepackage{amsthm}
\usepackage{bm}
\usepackage[capitalize,noabbrev]{cleveref}

\theoremstyle{plain}
\newtheorem{theorem}{Theorem}[section]

\newtheorem{lemma}[theorem]{Lemma}
\newtheorem{corollary}[theorem]{Corollary}
\theoremstyle{definition}

\newtheorem{assumption}[theorem]{Assumption}
\theoremstyle{remark}
\newtheorem{remark}[theorem]{Remark}
\icmltitlerunning{Learning-Rate-Free Stochastic Optimization over Riemannian Manifolds}

\usepackage{tikz}
\usepackage{pgfplots}
\begin{document}

\twocolumn[
\icmltitle{Learning-Rate-Free Stochastic Optimization over Riemannian Manifolds}

% It is OKAY to include author information, even for blind
% submissions: the style file will automatically remove it for you
% unless you've provided the [accepted] option to the icml2023
% package.

% List of affiliations: The first argument should be a (short)
% identifier you will use later to specify author affiliations
% Academic affiliations should list Department, University, City, Region, Country
% Industry affiliations should list Company, City, Region, Country

% You can specify symbols, otherwise they are numbered in order.
% Ideally, you should not use this facility. Affiliations will be numbered
% in order of appearance and this is the preferred way.
\icmlsetsymbol{equal}{}%{*}

\begin{icmlauthorlist}
\icmlauthor{Daniel Dodd}{equal,lancaster}
\icmlauthor{Louis Sharrock}{equal,lancaster}
\icmlauthor{Christopher Nemeth}{lancaster}
%\icmlauthor{}{sch}
%\icmlauthor{}{sch}
%\icmlauthor{}{sch}
\end{icmlauthorlist}

\icmlaffiliation{lancaster}{Department of Mathematics and Statistics, Lancaster University, UK}

\icmlcorrespondingauthor{Daniel Dodd}{d.dodd1@lancaster.ac.uk}
%\icmlcorrespondingauthor{Firstname2 Lastname2}{first2.last2@www.uk}

% You may provide any keywords that you
% find helpful for describing your paper; these are used to populate
% the "keywords" metadata in the PDF but will not be shown in the document
\icmlkeywords{Machine Learning, ICML}

\vskip 0.3in
]

% this must go after the closing bracket ] following \twocolumn[ ...

% This command creates the footnote in the first column
% listing the affiliations and the copyright notice.
% The command takes one argument, which is text to display at the start of the footnote.
% The \icmlEqualContribution command is standard text for equal contribution.
% Remove it (just {}) if you do not need this facility.

%\printAffiliationsAndNotice{}  % leave blank if no need to mention equal contribution %\icmlEqualContribution
\printAffiliationsAndNotice{} % otherwise use the standard text.

\begin{abstract}
In recent years, interest in gradient-based optimization over Riemannian manifolds has surged. However, a significant challenge lies in the reliance on hyperparameters, especially the learning rate, which requires meticulous tuning by practitioners to ensure convergence at a suitable rate. In this work, we introduce innovative learning-rate-free algorithms for stochastic optimization over Riemannian manifolds, eliminating the need for hand-tuning and providing a more robust and user-friendly approach. We establish high probability convergence guarantees that are optimal, up to logarithmic factors, compared to the best-known optimally tuned rate in the deterministic setting. Our approach is validated through numerical experiments, demonstrating competitive performance against learning-rate-dependent algorithms.
\end{abstract}

\section{Introduction}
\label{introduction}
We study Riemannian optimization problems of the form
\begin{equation}
\label{eq:objective}
\min_{x\in\mathcal{M}} f(x),
\end{equation}
where $f$ is a geodesically convex function, and $\mathcal{M}$ is a Riemannian manifold. In recent years, there has been a growing interest within the machine learning community in addressing optimization challenges on such geometric spaces. These problems manifest in diverse applications, including principal component analysis \citep{Edelman1998}, dictionary learning \citep{Sun_2017}, low-rank matrix completion \citep{boumal2011}, tensor factorization \citep{Ishteva2011}, Gaussian mixture models \citep{Hosseini2015} and metric learning \citep{zadeh16}. %, and optimistic likelihood calculation \citep{Nguyen2019}.

{\setlength{\subfigcapskip}{-1.5mm}
\begin{figure}[t]
    \centering
    \subfigure[Learning rate too large.]{\includegraphics
    [width=0.22\textwidth,trim={17cm 10cm 17cm 6cm},clip]{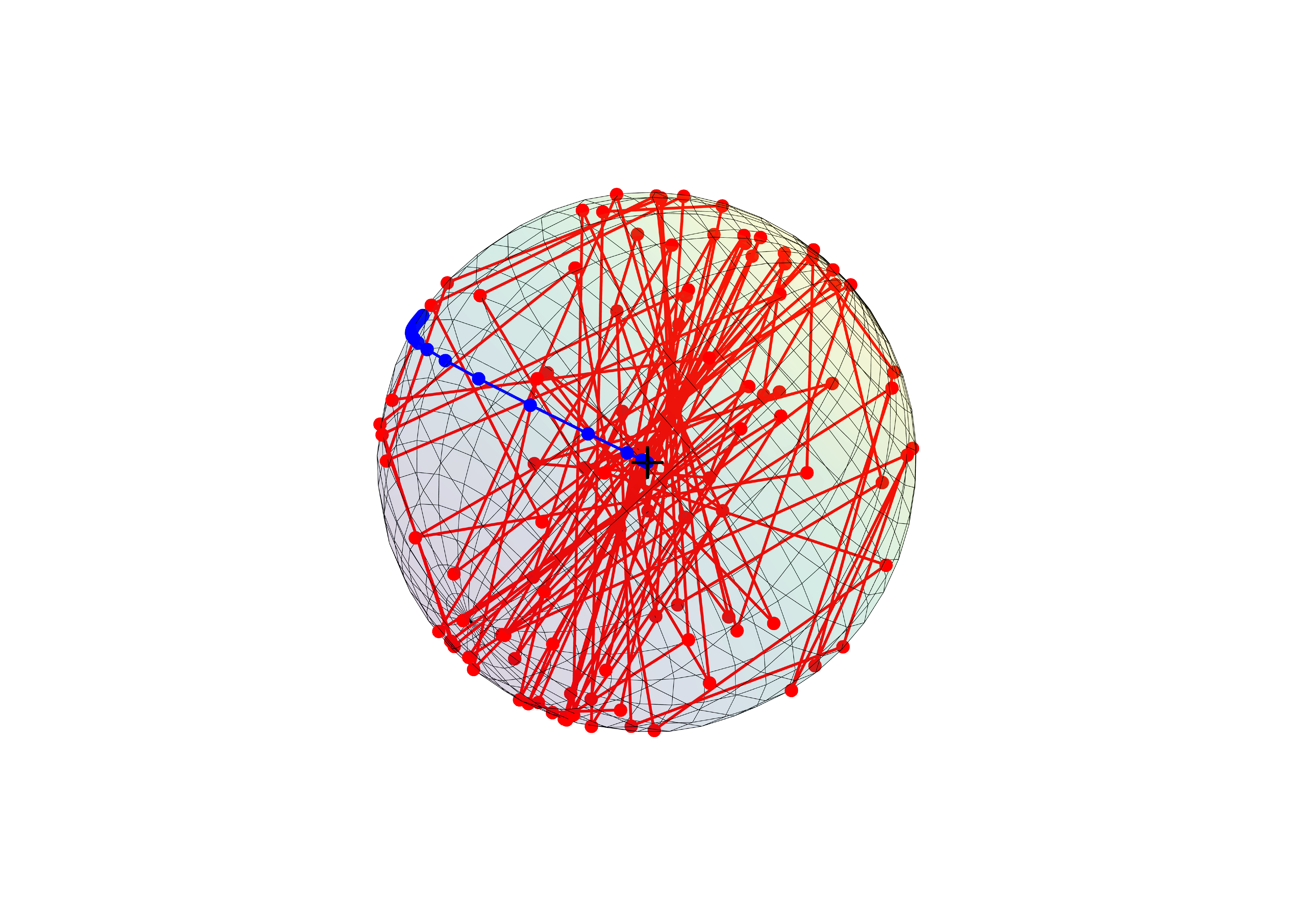}}
    \subfigure[Learning rate too small.]{\includegraphics[width=0.22\textwidth, trim={17cm 10cm 17cm 6cm},clip]{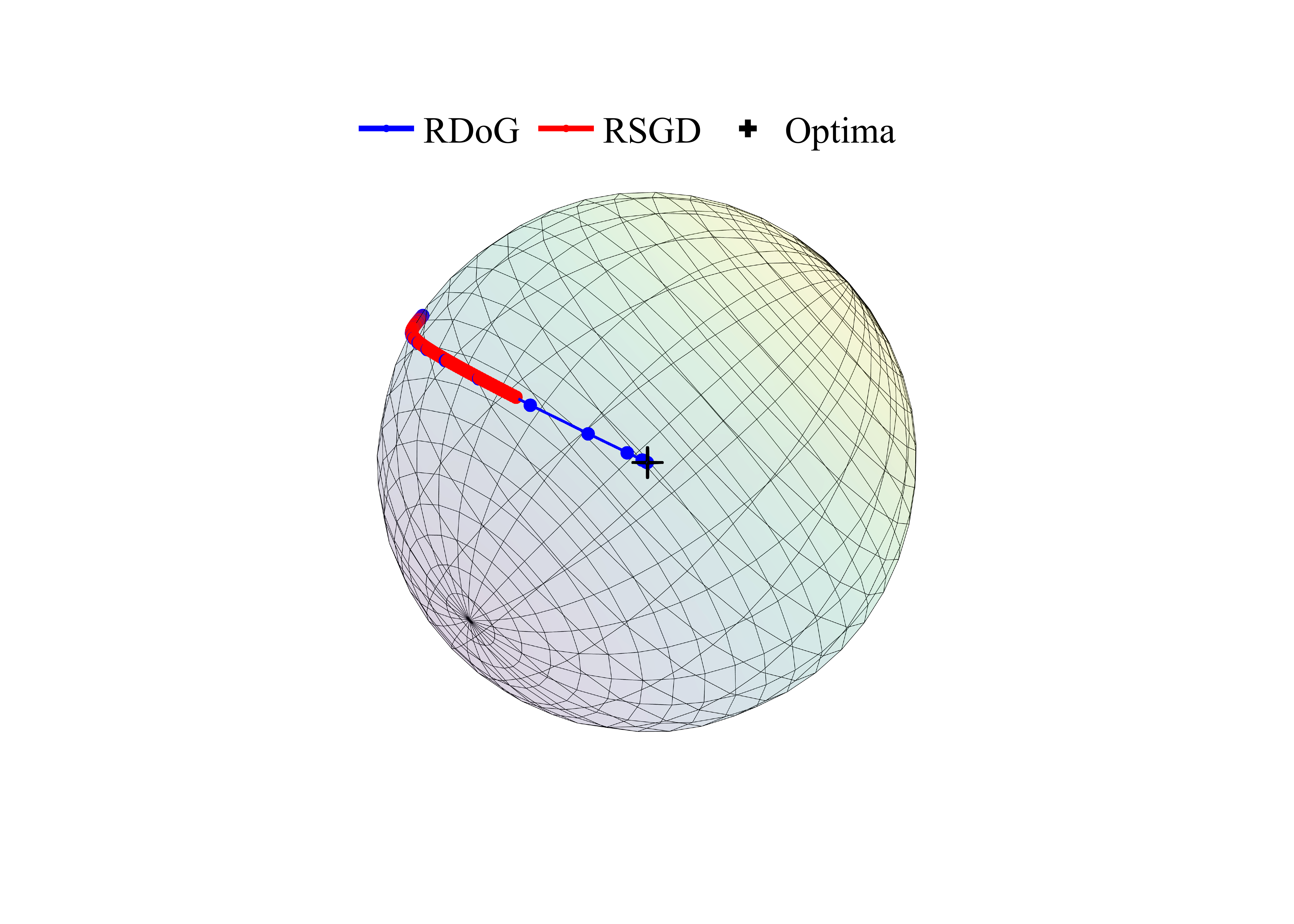}}
    \vspace{-3mm}
    \caption{Rayleigh quotient maximization on the unit sphere. Our algorithm, RDoG, converges without tuning, while RSGD shows sensitivity to the learning rate, leading to (a) overshooting or (b) slow convergence.
    }
    \vspace{-5mm}
    \label{fig:toy_illustration}
\end{figure}

}

One of the prominent hurdles in applying Riemannian gradient-based optimization is the requirement for careful tuning of the learning rate or step size parameter. Selecting an appropriate learning rate is imperative to the algorithm's performance, impacting the convergence rate, final solution quality, and overall algorithm stability. To illustrate, Figure \ref{fig:toy_illustration} showcases the impact of inadequate learning rates on the convergence rate of Riemannian stochastic gradient descent \citep[RSGD, ][]{Bonnabel_2013}.

Recently, the expense and lack of robustness associated with learning rate tuning have spurred substantial research of  \emph{learning-rate-free} methods for Euclidean optimization. These aim to automate tuning by crafting algorithms that achieve near-optimal convergence rates with minimal knowledge of the function's properties and do not have any tuning parameters. Notable examples include online learning schemes like coin betting \citep{orabona2016coin} and exponentiated gradients \cite{mcmahan2014} and bisection subroutines \citep{carmon2022making}. Our paper addresses the absence of comparable tools for Riemannian optimization with the first comprehensive study of learning-rate-free algorithms in this setting.

{\bfseries Contributions}: Building upon the recently proposed Distance over Gradients \citep[DoG,][]{ivgi23a} and Distance over Weighted Gradients \citep[DoWG,][]{khaled2023} Euclidean optimization approaches, we introduce dynamic learning-rate-scheduler algorithms for stochastic Riemannian optimization. Our results establish high probability convergence guarantees, achieving optimal convergence rates with logarithmic factors in smooth and Lipschitz settings, rendering them a robust solution for geodesically convex stochastic optimization on Riemannian manifolds.

\section{Preliminaries}
\label{preliminaries}

\subsection{Riemannian Geometry}
In this section, we recall some fundamental definitions from Riemannian geometry \citep[e.g.,][]{petersen2006, lee2012, boumal23}.

{\bfseries  Riemannian manifold, tangent space, metric}. A \emph{Riemannian manifold} $\mathcal{M}$ is a smooth, locally Euclidean space. At each point $x$ on $\mathcal{M}$, there is a corresponding \emph{tangent space} $\mathcal{T}_x \mathcal{M}$ representing all possible tangential directions, endowed with a smoothly varying inner product $\langle \cdot, \cdot' \rangle_x \colon \mathcal{T}_x \times \mathcal{T}_x \to \mathbb{R}$ termed the \emph{Riemannian metric}, that induces a norm $\lVert \cdot \rVert_x = \sqrt{\langle \cdot, \cdot \rangle_x}$. The metric measures angles, curve lengths, surface areas, and volumes locally, with global quantities obtained by integrating these contributions.

{\bfseries Geodesics and distances.} The length of a curve $c: [0,1] \mapsto c(t) \in \mathcal{M}$ is $L(c) = \int_{0}^{1} \lVert c^\prime(t) \rVert_{c(t)} \mathrm{d}t$. Generalizing straight lines leads to \emph{geodesics}, constant speed curves representing the shortest path between points $x$ and $y$ on the manifold: $\gamma = \arg \min_c L(c)$ with $\gamma(0) = x$, $\gamma(1) = y$, and $\lVert \gamma^\prime(t)\rVert_{\gamma(t)} = 1$, establishing a metric space structure with \emph{geodesic distance} $d(x, y) = \inf_c L(c)$.

{\bfseries Exponential maps}. 
The concept of moving along a
``straight'' curve with constant velocity is given by the \emph{exponential map}. As such, for any point $x$ on $\mathcal{M}$, and any tangent vector $v \in \mathcal{T}_x \mathcal{M}$, there is a unique
unit speed geodesic $\gamma$ satisfying $\gamma(0) = x$ and $\gamma^\prime(0) = v$. The corresponding exponential map $\exp_x \colon \mathcal{T}_x \mathcal{M} \to \mathcal{M}$ is defined as $\exp_x(v) = \gamma(1)$. When $\exp_x$ is well-defined on $\mathcal{T}_x \mathcal{M}$ for all $x \in \mathcal{M}$, the {geodesic distance} $d(x, y)$ is given by $\lVert \exp_x^{-1}(y) \rVert_{x}$.

{\bfseries Parallel transport}. \emph{Parallel transport} $\Gamma_x^y \colon \mathcal{T}_x \mathcal{M} \to \mathcal{T}_y \mathcal{M}$ provides a means to move tangent vectors from one tangent space to another while preserving their norm, and roughly speaking, ``direction,'' analogous to translation in Euclidean space.  

{\bfseries Curvature}. The \emph{curvature} of a Riemannian manifold is determined by its metric at each point. The \emph{sectional curvature} at a point $x$ on the manifold is the Gauss curvature of a two-dimensional submanifold formed as the image of a two-dimensional subspace of the tangent space $\mathcal{T}_x\mathcal{M}$ under the exponential map. %Critically, for worst-case analysis, considering lower bounds on the sectional curvature suffices.

{\bfseries Trigonometric bound}.
The law of cosines in Euclidean space is fundamental for analyzing optimization algorithms,
\begin{align*}
    a^2 = b^2 + c^2 - 2bc \cos(A),
\end{align*}
where $a, b, c$ are the sides of a Euclidean triangle with $A$ the angle between sides $b$ and $c$. Trigonometric geometry behaves differently in manifolds compared to Euclidean spaces. While the equality does not hold for nonlinear spaces, a trigonometric distance bound can be established for manifolds with sectional curvature bounded below.

\begin{lemma} \citep[Lemma 5]{zhang16}
\label{lemma:geodesic_triangle_inq}
    Suppose $a, b, c$ are the side lengths of a geodesic triangle $\Delta$ in a Riemannian manifold with sectional curvature lower bounded by $\kappa>-\infty$ and $A$ is the angle between sides $b$ and $c$ (defined through the inverse exponential map and inner product in tangent space). Then
    \begin{align*}
        a^2 \le \zeta_\kappa(c) b^2 + c^2 - 2bc \cos(A),
    \end{align*}
    where $\zeta_\kappa \colon \mathbb{R}_{+} \to \mathbb{R}$ is the geometric curvature function
        \[
    \zeta_\kappa(d) =
    \begin{cases} 
    \frac{\sqrt{|\kappa|} \cdot d}{\tanh\left({\sqrt{|\kappa|} \cdot d} \right)}, & \text{if } \kappa < 0, \\
    1, & \text{if } \kappa \ge 0.
    \end{cases}
    \]
    \begin{proof}
         Given by Lemma 3.12 of \cite{Cordero2001} and by Lemma 5 of \cite{zhang16}.
    \end{proof}
\end{lemma}

\subsection{Function Classes}
{\bfseries Geodesic convexity}. $\mathcal{M}$ is \emph{geodesically convex} if every two points are connected by a geodesic. A function $f \colon \mathcal{M} \to \mathbb{R}$ is \emph{geodesically convex} if, for any geodesic $\gamma \subset \mathcal{M}$,
\begin{align*}
    f(\gamma(t)) \le (1-t)f(\gamma(0))+tf(\gamma(1)), \quad \forall t \in [0,1].
\end{align*}
Equivalently, $\mathcal{M}$ is geodesically convex if, for any $x, y \in \mathcal{M}$, there exists a tangent vector $\partial f(x)\in\mathcal{T}_{x}\mathcal{M}$ such that
\[ f(y) \ge f(x) + \langle \partial f(x), \exp_x^{-1}(y) \rangle_x, \]
where $\partial f(x)$ is a \emph{Riemannian subgradient of $f$ at $x$}. When $f$ is differentiable, $\{\partial f(x)\} = \operatorname{grad}f(x)$, the \emph{Riemannian gradient of $f$ at $x$}, defined as the tangent vector in $\mathcal{T}_x \mathcal{M}$ satisfying 
\[ \langle \operatorname{grad} f (x), v \rangle_x = df (x)[v], \]
where $df(x) \colon \mathcal{T}_x \mathcal{M} \to \mathbb{R}$ denotes the differential of $f$ at $x$.

{\bfseries Geodesic Lipschitz}. A function $f:\mathcal{M}\rightarrow\mathbb{R}$ is said to be \emph{geodesically $L$-Lipschitz} if, for all $x,y\in\mathcal{M}$, there exists a constant $L > 0$ such that,
\begin{align*}
|f (y) - f (x)| \le L \cdot \lVert \exp_{x}^{-1}(y) \rVert_x.
\end{align*}
When $f$ is differentiable, the geodesically $L$-Lipschitzness is equivalent to $\lVert \operatorname{grad} f(x) \rVert_x \le L$ for all $x \in \mathcal{M}$.

{\bfseries Geodesic smoothness}. A differentiable function $f$ is  \emph{geodesically  $S$-smooth} if its gradient is geodesically $S$-Lipschitz. That is, if for all $x, y \in \mathcal{M}$, 
\begin{align*}
    \lVert \operatorname{grad}f(x) - \Gamma_y^x  \operatorname{grad}f(y) \rVert_x \le S \cdot \lVert \exp_{x}^{-1}(y) \rVert_x, 
\end{align*}
where $\Gamma_{y}^x$ is the parallel transport from $y$ to $x$. One can show in this case that
\begin{align*}
    f(y) \le f(x) + \langle \operatorname{grad}f(x), \exp_x^{-1}(y) \rangle_x + \frac{S}{2} \cdot \lVert \exp_{x}^{-1}(y) \rVert_x^2.
\end{align*}
\section{Algorithms and Theory}
\label{criterion}

%\subsection{Problem Statement}
We are interested in solving optimization problems of the form \eqref{eq:objective}. 
%on a Riemannian manifold. 
We proceed under the following standard regularity conditions \citep{zhang16,alimisis20}. 

\begin{assumption}[Geodesic convexity] \label{assumption:g_convex}
   The geodesically convex function $f \colon \mathcal{M} \to \mathbb{R}$ attains its minimum at $x_\star$ within its closed and geodesically convex domain $\mathcal{M},$ which includes a well-defined exponential map.
\end{assumption}

\begin{assumption}[Lower bounded sectional curvature] \label{assumption:Hadamard}
    $\mathcal{M}$ exhibits sectional curvature bounded from below: $\kappa > -\infty$.
\end{assumption}

To minimize $f$, we will assume access to a \emph{stochastic gradient oracle} $\mathcal{G}$. When queried at $x \in \mathcal{M},$ the oracle returns a stochastic (sub)gradient estimator $\mathcal{G}(x)$ which satisfies $\mathbb{E}\left[\mathcal{G}(x) | x \right] \in \partial f(x)$. In a slight abuse of notation, we will henceforth write $\operatorname{grad} f(x) \coloneqq \mathbb{E}\left[\mathcal{G}(x) | x \right]$. We consider the following additional assumptions.

\begin{assumption}[Locally bounded stochastic gradients] \label{assumption:bounded_stochastic_grads}
   There exists some continuous function $\ell \colon \mathcal{M} \to \mathbb{R}_{+}$ such that $\lVert \mathcal{G}(x) \rVert_{x} \le \ell (x)$ almost surely.
\end{assumption}

\begin{assumption}[Locally smooth stochastic gradients]
\label{assumption:bounded_stochastic_lsmooth}
There exists some continuous function $s \colon \mathcal{M} \to \mathbb{R}_{+}$ such that $\lVert \mathcal{G}(x) - \Gamma_{y}^x \mathcal{G}(y) \rVert_{x} \le s(x) \lVert \exp_x^{-1}(y) \rVert_x$, almost surely.
\end{assumption}

\cref{assumption:bounded_stochastic_grads} corresponds to the Riemannian analog of \citet{ivgi23a}'s locally bounded gradient assumption, whereas \cref{assumption:bounded_stochastic_lsmooth} introduces a novel condition.

\subsection{Riemannian Stochastic Gradient Descent}

Our work centers on the Riemannian stochastic gradient descent algorithm (RSGD) introduced by \citet{Bonnabel_2013}, which from an initial point \(x_0 \in \mathcal{M}\) iterates the following update rule:
\[x_{t+1} = \exp_{x_{t}}(-\eta_t g_t).\]
Here \(t \ge 0\) denotes the iteration index, \(g_t \coloneqq \mathcal{G}(x_t)\) represents the stochastic gradient oracle, and \(\eta_t > 0\) is a user-chosen learning rate or step size parameter.

Our analysis commences by characterizing the ``ideal step size'' in the deterministic gradient setting, an extension of Theorem 9 from \citet{zhang16}.

\begin{theorem}
\label{thm:intractable_oracle}
    Under noiseless conditions and Assumption \ref{assumption:g_convex}, and  \ref{assumption:Hadamard}, RSGD with a constant step size $\eta_t = \eta>0$, for a geodesically $L$-Lipschitz function, satisfies
    \begin{align*}
        \min_{t \le T} \left[f(x_t) - f(x_\star) \right] \le \left[\frac{\bar{d}_T^2}{2 \eta T} + \frac{\eta \zeta_{\kappa}(\bar{d}_T) \sum_{t=0}^{T} \lVert g_t \rVert_{x_t}^2}{2T} \right],
    \end{align*}
    where $\bar{d}_t \coloneqq \max_{s \le t} d_s, d_s = d(x_s,x_{\star})$.
    Minimizing this bound with respect to $\eta$, gives a convergence rate of $O\left(L\bar{d}_T\sqrt{\tfrac{\zeta_\kappa(\bar{d}_T)}{T}}
\right)$  with corresponding ``ideal step size''
    \begin{align*}
        \eta_\star = \frac{\bar{d}_T}{\sqrt{\zeta_{\kappa}(\bar{d}_T) \sum_{t=0}^{T} \lVert g_t \rVert_{x_t}^2}}.
    \end{align*}
\end{theorem}
\textit{Proof.} See \cref{thm:intractable_oracle_proof}. \qed

\subsection{Riemannian Distance Over Gradients}
In practice, determining the ``ideal step size'' $\eta_\star$, even in hindsight, is challenging due to its dependence on the unknown maximum distance $\bar{d}_T$. In this section, we introduce an adaptive algorithm that estimates this whilst attaining the optimal convergence rate up to a logarithmic factor.

\textbf{Learning-rate-free schedule for RSGD}. Our key proposal, inspired by \citet{ivgi23a}, is to estimate $\bar{d}_T$ via a proxy,
\begin{equation*}
    \bar{r}_t \coloneqq \max_{s\leq t} r_s,\quad r_s\coloneqq \max(d(x_0, x_s),\epsilon),
\end{equation*}
where $\epsilon >0$ is an initial estimate. Intuitively, the maximum deviation from the starting point should reflect the maximum deviation from the optimum, assuming the RSGD iterations converge to the optimum. Integrating this estimation into the ``ideal step size,'' we establish an adaptive sequence of step sizes
\begin{align*}
        \eta_t = \frac{\bar{r}_t}{\sqrt{\zeta_\kappa(\bar{r}_t) \sum_{s=0}^{t} \lVert g_s \rVert_{x_s}^2}}.
\end{align*}
We term this step size schedule as \emph{Riemannian Distance over Gradients} (RDoG, \cref{alg:rdog}). Observe that the initial step gives a step size of $\epsilon/\lVert g_0 \rVert_{x_0}$, a normalized gradient step of size $\epsilon$. We demonstrate that, provided $\epsilon$ is chosen sufficiently small, the specific value is insensitive.

\begin{algorithm}[!htbp]
   \caption{RDoG}
   \label{alg:rdog}
\begin{algorithmic}
   \STATE {\bfseries Input:} initial point $x_0$, initial estimate $\epsilon >0$, $G_{-1}=0$.
   \FOR{$t=0$ {\bfseries to} $T-1$}
    %\STATE Play $x_t$ and receive $f$.
    \STATE $g_{t} = \mathcal{G}(x_t)$
    \STATE $\bar{r}_t = \max\left(\epsilon, \max_{s\le t} d(x_{s}, x_{0})\right)$
    \STATE $G_t = G_{t-1} + ||g_t||_{x_t}^2$
    \STATE $\eta_t = \frac{\bar{r}_t} {\sqrt{ \zeta_\kappa( \bar{r}_t) G_t} }$
    \STATE $x_{t+1} = \exp_{x_{t}} \left( - \eta_t g_{t} \right)$
   \ENDFOR
\end{algorithmic}
\end{algorithm}

\textbf{Optimality gap bounds assuming bounded iterates}. We bound the error of the weighted average sequence
\begin{align*}
        \tilde{x}_{t+1} = \operatorname{exp}_{\tilde{x}_{t}}\left(\frac{\bar{r}_t /\sqrt{\zeta_\kappa(\bar{r}_t)}}{\sum_{s=0}^{t} \bar{r}_s /\sqrt{\zeta_\kappa(\bar{r}_s)}} \operatorname{exp}_{\tilde{x}_{t}}^{-1}\left( x_{t} \right) \right), \ \tilde{x}_{1} = x_{0}.
\end{align*}
To simplify our analysis, we write $\log_{+}(\cdot)\coloneqq 1 + \log(\cdot)$ where the logarithm has a base of $e$, and introduce the following quantities
\begin{align*}
    G_t \coloneqq \sum_{s=0}^{t} \lVert g_s \rVert_{x_s}^2, \ \theta_{t, \delta} \coloneqq \log\left(\frac{60 \log (6 t)}{\delta}\right).
\end{align*}
Our first result establishes a bound on the optimality gap under bounded iterates.

\begin{theorem}
    \label{thm:opt_gap_dog_known}
    Suppose that \cref{assumption:g_convex}, \ref{assumption:Hadamard}, and \ref{assumption:bounded_stochastic_grads} hold. Then, for all $\delta \in (0, 1)$ and $L>0$, and for all $t \le T$, RDoG (\cref{alg:rdog}) satisfies the optimality gap $f(\tilde{x}_t) - f(x_\star)$ of
    \begin{align*}
        O \left(  \frac{(d_0 + \bar{r}_t)\sqrt{\zeta_{\kappa}(d_0 + \bar{r}_t)}\sqrt{ G_{t-1} + \theta_{t, \delta} G_{t-1} + \theta^{2}_{t, \delta} L^2}}{\sum_{s=0}^{t-1} \frac{\bar{r}_s /\sqrt{\zeta_\kappa(\bar{r}_s)}}{\bar{r}_t /\sqrt{\zeta_\kappa(\bar{r}_t)}}}\right),
    \end{align*}
     with probability at least $1 - \delta - \mathbb{P}(\bar{\ell}_{T} > L)$,  where $\bar{\ell}_{T} \coloneqq \max_{s \le T} \ell(x_s)$.
\end{theorem}
 \textit{Proof.} See \cref{sec:rdog-non-smooth-analysis}. \qed

This theorem yields a corollary for bounded manifolds.

\begin{corollary}
   Under  \cref{assumption:g_convex}, \ref{assumption:Hadamard}, and \ref{assumption:bounded_stochastic_grads}, for any $D \ge d_0$, let $L_{D} \coloneqq \max_{x\in \mathcal{M} : d(x, x_0) \le D} \ell(x)$. Then, for all $\delta \in (0, 1)$ and for $\tau \in \arg \max_{t \le T} \sum_{s=0}^{t-1} \tfrac{\bar{r}_s /\sqrt{\zeta_\kappa(\bar{r}_s)}}{\bar{r}_t /\sqrt{\zeta_\kappa(\bar{r}_t)}}$, RDoG (\cref{alg:rdog}) satisfies the optimality gap $f(\tilde{x}_\tau) - f(x_\star)$ of
    \begin{align*}
    O\left(\frac{D\sqrt{\zeta_{\kappa}(D)} \sqrt{G_{\tau-1} \theta_{\tau, \delta} + L^2_D \theta_{\tau, \delta}^2}}{T} \log_{+}\left(\frac{D \sqrt{\zeta_{\kappa}(\epsilon)}}{\epsilon\sqrt{\zeta_{\kappa}(D)}}\right)\right),
    \end{align*}
    with probability at least $1 - \delta - \mathbb{P}(\bar{\ell}_{T} > L)$.
 \end{corollary}
 \textit{Proof.}  See \cref{sec:rdog-non-smooth-analysis}. \qed

Unlike prior work on bounded domains \citep[e.g.,][]{zhang16, wang2021}, our approach adapts without knowledge of the domain width to set the learning rate, achieving optimality up to a logarithmic factor.

\begin{remark}
We enhance this result to a high probability convergence guarantee of $O(1/T)$ under uniformly averaged iterates, following \cref{assumption:bounded_stochastic_lsmooth} in \cref{sec:rdog-uniform-averaging}.
\end{remark}

\begin{remark}
In \cref{sec:rdog-stability-analysis}, we ensure bounded iterates with high probability by slightly reducing step sizes.
\end{remark}

\begin{remark}
Omitting the geometric curvature term $\zeta_\kappa(\cdot)$ from RDoG's step sizes and weighted averaging results in an additional cost of $O(\sqrt{\zeta_\kappa(D)})$ in the optimality gap. Further details are available in \cref{sec:theoretical-results-RDoG-no-curvature}.
\end{remark}

\subsection{Normalized Riemannian Stochastic Gradient Descent}
We consider extending standard Euclidean normalized gradient descent \citep{shor2012minimization,levy2016power,konnov2003convergence,hazan2015beyond} to Riemannian manifolds, providing scale-free adaptability, with updates of the form
\begin{align*}
    x_{t+1} = \exp_{x_t}\left(-\eta_t \frac{\operatorname{grad} f(x_t)}{\lVert \operatorname{grad} f(x_t) \rVert_{x_t}}\right),\quad x_0\in\mathcal{M}.
\end{align*}
We term this algorithm \emph{Normalized Riemannian Stochastic Gradient Descent} (NRSGD). In the deterministic Euclidean setting, normalized gradient descent automatically adjusts to the Lipschitz constant in non-smooth optimization \citep[Theorem 3.2.2]{Nesterov18} and the smoothness constant(s) in smooth optimization \citep[Corollary 2.2]{Grimmer2019}. This adaptability extends to NRSGD, as we will demonstrate.

\begin{theorem}
\label{thm:ngd_oracle}
Under noiseless conditions and Assumption \ref{assumption:g_convex}, and \ref{assumption:Hadamard}, NRSGD with a constant step size $\eta_t = \eta>0$, for a geodesically $L$-Lipschitz function, satisfies
    \begin{align*}
        \min_{t \le T}\left[f(x_t) - f(x_\star) \right]\le L \left[\frac{\bar{d}_T^2}{2 \eta T} + \frac{\eta}{2} \zeta_\kappa(\bar{d}_T)\right].
    \end{align*}
While for a geodesically  $S$-smooth function, we have
    \begin{align*}
        \min_{t \le T}\left[f(x_t) - f(x_\star) \right]\le 2 S \left[\frac{\bar{d}_T^2}{2 \eta T} + \frac{\eta}{2} \zeta_\kappa(\bar{d}_T) \right]^{2}.
    \end{align*}
    Minimizing these give respective convergence rates $O\left(L\bar{d}_T\sqrt{\tfrac{\zeta_\kappa(\bar{d}_T)}{T}}
\right)$ and  $O\left(\tfrac{2 S \bar{d}_T^2 \zeta_\kappa(\bar{d}_T)}{T} \right)$ with corresponding ``ideal step size''
    \begin{equation*}
        \eta_\star = \frac{\bar{d}_T}{{\sqrt{T \zeta_{\kappa}(\bar{d}_T)}}}.
    \end{equation*}
%\begin{proof}
%    See \cref{thm:ngd_oracle_proof}.
%\end{proof}
\end{theorem}
\textit{Proof.} See \cref{thm:ngd_oracle_proof}. \qed

\textbf{Learning-rate-free schedule for NRSGD.} Normalization brings adaptivity to Lipschitz and smoothness settings, using a common \emph{universal} ``ideal step size''. However, like RSGD, this ``ideal step size'' relies on the intractable maximum distance quantity $\bar{d}_T$. Our solution is to substitute this with our proxy $\bar{r}_t$, resulting in our second algorithm, \emph{Normalized Riemannian Distance over Gradients} (NRDoG) algorithm, summarized in \cref{alg:nrdog} in \cref{sec:nrdog}.

\subsection{Riemannian Distance Over Weighted Gradients}

\textbf{Weighted learning-rate-free schedule for RSGD.} We introduce a third algorithm \emph{Riemannian Distance over Weighted Gradients} (RDoWG, \cref{alg:rdowg}) that extends the recently proposed Distance over Weighted Gradients (DoWG) \citep{khaled2023} to the Riemannian setting. Like RDoG and NRDoG, RDoWG estimates the intractable maximum distance quantity $\bar{d}_T$ by utilizing the maximum distance deviation from the initial point, $\bar{r}_t$. However, in RDoWG, the normalization is based on the square root of the \emph{weighted} gradient sum, $\smash{v_t = \sum_{s=0}^{t} \bar{r}_s^2 ||g_s||_{x_s}^2}$, rather than simply the square root of the gradient sum $\smash{G_t = \sum_{s=0}^{t} ||g_s||_{x_s}^2}$.

The motivation for this normalization choice, as discussed in \citet{khaled2023}, lies in its improved adaptation to the problem geometry, especially in regions far from the initialization at $x_0$. Specifically, as the distances $\{\bar{r}_t\}_{t\geq 0}$ monotonically increase, later gradients receive greater weights than earlier gradients. This choice aligns with the practice in previous Riemannian optimization schemes such as RADAM \cite{becigneul2018riemannian}, which also utilizes weighted gradient sums. However, unlike RADAM, where weights are determined by fixed user-selected hyperparameters, RDoWG adaptively estimates these weights.

\begin{algorithm}[!htbp]
   \caption{RDoWG}
   \label{alg:rdowg}
\begin{algorithmic}
   \STATE {\bfseries Input:} initial point $x_0$, initial estimate $\epsilon >0$, $v_{-1} =0$.
   \FOR{$t=0$ {\bfseries to} $T-1$}
    %\STATE Play $x_t$ and receive $f$.
    \STATE $g_{t} = \mathcal{G}(x_t)$
    \STATE $\bar{r}_t = \max\left(\epsilon, \max_{s\le t} d(x_{s}, x_{0})\right)$
    \STATE $v_t = v_{t-1} + \bar{r}_t^2\|g_{t}\|_{x_{t}}^{2}$
    %\STATE $v_{t} = v_{t-1} + \frac{\bar{r}_t^2}{\zeta_\kappa( \bar{r}_t)} \|g_{s}\|_{x_{s}}^{2}$
    %\STATE $\eta_t = \frac{\bar{r}_t^2}{\zeta_\kappa( \bar{r}_t) \sqrt{v_{t}}}$
    \STATE $\eta_t = \frac{\bar{r}_t}{\sqrt{\zeta_{\kappa}(\bar{r}_t)v_t}}$
    \STATE $x_{t+1} = \exp_{x_{t}} \left( - \eta_t g_{t} \right)$
   \ENDFOR
\end{algorithmic}
\end{algorithm}

\textbf{Optimality gap bounds assuming bounded iterates}. We bound the error of the weighted average sequence
\begin{align*}
        \tilde{x}_{t+1} = \operatorname{exp}_{\tilde{x}_{t}}\left(\frac{\bar{r}_t^2 / \zeta_\kappa(\bar{r}_t)}{\sum_{s=0}^{t} \bar{r}_s^2 / \zeta_\kappa(\bar{r}_s)} \operatorname{exp}_{\tilde{x}_{t}}^{-1}\left( x_{t} \right) \right), \  \tilde{x}_{1} = x_{0}.
\end{align*}
We initiate our analysis in the non-smooth setting before transitioning to the smooth setting. Our initial result, assuming bounded iterates, provides the optimality gap achieved by RDoWG.
\begin{theorem}
  \label{thm:dowg_nonsmooth_curvature}
  Suppose that  \cref{assumption:g_convex}, \ref{assumption:Hadamard}, and \ref{assumption:bounded_stochastic_grads} hold. Then, for all $\delta \in (0, 1)$ and $L>0$, and for all $t \le T$, RDoWG (\cref{alg:rdowg}) satisfies the optimality gap $f(\tilde{x}_t) - f(x_\star)$ of
    \begin{align*} 
            O\left(\frac{(d_0 + \bar{r}_t)\sqrt{\zeta_{\kappa}(d_0 + \bar{r}_t)} \sqrt{G_{t-1} + \theta_{t, \delta} G_{t-1} + \theta^{2}_{t, \delta} L^2}}{\sum_{s=0}^{t-1} \frac{ \bar{r}_s^2 / \zeta_\kappa(\bar{r}_s) }{\bar{r}_t^2 / \zeta_\kappa(\bar{r}_t)}} \right),
    \end{align*}
    with probability at least $1 - \delta - \mathbb{P}(\bar{\ell}_{T} > L)$.
\end{theorem}
\textit{Proof.} See Appendix \ref{sec:rdowg-non-smooth-analysis} \qed

We obtain a result on bounded domains which is optimal up to a logarithmic factor.
\begin{corollary}
    Suppose that \cref{assumption:g_convex}, \ref{assumption:Hadamard}, and \ref{assumption:bounded_stochastic_grads} hold. In addition, for any $D \ge d_0$, let $L_{D} \coloneqq \max_{x\in \mathcal{M} : d(x, x_0) \le D} \ell(x)$. Then, for all $\delta \in (0, 1)$ and for $\tau \in \arg \max_{t \le T} \sum_{s=0}^{t-1} \tfrac{\bar{r}_s^2/ \zeta_\kappa(\bar{r}_s)}{\bar{r}_t^2/ \zeta_\kappa(\bar{r}_t)}$, RDoWG (\cref{alg:rdowg}) satisfies the optimality gap $f(\tilde{x}_\tau) - f(x_\star)$ of
    \begin{align*}
        O\left( \frac{D\sqrt{\zeta_{\kappa}(D)} \sqrt{G_{\tau-1} \theta_{\tau, \delta} + L^2_D \theta_{\tau, \delta}^2}}{T} \log_{+}\left(\frac{D\sqrt{\zeta_{\kappa}(\epsilon)}}{\epsilon \sqrt{\zeta_{\kappa}(D)}}\right)\right),
    \end{align*}
    with probability at least $1 - \delta - \mathbb{P}(\bar{\ell}_{T} > L)$.
\end{corollary}
\textit{Proof.} See Appendix \ref{sec:rdowg-non-smooth-analysis} \qed

%\textcolor{red}{To do: perhaps add some remarks about how this bound simplifies under slightly stronger assumptions, e.g., bounded diameter, uniform Lipschitz?}

{\setlength{\subfigcapskip}{-1.5mm}
\begin{figure*}[htb]
  \centering
  \subfigure[Initial distance sensitivity.]{\includegraphics[width=0.325\textwidth]{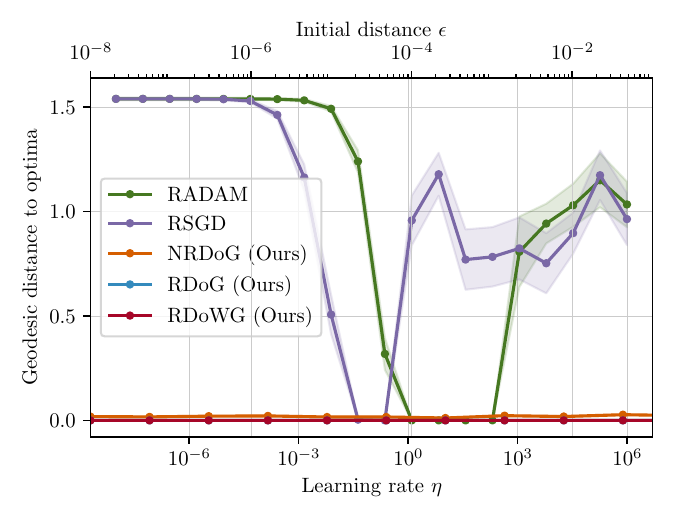}}
  \subfigure[Regret trace plot RSGD.]{\includegraphics[width=0.325\textwidth]{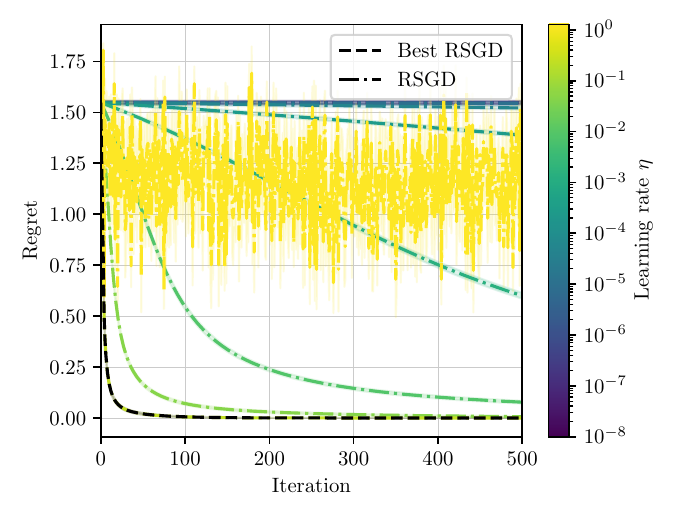}}
  \subfigure[Regret trace plot RDoG.]{\includegraphics[width=0.325\textwidth]{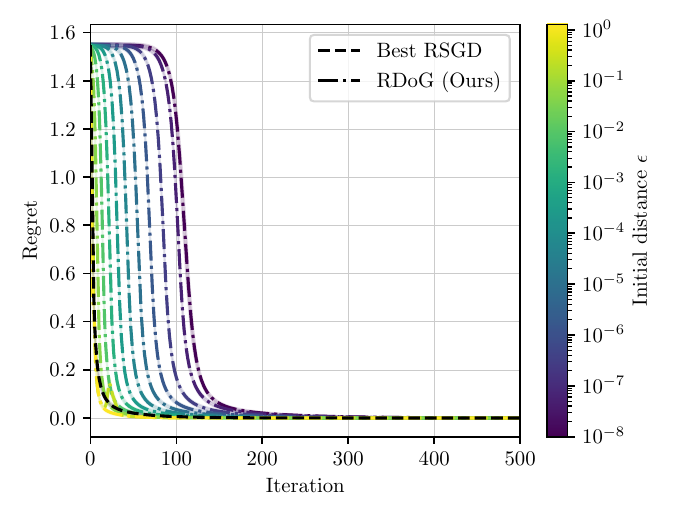}}
  \vspace{-3mm}
  \caption{\textbf{Results for Rayleigh quotient maximization on the sphere.} (a) Geodesic distance between the final iterate and the numerical solution after $T=5000$ iterations as a function of the learning rate for RADAM and RSGD and as a function of the initial distance estimate for RDoG, RDoWG, and NRDoG. (b) Shows the regret (the function value of each iterate minus the function value of the numerical solution) for RSGD for a selection of learning rates. (c) Shows the regret for RDoG for a selection of different initial distance estimates. Results are averaged over ten replications.}

  \label{fig:rayleigh}
  \vspace{-2mm}
\end{figure*}
}

We proceed with analyzing the smooth setting. Our initial result yields an optimality gap for bounded iterates. It is worth noting that RDoG achieves similar results via uniform averaging, albeit with an additional cost (see Appendix \ref{sec:rdog-uniform-averaging} for further details).
\begin{theorem}
    \label{thm:dowg_smooth_known_curvature}
    Suppose that \cref{assumption:g_convex}, \ref{assumption:Hadamard}, and \ref{assumption:bounded_stochastic_lsmooth} hold and write $\bar{s}_{T} \coloneqq \max_{t \le T} s(x_t)$. Then, for all $\delta \in (0, 1)$ and $S>0$, and for all $t \le T$, RDoWG (\cref{alg:rdowg}) satisfies the optimality gap $f(\tilde{x}_t) - f(x_\star)$ of
    \begin{align*}
        O \left(\frac{ (d_0 + \bar{r}_t)^2 \zeta_{\kappa}(d_0 + \bar{r}_t) (S \theta_{t, \delta}^2) }{\sum_{s=0}^{t-1} \frac{ \bar{r}_s^2 / \zeta_\kappa(\bar{r}_s)}{\bar{r}_t^2 / \zeta_\kappa(\bar{r}_t)} } \right),
    \end{align*}
    with probability at least $1 - \delta - \mathbb{P}(\bar{s}_{T} > S)$.
\end{theorem}

\textit{Proof.} See Appendix \ref{sec:rdowg-smooth-analysis}
\qed 

This result achieves the optimal rate, aligning with the smooth analysis of \citet{zhang16}, with an additional logarithmic factor on bounded domains.

\begin{corollary}
    Suppose that \cref{assumption:g_convex}, \ref{assumption:Hadamard}, and \ref{assumption:bounded_stochastic_lsmooth}  hold. In addition, for any $D \ge d_0$, let $S_{D} \coloneqq \max_{x\in \mathcal{M} : d(x, x_0) \le D} s(x)$. Then, for all $\delta \in (0, 1)$ and for $\tau \in \arg \max_{t \le T} \sum_{s=0}^{t-1} \frac{\bar{r}_s^2 / \zeta_\kappa(\bar{r}_s)}{\bar{r}_t^2 / \zeta_\kappa(\bar{r}_t)}$, RDoWG (\cref{alg:rdowg}) satisfies the optimality gap $f(\tilde{x}_{\tau}) - f(x_\star)$ of 
    \begin{align*}
         O\left(\frac{D^2 \zeta_{\kappa}(D) S_D \theta_{\tau, \delta}^2}{T} \log_{+}\left(\frac{D\sqrt{\zeta_{\kappa}(\epsilon)}}{\epsilon \sqrt{\zeta_{\kappa}(D)}}\right)\right),
    \end{align*}
    with probability at least $1 - \delta - \mathbb{P}(\bar{s}_{T} > S)$.
\end{corollary}
\textit{Proof.} See Appendix \ref{sec:rdowg-smooth-analysis} 
\qed

\textbf{Stability analysis}. While RDoWG is generally stable in practice, in theory, the algorithm trajectories can diverge. Drawing inspiration from \citet{ivgi23a}, we now introduce a variant of RDoWG that guarantees iterates remain bounded with high probability. The concept involves using step sizes that are smaller by a polylogarithmic factor. Following the taxonomy introduced in \citet{ivgi23a}, we refer to this scheme as \emph{Tamed Riemannian Distance over Weighted Gradients} (T-RDoWG, \cref{alg:t-rdowg}).

\begin{algorithm}[!htbp]
   \caption{T-RDoWG}
   \label{alg:t-rdowg}
\begin{algorithmic}
   \STATE {\bfseries Input:} initial point $x_0$, initial estimate $\epsilon >0$, $v_{-1} =0$.
   \FOR{$t=0$ {\bfseries to} $T-1$}
    \STATE $g_{t} = \mathcal{G}(x_t)$
    \STATE $\bar{r}_t = \max\left(\epsilon, \max_{s\le t} d(x_{s}, x_{0})\right)$
    \STATE $v_t = v_{t-1} + \bar{r}_t^2\|g_{t}\|_{x_{t}}^{2}$
    \STATE \footnotesize $v_t^\prime = 8^4 \theta_{T, \delta}^2 \log_{+}^2\left( \frac{(1 +t)\bar{r}_t^2\bar{\ell}_t^2/\zeta_\kappa(\bar{r}_t)}{\bar{r}_0^2\bar{\ell}_0^2/\zeta_\kappa(\bar{r}_0)}\right)(v_{t-1} + 16 \frac{\bar{r}_t^2}{\zeta_\kappa(\bar{r}_t)}\bar{\ell}^2_t)$
    \STATE $\eta_t = \frac{\bar{r}_t}{\sqrt{\zeta_{\kappa}(\bar{r}_t)v_t^{\prime}}}$
    \STATE $x_{t+1} = \exp_{x_{t}} \left( - \eta_t g_{t} \right)$
   \ENDFOR
\end{algorithmic}
\end{algorithm}

Our first result characterizes the key property of T-RDoWG: bounded iterates with high probability. 
\begin{theorem}
Suppose that \cref{assumption:g_convex}, \ref{assumption:Hadamard}, and \ref{assumption:bounded_stochastic_grads}  hold, and $\epsilon \le 3 d_0$. Then, for any $\delta \in (0, 1)$, and for any $t\in\mathbb{N}$, the iterations of T-RDoWG (\cref{alg:t-rdowg}) satisfy $\mathbb{P}(\bar{r}_t > 3 d_0) \le \delta$.
\end{theorem}
\textit{Proof.} See Appendix \ref{sec:rdowg-stability-analysis}
\qed

Using this result, we can now obtain the convergence rate of T-RDoWG.
\begin{corollary}
    Suppose that \cref{assumption:g_convex}, \ref{assumption:Hadamard}, and \ref{assumption:bounded_stochastic_grads}  hold, and $\epsilon \le 3d_0$. For any $\delta \in (0, 1/2)$, and for any $t \in \mathbb{N}$, let $\tau \in \arg\max_{t \le T} \sum_{s=0}^{t-1} \frac{\bar{r}_s^2/\zeta_\kappa(\bar{r}_s)}{\bar{r}_t^2/\zeta_\kappa(\bar{r}_t)}$. Then T-RDoWG (\cref{alg:t-rdowg}) satisfies the optimality gap $f(\tilde{x}_{\tau}) - f(x_{\star})$ of
    \begin{align*}
     O\left(c \frac{d_0 \sqrt{\zeta_\kappa(d_0) (G_{\tau}+L_{\star}^2)}}{T} \right) 
     = O\left(c \frac{d_0 \sqrt{\zeta_\kappa(d_0)} L_\star}{\sqrt{T}} \right),
    \end{align*}
    with probability at least $1-2\delta$, where $L_\star \coloneqq \max_{x\in \mathcal{M}: d(x, x_0) \le 3 d(x_\star, x_0)} \ell (x)$ and $c = \log_{+}(T \frac{d_0 L_\star}{f(x_0) - f(x_\star)})\log_{+}(\frac{d_0}{\epsilon})\log(\frac{\log_{+}(T)}{\delta})$. % previously $c_{\delta, \epsilon, T}$
\end{corollary}

\textit{Proof.} See Appendix \ref{sec:rdowg-stability-analysis}
\qed 

\begin{remark}
    We can also extend the analysis to obtain a similar optimality gap in the smooth setting. For brevity, we omit the details here.
\end{remark}

\section{Related Work}
\label{related-work}

\textbf{Riemannian optimization}. Numerous authors have studied optimization on Riemannian manifolds. Earlier works on this topic established the asymptotic convergence of first-order methods in both the deterministic \citep{udriste1994,absil2008} and the stochastic \citep{liu2004, Bonnabel_2013} settings. More recently, \citet{zhang16} obtained the first non-asymptotic analysis for Riemannian stochastic gradient descent, assuming geodesic convexity. Subsequently, other authors have obtained iteration complexity results for Riemannian proximal-point methods \citep{bento2017}, Frank-Wolfe schemes \citep{weber2021}, variance reduced methods \citep{zhang2016c,kasai2017,sato2019,zhou2021}, trust-region methods \citep{boumal18,agarwal2021}, amongst others. In parallel, there has also been growing interest in obtaining Riemannian counterparts of accelerated \citep{liu2017,alimsis2020,zhang2018,ahn2020} and adaptive \citep{becigneul2018riemannian,kasai19,cho2017,roy2018} methods used in Euclidean optimization. No existing works, however, consider learning-rate-free Riemannian optimization algorithms.

\textbf{Learning-rate-free Euclidean optimization}. On the other hand, learning-rate-free methods for (stochastic) optimization on Euclidean spaces are substantial; see, e.g., \citet{orabona2020,carmon2022making} and references therein. Most relevant to our work, \citet{carmon2022making} recently introduced a learning-rate-free algorithm for stochastic convex optimization based on interval bisection. Building on this work, \citet{ivgi23a}, \citet{defazio2023} and \citet{khaled2023} have since obtained learning-rate-free (stochastic) convex optimization algorithms which, under varying assumptions, achieve the optimal convergence rate of (stochastic) gradient descent up to a logarithmic factor. Many other learning-rate-free optimization algorithms originate in the online learning literature. These include methods based on coin betting \citep{orabona2016coin,orabona2017}, exponentiated gradients \citep{streeter2012,orabona2013}, amongst others \citep[e.g.,][]{mcmahan2014,orabona2020}. Recently, coin betting ideas have demonstrated effectiveness on Wasserstein spaces \citep{sharrock2023,sharrock2023a,sharrock2023b}, that heuristically follow a Riemannian interpretation \citep{villani2003}.

\section{Experiments}
\label{experiments}

{\setlength{\subfigcapskip}{-1.5mm}
\begin{figure*}[t]
  \centering
  \subfigure[Wine $(d=13, r=1)$.]{\includegraphics[width=0.325\textwidth]{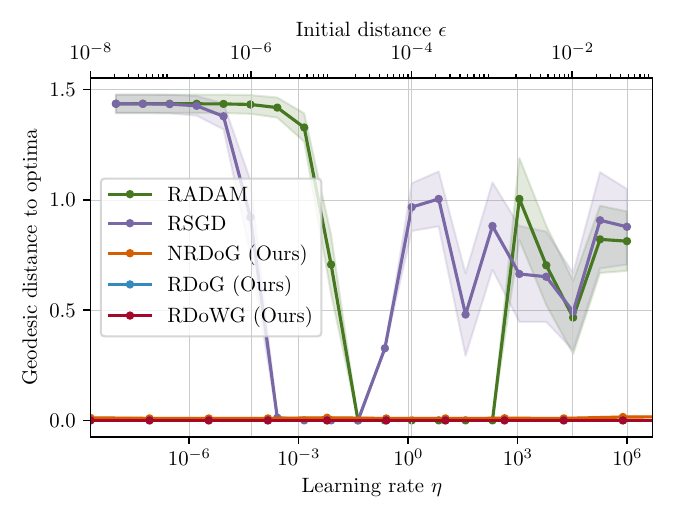}}
  \subfigure[Waveform $(d=40, r=2)$.]{\includegraphics[width=0.325\textwidth]{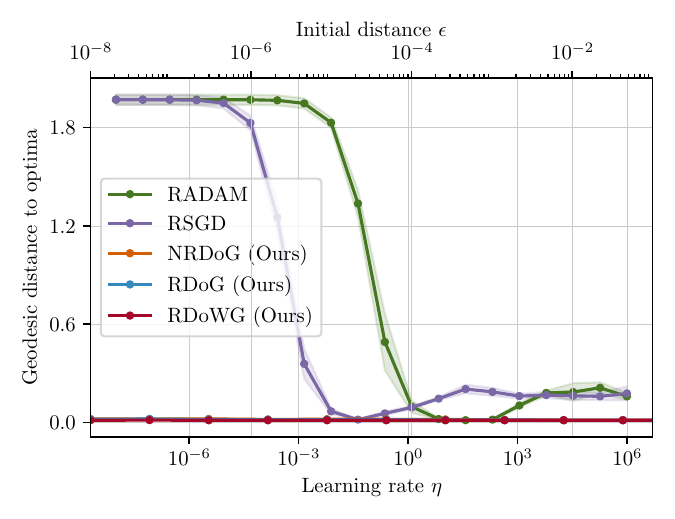}}
  \subfigure[Tiny ImageNet $(d=12,288, r=5)$.]{\includegraphics[width=0.325\textwidth]{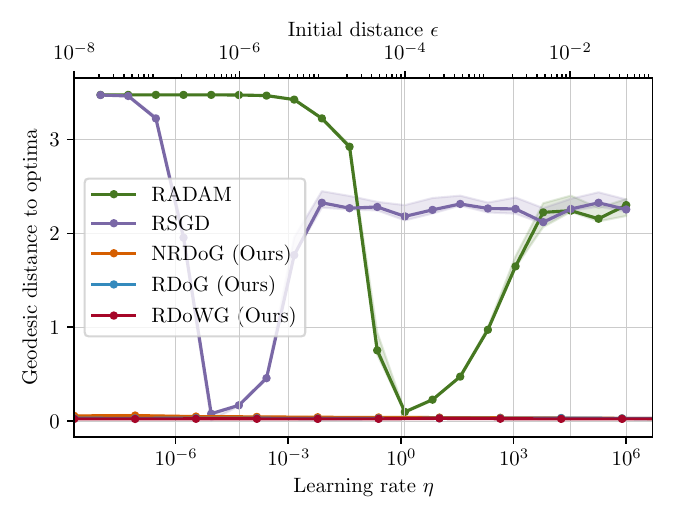}}
  \vspace{-3mm}
  \caption{\textbf{Results for PCA on the Grassmann manifold}. (a)-(c) Geodesic distance between the final iterate and the numerical solution after $T=2000$ iterations as a function of the learning rate for RADAM and RSGD and as a function of the initial distance estimate for RDoG, RDoWG, and NRDoG. (b)-(c) Uses the final iterate of the \emph{weighted average} sequence for RDoG, RDoWG, and NRDog. Results are averaged over five replications.}
  \label{fig:pca}
\vspace{-2mm}
\end{figure*}
}

In this section, we assess the numerical performance of RDoG (\cref{alg:rdog}), RDoWG (\cref{alg:rdowg}), and NRDoG (\cref{alg:nrdog}) against manually tuned RSGD \cite{Bonnabel_2013} and RADAM \cite{becigneul2018riemannian}. Implementing all algorithms in Python 3 with JAX \cite{jax2018github}, our experiments run on a MacBook Pro 16" (2021) with an Apple M1 Pro chip and 16GB of RAM. Detailed manifold descriptions and required operations for the experiments are provided in \cref{sec:geometry}. Code to reproduce the experiments is available at \url{https://github.com/daniel-dodd/riemannian_dog}.

\subsection{Rayleigh Quotient Maximization on the Sphere}
\label{experiments:rayleigh}
We seek to find the dominant eigenvector of a symmetric matrix $A$ in $\mathbb{R}^{d \times d}$ by minimizing $-\frac{1}{2} x^{T} A x$ on the unit sphere $\mathbb{S}^{d-1}$. This is challenging for high-dimensional and ill-conditioned $A$ in the Euclidean case. We consider $A=\frac{1}{d} BB^T$, with $B \in \mathbb{R}^{d \times q}$ having standard Gaussian entries.

For illustration purposes, we first consider $d=3$ and $q=5$ in \cref{fig:toy_illustration}, underscoring the pivotal role of selecting an optimal learning rate for RSGD, as deviations, whether too small or too large, adversely affect performance. 

In a higher-dimensional scenario with $d = 1000$ and $q = 1100 \cong d$, resulting in a high condition number, we employ RADAM and RSGD with a grid of twenty logarithmically spaced learning rates $\eta \in [10^{-8}, 10^6]$. On the other hand, we investigate RDoG and RDoWG with ten logarithmically spaced initial distance values $\epsilon \in [10^{-8}, 10^{0}]$. Here, we initialize ten starting points $x_0 \in \mathbb{R}^d$ by drawing their entries independently from a standard Gaussian distribution, then projecting them onto the sphere through normalizing, a shared procedure for each optimizer.

Our results show that RDoG, RDoWG, and NRDoG are insensitive to initial distance, consistently achieving robust performance in recovering negligible geodesic distance to the numerical solution via the eigendecomposition. In contrast, the effectiveness of RADAM and RSGD depends on selecting an appropriate learning rate. Notably, as seen in \cref{fig:rayleigh}, RSGD is highly sensitive to the learning rate, while RDoG rapidly adapts to optimal regret within a few hundred iterations, irrespective of the initial distance estimate's magnitude. Additional regret trace plots for other optimizers are available in \cref{additional:rayleigh}, along with similar plots for geodesic distance to the optima. These underscore that the algorithms quickly adapt within a few hundred iterations without prior knowledge of the function.

\subsection{PCA on the Grassmann Manifold}
\label{experiments:pca}

We investigate principal component analysis (PCA) on the Grassmann manifold $\mathbb{G}(d, r)$, where points are represented as equivalence classes with an orthogonal matrix $x \in \mathbb{R}^{d \times r}$ having orthonormal columns ($x^{T}x= \mathrm{I}$). The PCA problem minimizes the sum of squared residual errors between projected data points and the original data, $\min_{x \in \mathbb{G}(d, r)} \frac{1}{n} \sum_{i=1}^{n} \lVert z_i - x x^Tz_i \rVert_{2}^2$, with each $z_i$ represented as a $d$-dimensional data point. We consider datasets \mytexttt{Wine}, \mytexttt{Waveform-5000}, and \mytexttt{Tiny ImageNet}. The numerical solution is computed using the scikit-learn implementation \citep{scikit-learn}. The geodesic distances of final iterates (using weighted averages for RDoG and RDoWG) are compared against learning-rate-dependent algorithms, as shown in \cref{fig:pca}.

In training, \mytexttt{Wine} uses the full batch for $T=5000$ iterations, and \mytexttt{Waveform-5000} and \mytexttt{Tiny ImageNet} use batch sizes of 64 for $T=2000$ iterations. Each dataset has an 80:20 train-test split per replication. Following Pymanopt \citep{Pymanopt}, initial points $x_0 \in \mathbb{R}^{d \times r}$ are drawn from a standard Gaussian distribution and projected onto the manifold using vectorized QR decomposition.

Results in \cref{fig:pca} from five random train-test splits show RDoG, RDoWG, and NRDoG are insensitive to initial distance estimates across magnitudes, with ten logarithmically spaced values in $\epsilon \in [10^{-8}, 10^{0}]$. In contrast, RADAM and RSGD require a narrower tuning range of optimal learning rates, exploring twenty logarithmically spaced values in $\eta \in [10^{-8}, 10^6]$. Additional results in \cref{additional:pca} further highlight the robust adaptation of RDoG, RDoWG, and NRDoG.

\subsection{Embedding Graphs in the Poincar\'{e} Ball}
\label{experiments:poincare}

The WordNet noun hierarchy \citep{miller1990wordnet} is a lexical database of English words organized into a hierarchical structure, where each word is categorized based on its semantic relationships with other words. Moreover, the \emph{hypernymy relation}, often termed \emph{Is-A relation}, signifies that one concept (the hypernym) encompasses another (the hyponym). For instance, \mytexttt{mammal} is a hypernym of \mytexttt{dog} and \mytexttt{cat}. Following \citet{Nickel2017}, we consider representing the transitive closure of the mammals' subtree that involves 1,180 nouns denoted as $\mathcal{N}$ (of which \mytexttt{mammal} is a hypernym) and 6,450 hypernymy relations, represented as $\mathcal{R} = \{(u, v)\} \subset \mathcal{N} \times \mathcal{N}$. 

The embedding is performed in the Poincar\'{e} ball of hyperbolic geometry which is well-known to be better suited to embed tree-like graphs than the Euclidean space \citep{gromov1987hyperbolic, sala2018representation}. As such, the Poincar\'{e} ball model is defined as $\mathbb{B}_d = \{x \in \mathbb{R}^d: \lVert x \rVert < 1\}$ equipped with the Riemannian metric $\langle \cdot, \cdot' \rangle_x = 4/(1 - \lVert x \rVert^2)^2 \langle \cdot, \cdot' \rangle$. We adopt the loss function from the official code of \citet{Nickel2017}, deviating from the one described in the paper:
\[
\min_{\theta \colon \mathcal{N} \to \mathbb{B}_d } \sum_{(u,v) \in \mathcal{R}} -\log\left(\frac{e^{-d(\theta(u), \theta(v))}}{\sum_{v' \in \operatorname{Neg}(u, v)} e^{-d(\theta(u), \theta(v'))}}\right),
\]
where each noun pair $(u, v) \in \mathcal{R}$ has associated embeddings $\theta(u), \theta(v) \in \mathbb{B}_d$, and $\operatorname{Neg}(u, v) = \{v' : (u,v') \notin \mathcal{R}\} \cup \{v\}$ is the set of negative examples for $u$, including $v$, and 
\begin{align*}
    d(\cdot, \cdot') = \operatorname{arcosh}\left(1 + 2 \frac{\lVert \cdot - \cdot' \rVert^2}{(1 - \lVert \cdot \rVert^2)(1 - \lVert \cdot' \rVert^2)}\right),
\end{align*}
is the geodesic distance measuring the dissimilarity between the embeddings of two nouns in the Poincar\'{e} ball. Intuitively, minimizing this loss function encourages closely related mammals to be positioned closer together in the embedding space and less similar pairs to be farther apart.

For initialization, following \citet{Nickel2017}, we uniformly initialize the embeddings in $[-10^{-3}, 10^{-3}]^{d}$ and consider ten logarithmically spaced learning rates $\eta \in [10^{-2}, 10^2]$ and five logarithmically spaced initial distance estimates $\epsilon \in [10^{-10}, 10^{-6}]$. In the first ten epochs, we use RSGD with a reduced learning rate of $\eta/10$ for RSGD and RADAM. During this \emph{burn-in} phase, negative word sampling is based on the graph degree raised to the power of $3/4$, leading to numerical improvements. No burn-in heuristic is applied for RDoG, RDoWG, or NRDoG. Thereafter, we run the optimizers on the initialized embeddings for one thousand epochs, with each iteration having a batch size of ten and fifty uniformly sampled negative samples. We repeat this experiment over five replications.

To measure the quality of the embeddings obtained from each optimizer, we follow \citet{Nickel2017} and compute, for each observed edge $(u, v) \in \mathcal{R}$, the corresponding distance $d(u, v)$ in the embedding space and rank it among the distances of all unobserved edges for $u$, i.e., $\{d(u, v') : (u, v') \notin \mathcal{R}\}$. Subsequently, we calculate the mean average precision of this ranking.

In \cref{fig:poincare}, embeddings of dimension five are presented. RDoG and RDoWG demonstrate competitive performance, while RADAM and RSGD require careful tuning. The performance significantly degrades for RADAM and RSGD without a burn-in heuristic, as exemplified in \cref{additional:poincare}. Visualizing two-dimensional embeddings between RDoG and RSGD trained for two thousand epochs, with burn-in applied only for RSGD and using the optimal learning rate selected from ten logarithmically spaced values $\eta \in [10^{-2}, 10^2]$, we observe meaningful groupings across various categories without employing burn-in heuristics for RDoG. Additional embedding plots for the other optimizers are presented in \cref{additional:poincare}.

{\setlength{\subfigcapskip}{-1.5mm}
\begin{figure*}[t]
  \centering
  \subfigure[Mean average precision.]{\includegraphics[width=0.325\textwidth]{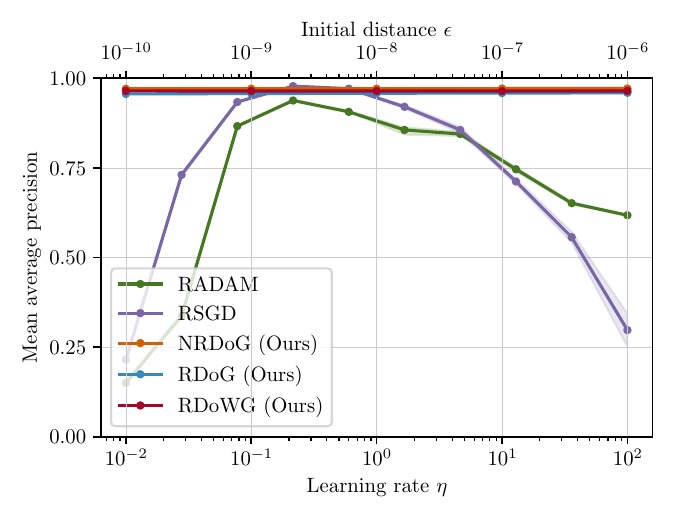}
  \label{fig:map}
  }
  \subfigure[RDoG embeddings.]{\includegraphics[width=0.325\textwidth]{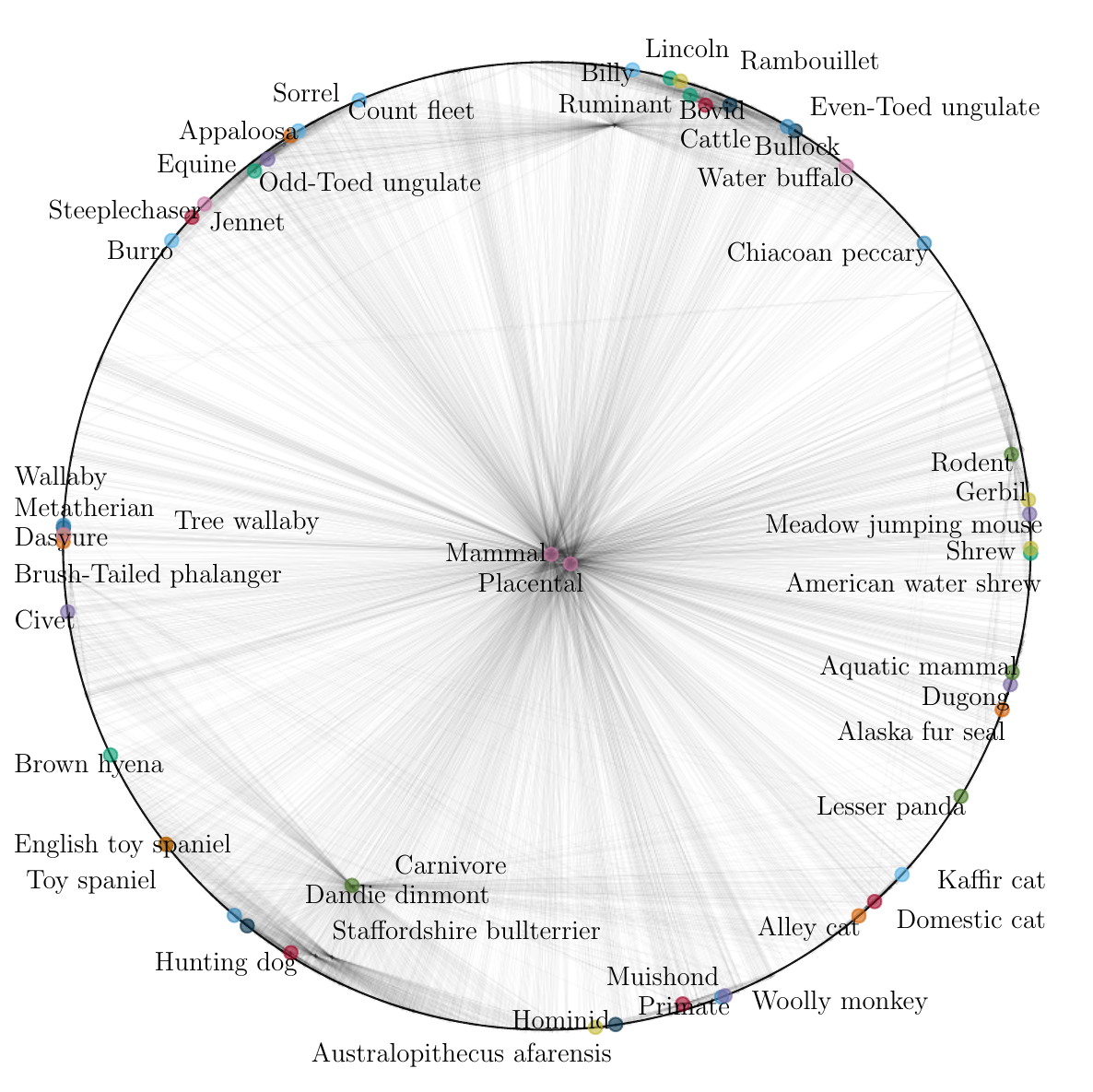}}
  \subfigure[RSGD embeddings.]{\includegraphics[width=0.325\textwidth]{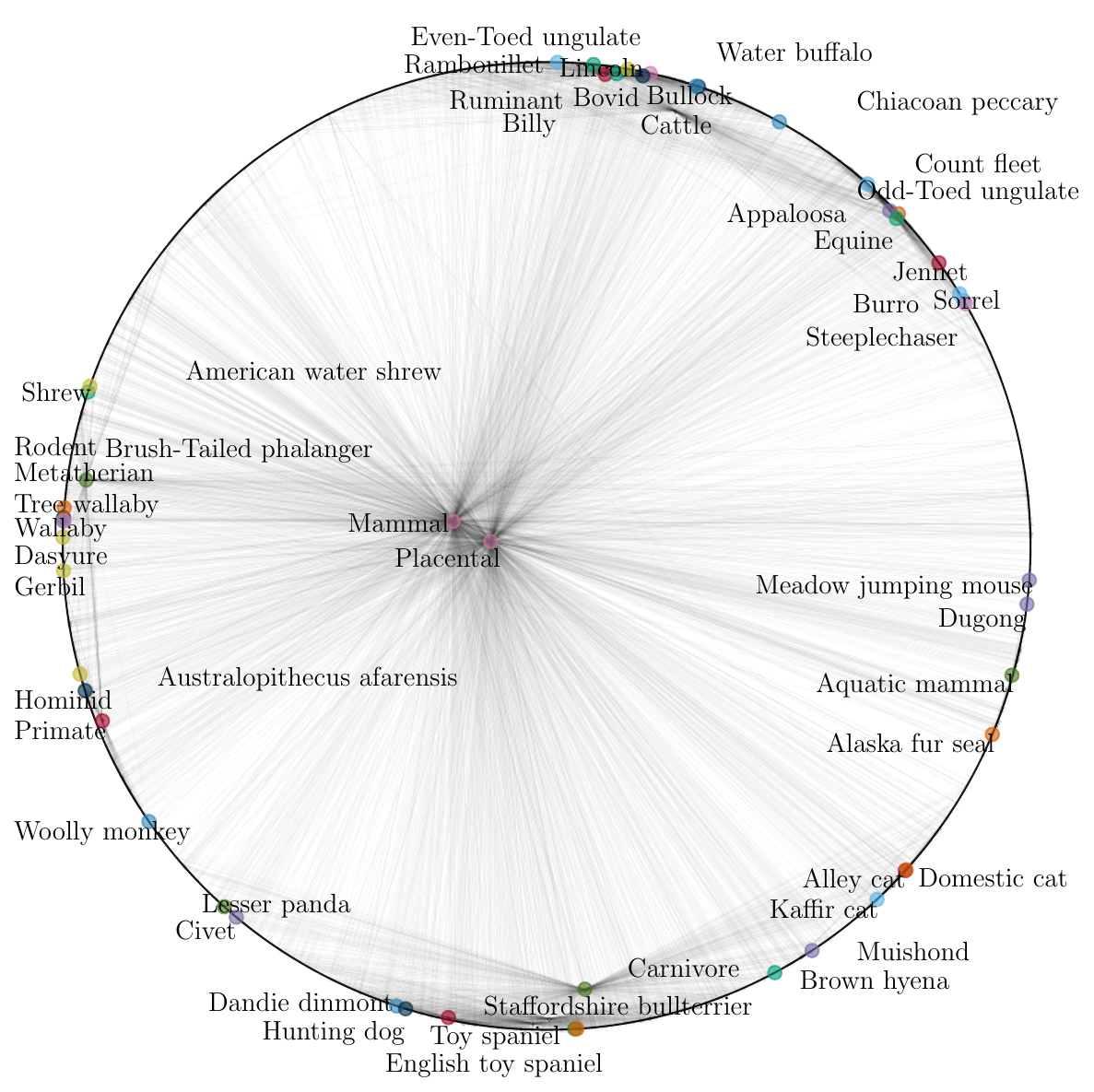}}
  \vspace{-3mm}
  \caption{\textbf{Results for Poincar\'{e} word embeddings}. (a) The mean average precision of the embeddings is assessed against the ground truth after 1000 training epochs. Results are averaged over five replications, with the embedding dimension set to five. (b)-(c) Two-dimensional embeddings after 2000 training epochs are visualized and annotated for the first 50 nouns of the mammal's subtree for RDoG and RSGD.}
  \label{fig:poincare}
\vspace{-2mm}
\end{figure*}
}

\section{Discussion}
We have introduced new learning-rate-free optimizers for Riemannian manifolds and have highlighted significant numerical improvements over learning-rate-dependent algorithms. Our theoretical results provide high probability convergence guarantees that are optimal, up to a logarithmic factor, compared to the theoretically, yet practically unavailable, optimal deterministic algorithms. 

Many existing Riemannian optimization methods rely on a retraction map, which serves as a cost-effective approximation of the exponential map on manifolds and is a reasonable choice in numerous real-world scenarios. Incorporating this into our framework is paramount for enhancing the effectiveness of our algorithms, particularly in large-scale optimization problems. Moreover, certain Riemannian manifolds, such as the Stiefel or multivariate Gaussian Fisher-Rao manifolds, pose challenges due to the intractability of the geodesic distance. Recognizing the argument underlying our convergence guarantees (though potentially less robust) holds for upper bounds on geodesic distance, exploring tractable or more economical approximations in these situations is essential.

Additionally, it is crucial to explore integrating these methods with recent proven advances in momentum acceleration \citep{liu2017,alimsis2020,zhang2018,ahn2020}, a challenge both in theory and practice. Furthermore, developing practical algorithms that offer guarantees on iterate boundedness is a consideration for future research.

\section*{Acknowledgements}
We thank the anonymous reviewers for their constructive feedback. DD was supported by the EPSRC-funded STOR-i Centre for Doctoral Training, grant number EP/S022252/1. LS and CN were supported by the Engineering and Physical Sciences Research Council (EPSRC), grant number EP/V022636/1. CN acknowledges further support from the EPSRC, grant numbers EP/S00159X/1 and  EP/Y028783/1.

\section*{Impact Statement}
This paper presents work whose goal is to advance the field of 
Machine Learning. There are many potential societal consequences 
of our work, none which we feel must be specifically highlighted here.

\bibliography{example_paper}
\bibliographystyle{icml2024}

%%%%%%%%%%%%%%%%%%%%%%%%%%%%%%%%%%%%%%%%%%%%%%%%%%%%%%%%%%%%%%%%%%%%%%%%%%%%%%%
%%%%%%%%%%%%%%%%%%%%%%%%%%%%%%%%%%%%%%%%%%%%%%%%%%%%%%%%%%%%%%%%%%%%%%%%%%%%%%%
% APPENDIX
%%%%%%%%%%%%%%%%%%%%%%%%%%%%%%%%%%%%%%%%%%%%%%%%%%%%%%%%%%%%%%%%%%%%%%%%%%%%%%%
%%%%%%%%%%%%%%%%%%%%%%%%%%%%%%%%%%%%%%%%%%%%%%%%%%%%%%%%%%%%%%%%%%%%%%%%%%%%%%%
\newpage
\appendix
\onecolumn

\section{Useful Results}

We begin by introducing essential lemmas for the establishment of our theory.

\subsection{Trigonometric Distance Bounds for Manifolds}
\label{appendix:distance_bound}

The law of cosines in Euclidean space is fundamental for analyzing optimization algorithms,
\begin{align}
    a^2 = b^2 + c^2 - 2bc \cos(A),
\end{align}
where $a, b, c$ are the sides of a Euclidean triangle with $A$ the angle between sides $b$ and $c$. 

Trigonometric geometry behaves differently in manifolds compared to Euclidean spaces. While the equality does not hold for nonlinear spaces, a trigonometric distance bound can be established for manifolds with curvature bounded below.

\begin{lemma} \citep[Lemma 5]{zhang16}
\label{lemma:geodesic_triangle_inq}
    If $a, b, c$ are the side lengths of a geodesic triangle $\Delta$ in a Riemannian manifold with sectional curvature lower bounded by $\kappa>-\infty$ and $A$ is the angle between sides $b$ and $c$ (defined through the inverse exponential map and inner product in tangent space), then
    \begin{align}
        a^2 \le \zeta_\kappa(c) b^2 + c^2 - 2bc \cos(A).
    \end{align}
    \begin{proof}
         Given by Lemma 3.12 of \cite{Cordero2001} and by Lemma 5 of \cite{zhang16}.
    \end{proof}
\end{lemma}

This lemma holds profound implications for our analysis of geodesically convex functions $f$. Specifically, the property of geodesic convexity allows us to bound $f(x_t) - f(x_\star)$ by the inner product $\langle -\operatorname{grad} f(x_t), \exp_{x_t}^{-1}(x_\star)\rangle_{x_t}$. The above trigonometric inequality empowers us to bound this inner product to devise tractable optimization algorithms.

To streamline future analysis, we expand our perspective to encompass bounding the inner product $\langle -g_t, \exp_{x_t}^{-1}(x_\star)\rangle_{x_t}$ for any tangent vector $g_t \in \mathcal{T}_{x_t} \mathcal{M}$.

\begin{lemma}\citep[Corollary 8]{zhang16}
    \label{cor:law_of_cosines}
    For any Riemannian manifold $\mathcal{M}$ where the sectional curvature is lower bounded by $\kappa>-\infty$ and any point $x_\star, x_t \in \mathcal{M}$ and any tangent vector $g_t \in \mathcal{T}_{x_t} \mathcal{M}$, scalar $\eta_t >0$  consider the RSGD update $x_{t+1} = \operatorname{exp}_{x_t}(-\eta_t g_t)$. Then by \cref{lemma:geodesic_triangle_inq}, we have
    \begin{align}
       \langle  -g_t, \exp_{x_t}^{-1}(x_\star)\rangle_{x_t} \le \frac{1}{2 \eta_t}\left(d_t^2 - d^2_{t+1} \right) + \frac{\eta_t}{2} \zeta_\kappa(d_t) \lVert g_t \rVert_{x_t}^2.
   \end{align}
    \begin{proof}
        Consider the geodesic triangle $\Delta$ with vertices $x_{t+1}$, $x_t$, and $x_{\star}$. Then we have the side lengths of $\Delta$ are given by
        \begin{align}
            a = d(x_{t+1},x_{\star}) = d_{t+1}, \quad
            b = d(x_{t+1}, x_t) = \eta_t \lVert g_t\rVert_{x_t},
            \quad
            c = d(x_{t}, x_{\star}) = d_{t}.
        \end{align}
   Recalling that the angle between two tangent vectors $u$ and $v$ at $x \in \mathcal{M}$ is given by $\arccos \frac{\langle u, v \rangle_x}{\lVert u \rVert_x \lVert v \rVert_x}$. Now, considering the angle, $A$, between side lengths $b$ and $c$, we have,
\begin{align}
    2bc \cos(A) = 2bc\cos \left(\arccos \left( \frac{\langle \exp_{x_t}^{-1}(x_{t+1}), \exp_{x_t}^{-1}(x_{\star}) \rangle_{x_t}}{\lVert \exp_{x_t}^{-1}(x_{t+1}) \rVert_{x_t} \lVert \exp_{x_t}^{-1}(x_{\star}) \rVert_{x_t}} \right) \right) = \langle  -\eta_t g_t, \exp_{x_t}^{-1}(x_\star)\rangle_{x_t}.
\end{align}
Substituting these terms in \cref{lemma:geodesic_triangle_inq} and rearranging yields the result as required.
\end{proof}
    
\end{lemma}

\subsection{Jensen's Inequality for Geodesically Convex Functionals}

We present an analog for Jensen's inequality for geodesically convex functions on Riemannian manifolds. This will allow us to leverage innovative weighted averaging strategies in the regret analysis of our algorithms.

\begin{lemma} \label{lemma:weighted_average}
    Let $f$ be geodesically convex. For any sequence of iterates $x_{0}, \dotsc, x_{t} \in \mathcal{M}$ and positive weights $w_0, \dotsc, w_t \in \mathbb{R}_{+}$, define the online weighted average sequence by 
    \begin{align}
        \tilde{x}_{t+1} = \operatorname{exp}_{\tilde{x}_{t}}\left(\frac{w_t}{\sum_{s=0}^{t} w_s} \operatorname{exp}_{\tilde{x}_{t}}^{-1}\left( x_t \right) \right), \qquad \tilde{x}_{1} = x_{0}.
    \end{align}
    Then we have 
    \begin{align}
    f(\tilde{x}_{t}) \le \frac{1}{\sum_{s=0}^{t-1} w_s} \sum_{s=0}^{t-1} w_s f(x_s).
    \end{align}
    \begin{proof}
        We prove this by induction. The base case for $t=1$ holds by definition. Now for $t\ge2$, for the inductive step, assume the statement is true for $t-1$ and consider $t$. We have,
        \begin{align}
        \frac{1}{\sum_{s=0}^{t-1} w_s} \sum_{s=0}^{t-1} w_s f(x_s)  &= \frac{w_{t-1}}{\sum_{s=0}^{t-1} w_s}  f(x_{t-1}) + \frac{1}{\sum_{s=0}^{t-1} w_s} \sum_{s=0}^{t-2} w_s f(x_s) \\
        &= \frac{w_{t-1}}{\sum_{s=0}^{t-1} w_s}  f(x_{t-1}) + \frac{\sum_{s=0}^{t-2} w_s}{\sum_{s=0}^{t-1} w_s} \frac{1}{\sum_{s=0}^{t-2} w_s} \sum_{s=0}^{t-2} w_s f(x_s) 
        \\ &\ge \frac{w_{t-1}}{\sum_{s=0}^{t-1} w_s}  f(x_{t-1}) + \frac{\sum_{s=0}^{t-2} w_s}{\sum_{s=0}^{t-1} w_s} f(\tilde{x}_{t-1}) \label{eqaution:geodesic_convex_sum_thing}.
        \end{align}
        In the final line, we have exploited the inductive assumption. Finally, we note that $\gamma(s) = \exp_x\left((1-s)\exp_x^{-1}(x) + s\exp_x^{-1}(y)\right)$ for $s \in [0, 1]$ defines a geodesic between any two points $x$ and $y$ in $\mathcal{M}$. Moreover, by geodesic convexity we have
        \begin{align}
            f(\gamma(s)) \le  (1-s) f(\gamma(0)) + sf(\gamma(1)) = (1-s) f(x) + sf(y).
        \end{align}
        Thus applying this to \cref{eqaution:geodesic_convex_sum_thing} with $x=\tilde{x}_{t-1}$, $y = x_{t-1}$ and $s = \frac{w_{t-1}}{\sum_{s=0}^{t-1} w_s} $ and noting that for this choice,
        \begin{align}
            \gamma\left(\frac{w_{t-1}}{\sum_{s=0}^{t-1} w_s}\right) &= \exp_{\tilde{x}_{t-1}} \left(\left(1-\frac{w_{t-1}}{\sum_{s=0}^{t-1} w_s}\right) \exp_{\tilde{x}_{t-1}}^{-1}(\tilde{x}_{t-1}) + 
             \frac{w_{t-1}}{\sum_{s=0}^{t-1} w_s}  \exp_{\tilde{x}_{t-1}}^{-1}(x_{t-1}) \right) \\
            &= \exp_{\tilde{x}_{t-1}} \left(\frac{w_{t-1}}{\sum_{s=0}^{t-1} w_s}  \exp_{\tilde{x}_{t-1}}^{-1}(x_{t-1}) \right) \\ &= \tilde{x}_{t},
        \end{align}
        yields the result as required.
    \end{proof}
\end{lemma}

\subsection{Smoothness Bounds}

We present smoothness results that establish bounds on individual gradient norms, that we will use in our later analysis to yield tighter regret bounds under the geodesic smoothness assumption.

\begin{lemma} \label{lemma:lsmooth_bound}
    Suppose $f$ is $S$-smooth and lower bounded by $f(x_\star)$. Then, for all $x \in \mathcal{M}$ we have
   \begin{align}
       \lVert \operatorname{grad} f(x) \rVert_x \le \sqrt{2S (f(x) - f(x_\star))}.
   \end{align}

   \begin{proof}
    This is a trivial consequence of e.g., Proposition 4.7 and 4.8 of \cite{boumal23}. We include the proof for completeness. Let $x \in M$ and define $y = \exp_x \left(-\frac{1}{S} \operatorname{grad} f(x)\right)$. Then geodesic smoothness provides,
    \begin{align}
        f(y) &\le f(x) + \langle \operatorname{grad} f(x), \exp_{x}^{-1}(y) \rangle_x + \frac{S}{2} \lVert \exp_{x}^{-1}(y) \rVert_x^{2} \\
        & = f(x) - \frac{1}{S} \lVert \operatorname{grad} f(x) \rVert_x^{2} + \frac{1}{2S} \lVert \operatorname{grad} f(x) \rVert_x^{2} \\
        & = f(x) - \frac{1}{2S} \lVert \operatorname{grad} f(x) \rVert_x^{2}.
    \end{align}
    Now since $f$ is lower bounded by $f(x_\star)$ we thus have
    \begin{align}
        f(x_\star) \le f(y) &\le f(x) - \frac{1}{2S} \lVert \operatorname{grad} f(x) \rVert_x^{2}.
    \end{align}
    Rearranging gives the result.
   \end{proof} 
\end{lemma}

Using the above argument, we provide a bound on the norm of the stochastic error.

\begin{lemma}
\label{lemma:bound_stochastic_error_smooth}
    Under locally smooth stochastic gradients (\cref{assumption:bounded_stochastic_lsmooth}), for the stochastic error $\Delta(x) \coloneqq \mathcal{G}(x) - \operatorname{grad} f(x)$ we almost surely have that
    \begin{align}
        \lVert \Delta(x) \rVert_{x} \le (\sqrt{s(x)} + \sqrt{S}) \sqrt{2(f(x) - f(x_\star))}.
    \end{align}
    \begin{proof}
        Noting that \cref{assumption:bounded_stochastic_lsmooth} implies that for any $x, y \in \mathcal{M}$ we almost surely have that
        \begin{align}
        f(s) \le f(x) + \langle \mathcal{G}(x), \exp_{x}^{-1}(y) \rangle_x + \frac{s(x)}{2} \lVert \exp_{x}^{-1}(y) \rVert_{x}^2.
        \end{align}
        We follow the same argument as in \cref{lemma:lsmooth_bound} to deduce that almost surely,
        \begin{align}
            \lVert \mathcal{G}(x) \rVert_x \le \sqrt{2 s(x) (f(x) - f(x_\star))}.
        \end{align}
        While the triangle inequality and applying \cref{lemma:lsmooth_bound}  to $\lVert  \operatorname{grad} f(x) \rVert_{x}$ gives,
        \begin{align}
        \lVert \Delta(x) \rVert_{x} \le  \lVert \mathcal{G}(x) \rVert_{x} + \lVert  \operatorname{grad} f(x) \rVert_{x} \le \sqrt{2 s(x) (f(x) - f(x_\star))} + \sqrt{2S (f(x) - f(x_\star))}.
        \end{align}
    \end{proof}
\end{lemma}

\subsection{Bounds for Real-Valued Series}

\begin{lemma}
    \citep[Lemma 3]{ivgi23a}
    \label{lemma:fraction_sum}
    Let $a_0, a_1, \dotsc, a_T$ be a positive increasing sequence. Then
    \begin{align}
        \max_{t \le T} \sum_{s=0}^{t-1} \frac{a_s}{a_t} \ge e^{-1} \left( \frac{T}{1 + \log(a_T/a_0)} - 1 \right).
    \end{align}
\end{lemma}
    \begin{proof}
        Lemma 3 of \cite{ivgi23a}. Define $K \coloneqq \lceil \log(a_T / a_0) \rceil$, and $n \coloneqq \lfloor T/K \rfloor$. Then, given the sequence is increasing we have
        \begin{align}
            \log\left(\frac{a_{T}}{a_0}\right) \ge \sum_{k=0}^{K-1} \log\left(\frac{a_{n(k+1)}}{a_{nk}}\right) \ge K \min_{k < K} \log \left(\frac{a_{n(k+1)}}{a_{nk}} \right).
        \end{align}
        Rearranging gives,
        \begin{align}
            \min_{k < K} \log \left(\frac{a_{n(k+1)}}{a_{nk}} \right) \le \log\left(\frac{a_{T}}{a_0}\right)/K \le 1 \implies \min_{k < K} \frac{a_{n(k+1)}}{a_{nk}} \le e.
        \end{align}
        Thus,
        \begin{align}
        \max_{t \le T} \sum_{s=0}^{t-1} \frac{a_s}{a_t} \ge \max_{t\in [n, T]} n \frac{a_{t-n}}{a_t} = \max_{k \le K} n \frac{a_{n(k-1)}}{a_{nk}} \ge n e^{-1} \\ = e^{-1} \left\lfloor \frac{T}{\lceil \log(a_T / a_0) \rceil}\right\rfloor \ge e^{-1} \left( \frac{T}{1 + \log(a_T/a_0)} - 1 \right).
        \end{align}
    \end{proof}

\begin{lemma}
    \citep[Lemma 4]{ivgi23a}. Let $a_0, \dotsc,a_t$ be a nondecreasing sequence of nonnegative numbers. Then
    \begin{align}
        \sum_{k=1}^{t} \frac{a_k - a_{k-1}}{\sqrt{a_k}} \le 2 \left (\sqrt{a_t} - \sqrt{a_0}\right). 
    \end{align}
    \begin{proof}
        This is a well-known result e.g., Lemma 4 of \cite{ivgi23a}. We have
        \begin{align}
            \sum_{k=1}^{t}\frac{a_k - a_{k-1}}{\sqrt{a_k}} &=  \sum_{k=1}^{t} \frac{(\sqrt{a_k} + \sqrt{a_{k-1}})(\sqrt{a_k} - \sqrt{a_{k-1}})}{\sqrt{a_k}} \\ & \le 2 \sum_{k=1}^{t}(\sqrt{a_k} - \sqrt{a_{k-1}}) = 2 (\sqrt{a_t} - \sqrt{a_0}).
        \end{align}
    \end{proof}
\end{lemma}

\begin{lemma}
    \label{lemma:sum_bounded_by_one}
    \citep[Lemma 6]{ivgi23a}. Recall $\log_{+}(z) \coloneqq 1 + \log(z)$. Consider a non-decreasing sequence of nonnegative numbers, $a_{-1}, a_0, a_1, \dotsc, a_t$, then
    \begin{align}
        \sum_{k=0}^{t} \frac{a_k - a_{k-1}}{a_k \log_{+}^2(a_k/a_{-1})} \le 1.
    \end{align}
    \begin{proof}
        Lemma 6 of \cite{ivgi23a}. We have
        \begin{align}
            \sum_{k=0}^{t} \frac{a_k - a_{k-1}}{a_k \log_{+}^2(a_k/a_{-1})} & \le \sum_{k=0}^{t} \int_{a_{k-1}/a_0}^{a_k/a_{-1}}\frac{\mathrm{d} \alpha}{\alpha \log_{+}^2(\alpha)} = \int_{1}^{a_t / a_{-1}}\frac{\mathrm{d} \alpha}{\alpha \log_{+}^2(\alpha)} \\ &\le \int_{1}^{\infty}\frac{\mathrm{d} \alpha}{\alpha \log_{+}^2(\alpha)} = \left[\frac{1}{1 + \log(\alpha)}\right]_{1}^{\infty} = 1.
        \end{align}
    \end{proof}
\end{lemma}

\subsection{Martingale Concentration bound}

\begin{lemma}
    \label{lemma:martingale_bound}
    \citep[Lemma 7]{ivgi23a}. Consider a filtration process $\mathcal{F}_t$ and let $\mathbb{S}$ be the set of nonnegative and nondecreasing sequences. Let $C_t \in \mathcal{F}_{t-1}$ and let $X_t$ be a martingale difference sequence adapted to $\mathcal{F}_{t-1}$ such that $\lvert X_t \rvert \le C_t$ with probability $1$ for all $t$. Recalling that $\theta_{t, \delta} \coloneqq \log(60 \log (6 t) / \delta).$ Then, for all $\delta \in (0, 1)$, $c > 0$, and $\hat{X}_t \in \mathcal{F}_{t-1}$ such that $\lvert \hat{X}_t \rvert \le C_t$ with probability $1$,
    \begin{align}
        \mathbb{P} \left( \exists t \le T, \exists \{y_s\}_{s=1}^{\infty} \in \mathbb{S} : \left\lvert \sum_{s=1}^{t} y_s X_s \right\rvert \ge 8 y_t \sqrt{\theta_{t, \delta} \sum_{s=1}^{t} \left( X_s - \hat{X}_s\right)^2 + c^2\theta_{t, \delta} ^2} \right) \le \delta + \mathbb{P}\left(\exists t \le T : C_t > c \right).
    \end{align}
    \begin{proof}
        See Lemma 7 of \cite{ivgi23a}.
    \end{proof}
\end{lemma}

\section{RGD ``Ideal Step Size'' Analysis}
\subsection{Proof of \cref{thm:intractable_oracle}}
\label{thm:intractable_oracle_proof}
\begin{proof}
    Using geodesic convexity and applying \cref{cor:law_of_cosines}, we have
    \begin{align}
        \min_{t \le T} \left[f(x_t) - f(x_\star)\right] &\le \frac{1}{T}\sum_{t=0}^{T}\left[f(x_t) - f(x_\star)\right] \\ 
        & \le \frac{1}{T}\sum_{t=0}^{T} \langle -g_t, \exp_{x_t}^{-1}(x_\star)\rangle_{x_t} \\
    &\le \frac{1}{T}\sum_{t=0}^{T} \left[\frac{1}{2 \eta}\left(d_t^2 - d_{t+1}^2\right) + \frac{\eta}{2} \zeta_\kappa(d_t) \lVert g_t \rVert_{x_t}^2 \right] \\
    &= \frac{d_0^2}{2 \eta T} + \frac{\eta \sum_{t=0}^{T}\zeta_{\kappa}(d_t) \lVert g_t \rVert_{x_t}^2}{2T} \\
    &\le \frac{\bar{d}_T^2}{2 \eta T} + \frac{\eta \zeta_{\kappa}(\bar{d}_T)\sum_{t=0}^{T} \lVert g_t \rVert_{x_t}^2}{2T}.
    \end{align}
    Now, setting $\eta = \frac{\bar{d}_T}{\sqrt{\zeta_{\kappa}(\bar{d}_T)\sum_{t=0}^{T} \lVert g_t \rVert_{x_t}^2}}$, we have
    \begin{align}
    \min_{t \le T} \left[f(x_t) - f(x_\star)\right]  &\le \frac{\bar{d}_T \sqrt{\zeta_{\kappa}(\bar{d}_T) \sum_{t=0}^{T} \lVert g_t \rVert_{x_t}^2}}{2 T} + \frac{\zeta_{\kappa}(\bar{d}_T) \sum_{t=0}^{T} \lVert g_t \rVert_{x_t}^2}{2T \zeta_{\kappa}(\bar{d}_T) \sqrt{\sum_{t=0}^{T} \lVert g_t \rVert_{x_t}^2}} \\
    &= \frac{\bar{d}_T \sqrt{\zeta_{\kappa}(\bar{d}_T) \sum_{t=0}^{T} \lVert g_t \rVert_{x_t}^2}}{T}
    \\
    & \le  \frac{L \bar{d}_T \sqrt{\zeta_{\kappa}(\bar{d}_T)}}{\sqrt{T}}.
    \end{align}
    Where we have bounded $\lVert g_t \rVert_{x_t} \le L$ due to the Lipschitz assumption, and $d_\infty \ge d_t$ for all $t \ge 0$.
    
\end{proof}

\section{NRGD ``Ideal Step Size'' Analysis}
\label{sec:theoretical-results}

\subsection{Proof of \cref{thm:ngd_oracle}}
\label{thm:ngd_oracle_proof}
\begin{proof}[Proof of \cref{thm:ngd_oracle}]
Using \cref{cor:law_of_cosines} we have
\begin{align}
    \left\langle \frac{-\operatorname{grad} f(x_t)}{\lVert \operatorname{grad} f(x_t) \rVert_{x_t}}, \exp_{x_t}^{-1}(x_\star) \right\rangle_{x_t} \le \frac{1}{2 \eta} \left( d_{t}^2 - d_{t+1}^2 \right) + \frac{\eta}{2} \zeta_{\kappa}(d_{t}).
\end{align}
Averaging the above, we have
\begin{align}
    \frac{1}{T} \sum_{t=0}^{T}\left\langle \frac{-\operatorname{grad} f(x_t)}{\lVert \operatorname{grad} f(x_t) \rVert_{x_t}}, \exp_{x_t}^{-1}(x_\star) \right\rangle_{x_t} \le \frac{d_0^2}{2 \eta T} + \frac{\eta}{2 T} \sum_{t=0}^{T}\zeta_{\kappa}(d_t).
\end{align}
Now for the Lipshitz setting, we have $\lVert \operatorname{grad} f(x_t) \rVert_{x_t} \le L$ thus,
\begin{align}
   \frac{1}{T} \sum_{t=0}^{T}\left\langle \frac{-\operatorname{grad} f(x_t)}{L}, \exp_{x_t}^{-1}(x_\star) \right\rangle_{x_t} \le \frac{1}{T} \sum_{t=0}^{T}\left\langle \frac{-\operatorname{grad} f(x_t)}{\lVert \operatorname{grad} f(x_t) \rVert_{x_t}}, \exp_{x_t}^{-1}(x_\star) \right\rangle_{x_t} \le \frac{d_0^2}{2 \eta T} + \frac{\eta}{2 T} \sum_{t=0}^{T}\zeta_{\kappa}(d_t).
\end{align}
Multiplying through by $L$ and using definition of geodesic convexity yields,
\begin{align}
    \min_{t \le T} \left[f(x_t) - f(x_\star)\right] \le \frac{1}{T} \sum_{t=0}^{T} \left[f(x_t) - f(x_\star)\right] \le L \left[\frac{d_0^2}{2 \eta T} + \frac{\eta}{2T} \sum_{t=0}^{T}\zeta_{\kappa}(d_t) \right] \le \left[\frac{\bar{d}_T^2}{2 \eta T} + \frac{\eta}{2} \zeta_{\kappa}(\bar{d}_T) \right]. 
\end{align}
Now substituting $\eta = \frac{\bar{d}_T}{\sqrt{T\zeta_\kappa(\bar{d}_T)}}$,
gives
\begin{align}
\min_{t \le T} \left[f(x_t) - f(x_\star)\right] \le \frac{L \bar{d}_T \sqrt{T \zeta_\kappa(\bar{d}_T)}}{T} \le \frac{L\bar{d}_T\sqrt{\zeta_\kappa(\bar{d}_T)}}{\sqrt{T}},
\end{align}
which completes the proof for the Lipshitz case. 

Now we proceed to consider the smooth setting. By convexity we have
\begin{align}
    \langle -\operatorname{grad} f(x_t), \exp_{x_t}^{-1}(x_\star) \rangle_{x_t} \ge f(x_t) - f(x_\star) \ge 0.
\end{align}

And by smoothness (\cref{lemma:lsmooth_bound}), we have
\begin{align}
       \lVert \operatorname{grad} f(x_t) \rVert_{x_t} \le \sqrt{2S (f(x_t) - f(x_\star))}.
\end{align}
Now if $f(x_t) = f(x_\star)$ the theorem holds trivially, suppose not. Then combining the above expressions, we have
\begin{align}
\left \langle \frac{-\operatorname{grad} f(x_t)}{\lVert \operatorname{grad} f(x_t) \rVert_{x_t}}, \exp_{x_t}^{-1}(x_\star) \right\rangle_{x_t} \ge \frac{f(x_t) - f(x_\star)}{\sqrt{2S (f(x_t) - f(x_\star))}} = \frac{\sqrt{(f(x_t) - f(x_\star))}}{\sqrt{2S}}.
\end{align}
Thus we have
\begin{align}
\min_{t \le T} \left[\sqrt{f(x_t) - f(x_\star)}\right] \le \frac{1}{T} \sum_{t=0}^{T} \sqrt{f(x_t) - f(x_\star)} \le \sqrt{2S} \left[\frac{d_0^2}{2 \eta T} + \frac{\eta}{2T} \sum_{t=0}^{T}\zeta_{\kappa}(d_t) \right] \le \sqrt{2S} \left[\frac{\bar{d}_T^2}{2 \eta T} + \frac{\eta}{2} \zeta_{\kappa}(\bar{d}_T) \right].
\end{align}
Squaring gives us the first result. Now, plugging in $\eta = \frac{\bar{d}_T}{\sqrt{T\zeta_\kappa(\bar{d}_T)}}$,
gives
\begin{align}
\min_{t \le T} \left[f(x_t) - f(x_\star)\right] \le \frac{2S\bar{d}_T^{2} T\zeta_\kappa(\bar{d}_T)}{T^2} \le \frac{2S\bar{d}_T^{2}\zeta_\kappa(\bar{d}_T)}{T}.
\end{align}
\end{proof}

\section{RDoG Theoretical Analysis}

\subsection{Overview}
In this section, we analyze RDoG (\cref{alg:rdog}). Thus we consider RSGD with step sizes given by,
\begin{align}
    \eta_t = \frac{\bar{r}_t}{\sqrt{\zeta_\kappa(\bar{r}_t) \sum_{s=0}^{t} \lVert g_s \rVert_{x_s}^2}}, \quad
\end{align}

We consider bounding the error of the weighted average sequence,
\begin{align*}
        \tilde{x}_{t+1} = \operatorname{exp}_{\tilde{x}_{t}}\left(\frac{\bar{r}_t /\sqrt{\zeta_\kappa(\bar{r}_t)}}{\sum_{s=0}^{t} \bar{r}_s /\sqrt{\zeta_\kappa(\bar{r}_s)}} \operatorname{exp}_{\tilde{x}_{t}}^{-1}\left( x_{t} \right) \right), \quad \tilde{x}_{1} = x_{0}.
\end{align*}
For a geodesically convex function $f \colon \mathcal{M} \to \mathbb{R}$, we have by \cref{lemma:weighted_average} that $\tilde{x}_t$ satisfies,
\begin{align}
    f(\tilde{x}_t) - f(x_\star) \le \frac{1}{\sum_{s=0}^{t-1} (\bar{r}_s /\sqrt{\zeta_\kappa(\bar{r}_s)})}\sum_{s=0}^{t-1} (\bar{r}_s /\sqrt{\zeta_\kappa(\bar{r}_s)}) \langle  -\operatorname{grad} f(x_s), \exp_{x_s}^{-1}(x_\star)\rangle_{x_s}.
\end{align}
Recalling that $g_s$ represents the stochastic oracle evaluation at $x_s$, denoted as $\mathcal{G}(x_s)$, we can decompose the numerator into two components:
\begin{align}
    \underset{\text{weighted regret}}{\underbrace{\sum_{s=0}^{t-1} (\bar{r}_s /\sqrt{\zeta_\kappa(\bar{r}_s)}) \langle -g_s, \exp_{x_s}^{-1}(x_\star)\rangle_{x_s}}} + \underset{\text{noise}}{\underbrace{\sum_{s=0}^{t-1} (\bar{r}_s /\sqrt{\zeta_\kappa(\bar{r}_s)}) \langle \Delta_s, \exp_{x_s}^{-1}(x_\star)\rangle_{x_s}}},
\end{align}
with $\Delta_s \coloneqq g_s - \operatorname{grad} f(x_s)$.

\subsection{Non-Smooth Analysis}
\label{sec:rdog-non-smooth-analysis}

We give deterministic bounds for the weighted regret (\cref{lemma:weighted_regret_bound}) and high probability bounds for the noise term (\cref{lemma:noise_bound}).

\begin{lemma} \label{lemma:weighted_regret_bound}
   Under \cref{assumption:g_convex} and \ref{assumption:Hadamard}, we have that the iterates of RDoG (\cref{alg:rdog}) satisfy
   \begin{align}
       \sum_{s=0}^{t-1} (\bar{r}_s/\sqrt{\zeta_\kappa(\bar{r}_s)}) \langle  -g_s, \exp_{x_s}^{-1}(x_\star)\rangle_{x_s} \le \bar{r}_t \left(2 \bar{d}_t + \frac{\bar{r}_t \zeta_\kappa(\bar{d}_t)}{\zeta_\kappa(\bar{r}_t)}\right)  \sqrt{G_{t-1}}.
   \end{align}
   \begin{proof}
       Applying \cref{cor:law_of_cosines}, we can bound the weighted average as
   \begin{align}
      \sum_{s=0}^{t-1} \left(\bar{r}_s/\sqrt{\zeta_\kappa(\bar{r}_s)}\right) \langle  -g_s, \exp_{x_s}^{-1}(x_\star)\rangle_{x_s} \le \frac{1}{2} \underset{(A)}{\underbrace{\sum_{s=0}^{t-1}  \frac{\left(\bar{r}_s/\sqrt{\zeta_\kappa(\bar{r}_s)}\right)}{\eta_s}\left(d_s^2 - d_{s+1}^2\right)}} + \frac{1}{2} \underset{(B)}{\underbrace{\sum_{s=0}^{t-1} \left(\bar{r}_s/\sqrt{\zeta_\kappa(\bar{r}_s)}\right) \eta_s \zeta_\kappa(d_s) \lVert g_s \rVert_{x_s}^2}}.
   \end{align}
We bound the terms $(A)$ and $(B)$ in turn, beginning with the former:
    \begin{align}
       (A) &= \sum_{s=0}^{t-1} \sqrt{G_s} \left(d_s^2 - d_{s+1}^2 \right) = d_0^2 \sqrt{G_0} - d_t^2 \sqrt{G_{t-1}} + \sum_{s=0}^{t-1} d_s^2 \left( \sqrt{G_s} - \sqrt{G_{s-1}} \right) \\ 
       &\overset{(i)}{\le} {\bar{d}_t}^2 \sqrt{G_0} - d_t^2 \sqrt{G_{t-1}} +  \bar{d}_t^2 \sum_{s=0}^{t-1} \left( \sqrt{G_s} -  \sqrt{G_{s-1}} \right) = \sqrt{G}_{t-1} (\bar{d}_t^2 - d_t^2)
       \overset{(ii)}{\le} 4 \bar{r}_t \bar{d}_t \sqrt{G_{t-1}}.
    \end{align}
Inequality $(i)$ uses $d_s \le \bar{d}_t$ and that $G_t$ is nondecreasing. Inequality $(ii)$ use that for $k \in \arg \max_{s \le t} d_s$, we have $\bar{d}_t^2 - d_t^2 = d_k^2 - d_t^2 = (d_k - d_t) (d_k + d_t)  \le d(x_k, x_t) (d_k + d_t) \le (\bar{r}_k + \bar{r}_t)(d_k + d_t) \le 4 \bar{r}_t \bar{d}_t$.
Bounding the second term $(B)$, we have for $\kappa < 0$:
    \begin{align}
       (B) = \sum_{s=0}^{t-1} \frac{\bar{r}_s^2 \zeta_\kappa(d_s) \lVert g_s \rVert_{x_s}^2}{\zeta_\kappa({\bar{r}_s)} \sqrt{G_s}} = \sum_{s=0}^{t-1} \frac{\bar{r}_s^2 \tanh(\sqrt{|\kappa|} \cdot \bar{r}_s) \zeta_\kappa(d_s) \lVert g_s \rVert_{x_s}^2}{\sqrt{|\kappa|} \cdot \bar{r}_s \sqrt{G_s}} = \sum_{s=0}^{t-1} \frac{\bar{r}_s \tanh(\sqrt{|\kappa|} \cdot \bar{r}_s) \zeta_\kappa(d_s) \lVert g_s \rVert_{x_s}^2}{\sqrt{|\kappa|} \cdot \sqrt{G_s}}
       \\ \le \frac{1}{\sqrt{|\kappa|}}  \bar{r}_t \tanh(\sqrt{|\kappa|} \cdot \bar{r}_t) \zeta_\kappa(\bar{d}_t) \sum_{s=0}^{t-1} \frac{\lVert g_s \rVert_{x_s}^2}{\sqrt{G_s}}
       \le \frac{2}{\sqrt{|\kappa|}} \bar{r}_t \tanh(\sqrt{|\kappa|}\bar{r}_t) \zeta_\kappa(\bar{d}_t) \sqrt{G_{t-1}} \\ = \frac{2\bar{r}_t^2 \tanh{(\sqrt{|\kappa|} \cdot \bar{r}_t})}{\sqrt{|\kappa|} \cdot \bar{r}_t} \zeta_\kappa(\bar{d}_t) \sqrt{G_{t-1}} = \frac{2\bar{r}_t^2}{\zeta_\kappa(\bar{r}_t)} \zeta_\kappa(\bar{d}_t) \sqrt{G_{t-1}}.
\end{align}
While for $\kappa=0$ the geometric curvature function $d \mapsto \zeta_\kappa(d)$ takes constant value one, thus the same bound above can be established trivially. Combining $(A)$ and $(B)$, gives the result.
   \end{proof}
\end{lemma}

\begin{lemma} \label{lemma:noise_bound}
   For all $\delta \in (0,1), T \in \mathbb{N}$ and $L > 0$, if \cref{assumption:g_convex}, \ref{assumption:Hadamard}, and \ref{assumption:bounded_stochastic_grads} hold, the iterates of RDoG (\cref{alg:rdog}) satisfy
    \begin{align}
       \mathbb{P} \left( \exists t \le T : \left\lvert  \sum_{s=0}^{t-1} \frac{\bar{r}_s}{\sqrt{\zeta_\kappa(\bar{r}_s)}} \langle -\Delta_s, \exp_{x_s}^{-1}(x_\star)\rangle_{x_s} \right\rvert \ge b_t \right) \le \delta + \mathbb{P}(\bar{\ell}_{T} > L),
   \end{align}
   where $b_t = 8 \frac{\bar{r}_{t-1}} {\sqrt{\zeta_\kappa(\bar{r}_{t-1})}} \bar{d}_{t-1} \sqrt{\theta_{t, \delta} G_{t-1} + \theta^{2}_{t, \delta} L^2}$ and $\bar{\ell}_{T} \coloneqq \max_{s \le T} \ell(x_s)$.
\begin{proof}
       For $1\le s \le T$ define the random variables
\begin{align}
    Y_s \coloneqq \frac{\bar{r}_{s-1}} {\sqrt{\zeta_\kappa(\bar{r}_{s-1})}} \bar{d}_{s-1}, \quad X_s \coloneqq \left\langle \Delta_{s-1}, \frac{\exp_{x_{s-1}}^{-1}(x_\star)}{\bar{d}_{s-1}}\right\rangle_{x_{s-1}}, \quad \hat{X}_s \coloneqq \left\langle -\operatorname{grad}f(x_{s-1}), \frac{\exp_{x_{s-1}}^{-1}(x_\star)}{\bar{d}_{s-1}}\right\rangle_{x_{s-1}}.
\end{align}
By the Cauchy-Schwartz inequality and \cref{assumption:bounded_stochastic_grads}, we have each $\lvert X_s \rvert \le \ell(x)$, and each $\lvert\hat{X}_s\rvert \le \ell(x)$ with probability 1. Moreover, and consider the filtration $\mathcal{F}_s = \sigma(\mathcal{G}(x_0) \dotsc, \mathcal{G}(x_s))$. Then we have that $X_s$ is a martingale difference sequence adapted to $\mathcal{F}_s$ and $\hat{X}_s \in \mathcal{F}_{s-1}$. By construction or any $t\le T$, we have
\begin{align}
    \sum_{s=1}^{t} Y_s X_s = \sum_{s=0}^{t-1} \frac{\bar{r}_{s}} {\sqrt{\zeta_\kappa(\bar{r}_{s})}} \langle \Delta_s, \exp_{x_s}^{-1}(x_\star) \rangle_{x_s}.
\end{align}
Therefore, applying \cref{lemma:martingale_bound} yields the result as required.
   \end{proof}
\end{lemma}

Combining the above results, we obtain the following.

\begin{theorem}
    \label{thm:opt_gap_dog_known}
    For all $\delta \in (0, 1)$ and $L>0$, if  \cref{assumption:g_convex}, \ref{assumption:Hadamard}, and \ref{assumption:bounded_stochastic_grads} hold, then with probability at least $1 - \delta - \mathbb{P}(\bar{\ell}_{T} > L)$, for all $t \le T$, the optimality gap on the weighted iterates $f(\tilde{x}_t) - f(x_\star)$ of RDoG (\cref{alg:rdog}) satisfy
    \begin{align}
        O \left(  \frac{(d_0 + \bar{r}_t)\sqrt{\zeta_{\kappa}(d_0 + \bar{r}_t)}\sqrt{ G_{t-1} + \theta_{t, \delta} G_{t-1} + \theta^{2}_{t, \delta} L^2}}{\sum_{s=0}^{t-1} \frac{\bar{r}_s /\sqrt{\zeta_\kappa(\bar{r}_s)}}{\bar{r}_t /\sqrt{\zeta_\kappa(\bar{r}_t)}}}\right).
    \end{align}

    \begin{proof}
        Combining \cref{lemma:weighted_regret_bound} and \cref{lemma:noise_bound}, we have for the given probability that
        \begin{align}
            f(\tilde{x}_t) - f(x_\star) \le  \frac{\left(2 \bar{d}_t \sqrt{\zeta_{\kappa}(\bar{r}_t)} + \frac{\bar{r}_t}{\sqrt{\zeta_{\kappa}(\bar{r}_t)}} \zeta_{\kappa}(\bar{d}_t)\right) \sqrt{G_{t-1}} + 8 \bar{d}_{t-1} \sqrt{\theta_{t, \delta} G_{t-1} + \theta^{2}_{t, \delta} L^2}}{\sum_{s=0}^{t-1} \frac{\bar{r}_s /\sqrt{\zeta_\kappa(\bar{r}_s)}}{\bar{r}_t /\sqrt{\zeta_\kappa(\bar{r}_t)}}}.
        \end{align}
    Now using the fact $\bar{d}_t \le d_0 + \bar{r}_t$ and that $d \mapsto \zeta_\kappa(d)$ and $d \mapsto \frac{d}{\sqrt{\zeta_\kappa(d)}}$ are increasing functions gives the result.
    \end{proof}
\end{theorem}

We then have a useful result when the manifold is bounded but its exact diameter is unknown.
\begin{corollary}
\label{cor:rdog_opt_gap}
    Under \cref{assumption:g_convex}, \ref{assumption:Hadamard}, and \ref{assumption:bounded_stochastic_grads}, for any $D \ge d_0$ let $L_{D} \coloneqq \max_{x\in \mathcal{M} : d(x, x_0) \le D} \ell(x)$. Then, for all $\delta \in (0, 1)$ and for $\tau \in \arg \max_{t \le T} \sum_{s=0}^{t-1} \frac{\bar{r}_s /\sqrt{\zeta_\kappa(\bar{r}_s)}}{\bar{r}_t /\sqrt{\zeta_\kappa(\bar{r}_t)}}$, with probability at least $1 - \delta - \mathbb{P}(\bar{\ell}_{T} > L)$, iterates of RDoG (\cref{alg:rdog}) satisfy the optimality gap bound
    \begin{align}
        f(\tilde{x}_{\tau}) - f(x_\star) = O\left(\frac{D\sqrt{\zeta_{\kappa}(D)} \sqrt{G_{\tau-1} \theta_{\tau, \delta} + L^2_D \theta_{\tau, \delta}^2}}{T} \log_{+}\left(\frac{D/\sqrt{\zeta_{\kappa}(D)}}{\epsilon/\sqrt{\zeta_{\kappa}(\epsilon)}}\right)\right).
    \end{align}
    \begin{proof}
        Apply \cref{lemma:fraction_sum} to the denominator term of  \cref{thm:opt_gap_dog_known}.
    \end{proof}
 \end{corollary}

 \subsection{Smooth Guarantees via Uniform Averaging}
 \label{sec:rdog-uniform-averaging}

 Under the assumption of locally smooth stochastic gradients (\cref{assumption:bounded_stochastic_lsmooth}), we can deduce an $O(1/T)$ convergence guarantee under uniformly averaged iterates,
\begin{align*}
        \hat{x}_{t+1} = \operatorname{exp}_{\hat{x}_{t}}\left(\frac{1}{t} \operatorname{exp}_{\hat{x}_{t}}^{-1}\left( x_{t} \right) \right), \quad \hat{x}_{1} = x_{0}.
\end{align*}
We begin by presenting a theorem that shows a bound under uniform iterate averaging in the non-smooth setting.

\begin{theorem}
    \label{thm:opt_gap_dog_unknown_uniform}
    For all $\delta \in (0, 1)$ and $L>0$, if  \cref{assumption:g_convex}, \ref{assumption:Hadamard}, and \ref{assumption:bounded_stochastic_grads} hold, then with probability at least $1 - \delta - \mathbb{P}(\bar{\ell}_{T} > L)$, for all $t \le T$, the optimality gap on the uniformly averaged iterates $f(\hat{x}_T) - f(x_\star)$ of RDoG (\cref{alg:rdog}) satisfy:
    \begin{align} 
            O\left(\frac{(d_0 \log_{+}\frac{\bar{r}_T}{\epsilon} + \bar{r}_T)\sqrt{\zeta_{\kappa}(d_0 + \bar{r}_T)} \sqrt{G_{T-1} + \theta_{T, \delta} G_{T-1} + \theta^{2}_{T, \delta} L^2}}{T} \right).
    \end{align}
    
    \begin{proof}
        Define the times $\tau_s = \min \left\{ \min \{k | \bar{r}_k \ge 2 \bar{r}_{\tau_{k-1}} \}, T \right\}$, with $\tau_0 \coloneqq 0$. Moreover, let $K$ be the first index such that $\tau_{K} = T$ and note that $K \le 1 + \log_2 \frac{\bar{r}_T}{\epsilon}$ by construction. Now using the argument of \cref{lemma:weighted_regret_bound_unknown_curvature}, we have that for $k \le K$
        \begin{align}
            \sum_{t=\tau_{k-1}}^{\tau_k-1} \frac{\bar{r}_t}{\sqrt{\zeta_\kappa(\bar{r}_t)}} \langle -g_t, \exp_{x_t}^{-1} (x_{\star}) \rangle &\le \bar{r}_{\tau_k} \left(2 \bar{d}_{\tau_k} + \frac{\bar{r}_{\tau_k}}{\zeta_\kappa(\bar{r}_{\tau_k})} \zeta_\kappa(\bar{d}_{\tau_k})\right) \sqrt{G_{\tau_k-1}} 
            \\ &= O\left(\bar{r}_{\tau_k}\frac{(d_0 + \bar{r}_{\tau_k})}{\zeta_{\kappa}(d_0 + \bar{r}_{\tau_k})} \zeta_{\kappa}(d_0 + \bar{r}_{\tau_k}) \sqrt{G_{T-1}}\right)
            \\ &= O\left(\bar{r}_{\tau_k}(d_0 + \bar{r}_{\tau_k}) \sqrt{G_{T-1}}\right).
        \end{align}
        Where the first equality holds due by the virtue of $d \mapsto \frac{d}{\zeta_\kappa(d)}$ is an increasing function and that $\bar{d}_{\tau_k} \le \bar{r}_{\tau_k} + d_0 $. Furthermore, by  \cref{lemma:noise_bound_unknown_curvature}
        we have for all $k \le K$ with probability at least $1 - \delta - \mathbb{P}(\bar{\ell}_T > L)$, 
        \begin{align}
            \left\lvert \sum_{t=\tau_{k-1}}^{\tau_k-1} \frac{\bar{r}_t}{\sqrt{\zeta_\kappa(\bar{r}_t)}} \langle \Delta_t, \exp_{x_t}^{-1}(x_\star) \rangle_{x_t} \right\rvert &\le \left\lvert \sum_{t=0}^{\tau_k-1} \frac{\bar{r}_t}{\sqrt{\zeta_\kappa(\bar{r}_t)}} \langle \Delta_t, \exp_{x_t}^{-1}(x_\star)  \rangle_{x_t} \right\rvert + \left\lvert \sum_{t=0}^{\tau_{k-1}-1} \frac{\bar{r}_t}{\sqrt{\zeta_\kappa(\bar{r}_t)}} \langle \Delta_t, \exp_{x_t}^{-1}(x_\star)  \rangle_{x_t} \right\rvert \\ &\le 16 \frac{\bar{r}_{\tau_{k}-1}}{\sqrt{\zeta_\kappa(\bar{r}_{\tau_{k}-1})}} \bar{d}_{\tau_{k}-1} \sqrt{\theta_{T, \delta} G_{T-1} + \theta^2_{T, \delta} L^2}.
        \end{align}
        Now combining these two bounds, we have
        \begin{align}
            \sum_{t=\tau_{k-1}}^{\tau_k-1}f(x_t) - f(x_\star) &\le \frac{1}{\bar{r}_{\tau_{k-1}} / \sqrt{\zeta_\kappa(\bar{r}_{\tau_{k-1}})}} \sum_{t=\tau_{k-1}}^{\tau_k-1} \frac{\bar{r}_t}{\sqrt{\zeta_\kappa(\bar{r}_t)}} [f(x_t) - f(x_\star)] \\ &\le \frac{1}{\bar{r}_{\tau_{k-1}} / \sqrt{\zeta_\kappa(\bar{r}_{\tau_{k-1}})}} \sum_{t=\tau_{k-1}}^{\tau_k-1} \frac{\bar{r}_t}{\sqrt{\zeta_\kappa(\bar{r}_t)}} \langle - \operatorname{grad}f(x_t), \exp_{x_t}^{-1} (x_\star) \rangle_{x_t} \\
            & =\frac{1}{\bar{r}_{\tau_{k-1}} / \sqrt{\zeta_\kappa(\bar{r}_{\tau_{k-1}})}} \sum_{t=\tau_{k-1}}^{\tau_k-1} \frac{\bar{r}_t}{\sqrt{\zeta_\kappa(\bar{r}_t)}} \left[\langle - g_t, \exp_{x_t}^{-1} (x_\star) \rangle_{x_t} + \langle \Delta_t, \exp_{x_t}^{-1} (x_\star) \rangle_{x_t} \right] \\
            &= O \left(\frac{\bar{r}_{\tau_{k}} / \sqrt{\zeta_\kappa(\bar{r}_{\tau_{k}})} }{\bar{r}_{\tau_{k-1}} / \sqrt{\zeta_\kappa(\bar{r}_{\tau_{k-1}})}} (d_0 + \bar{r}_{\tau_{k}} ) \sqrt{\zeta_{\kappa}(d_0 + \bar{r}_{\tau_k})} \sqrt{G_{T-1} + \theta_{T, \delta} G_{T-1} + \theta^2_{T, \delta} L^2} \right) \\
            & = O \left((d_0 + \bar{r}_{\tau_{k}} ) \sqrt{\zeta_{\kappa}(d_0 + \bar{r}_{\tau_k})} \sqrt{G_{T-1} + \theta_{T, \delta} G_{T-1} + \theta^2_{T, \delta} L^2} \right),
        \end{align}
        where final reduction holds since $d \mapsto \frac{d}{\zeta_\kappa(d)}$ is an increasing function, and for any $t$,
        \begin{align}
            \bar{r}_{t+1} \le \bar{r}_{t} + d(x_{t+1}, x_t) = \bar{r}_{t}\left(1 + \frac{\lVert g_t \rVert_{x_t}}{\sqrt{G_t}} \right) \le 2 \bar{r}_t.
        \end{align}
        Now summing over $k$ from $1$ to $K$ we have
        \begin{align}
            \sum_{t=0}^{T-1}[f(x_t) - f(x_\star)] &= \sum_{k=1}^{K} \sum_{t=\tau_{k-1}}^{\tau_k-1} [f(x_t) - f(x_\star)] \\ &= O \left(\sum_{k=1}^{K}(d_0 + \bar{r}_{\tau_{k}} ) \sqrt{\zeta_{\kappa}(d_0 + \bar{r}_{\tau_k})} \sqrt{G_{T-1} + \theta_{T, \delta} G_{T-1} + \theta^2_{T, \delta} L^2} \right) \\
            &\le O \left(\sum_{k=1}^{K}(d_0 + \bar{r}_{\tau_{k}} ) \sqrt{\zeta_{\kappa}\left( d_0 + \sum_{k=1}^{K}\bar{r}_{\tau_k}\right)} \sqrt{G_{T-1} + \theta_{T, \delta} G_{T-1} + \theta^2_{T, \delta} L^2} \right).
        \end{align}
    \end{proof}
Where the final reduction holds since $d \mapsto \zeta_\kappa(d)$ is an increasing function. Now, recall that $K = O\left(  \log_{+} \frac{\bar{r}_T}{\epsilon} \right)$ and note that $\sum_{k=1}^{K} \bar{r}_{\tau_k} = O(\bar{r}_T)$ since $\frac{\bar{r}_{\tau_s}}{\bar{r}_{\tau_{K-1}}} \le 2^{s-(K-1)}$ for all $s \le K-1$. The proof is complete noting that via \cref{lemma:weighted_average} we deduce $f(\hat{x}_T) \le \frac{1}{T} \sum_{t=0}^{T-1} f(x_t)$.
\end{theorem}

Now under smooth assumption, we present a result for bounding the stochastic term that depends on $S$.
\begin{lemma}   \label{lemma:uniform_dog_smooth_prob_bound_unknown_curvature}
   For all $\delta \in (0,1), T \in \mathbb{N}$ and $S > 0$, if \cref{assumption:g_convex}, \ref{assumption:Hadamard}, and \ref{assumption:bounded_stochastic_lsmooth} hold, then the iterates of RDoG (\cref{alg:rdog}) satisfy
    \begin{align}
       \mathbb{P} \left( \exists t \le T : \left\lvert  \sum_{s=0}^{t-1} \frac{\bar{r}_s}{\sqrt{\zeta_\kappa(\bar{r}_s)}} \langle \Delta_s, \exp_{x_s}^{-1}(x_\star)\rangle_{x_s} \right\rvert \ge b_t \right) \le \delta + \mathbb{P}(\bar{s}_{T} > S),
   \end{align}
   where $b_t =  8 \frac{\bar{r}_{t-1}}{\sqrt{\zeta_\kappa(\bar{r}_{t-1})}} \bar{d}_{t-1} \sqrt{2S \theta_{t, \delta} + 8S \theta_{t, \delta}^2} \sqrt{\sum_{s=0}^{t-1} \left[f(x_s) - f(x_\star)\right]}$ and $\bar{s}_{T} \coloneqq \max_{k \le T} s(x_k)$.
   
   \begin{proof}
       For $1 \le s \le T$ define the random variables
    \begin{align}
        Y_s \coloneqq \frac{\bar{r}_{s-1} \bar{d}_{s-1}}{\sqrt{\zeta_\kappa(\bar{r}_{s-1})}}, \quad X_s \coloneqq \left\langle \Delta_{s-1}, \frac{\exp_{x_{s-1}}^{-1}(x_\star)}{\bar{d}_{s-1}}\right\rangle_{x_{s-1}}, \quad \hat{X}_s \coloneqq \left\langle -\operatorname{grad} f(x_{s-1}), \frac{\exp_{x_{s-1}}^{-1}(x_\star)}{\bar{d}_{s-1}}\right\rangle_{x_{s-1}},
    \end{align}
    and consider the filtration $\mathcal{F}_s = \sigma(\mathcal{G}(x_0) \dotsc, \mathcal{G}(x_s))$. Then we have that $X_s$ is a martingale difference sequence adapted to $\mathcal{F}_s$ and $\hat{X}_s \in \mathcal{F}_{s-1}$. By construction or any $t\le T$, we have
    \begin{align}
        \sum_{s=1}^{t} Y_s X_s = \sum_{s=0}^{t-1} \bar{r}_s^2 \langle \Delta_s, \exp_{x_s}^{-1}(x_\star) \rangle_{x_s}.
    \end{align}
    Now we consider bounding, $\max\{\lvert X_t \rvert, \lvert\hat{X}_t\rvert\}$ by a constant $c$. Moreover, by the Cauchy-Schwartz inequality and \cref{lemma:bound_stochastic_error_smooth} we have with probability $\mathbb{P}(\bar{s}_{T} > S)$ we have
    \begin{align}
        |X_s|^2 &\le  \lVert \Delta_{s-1} \rVert^2_{x_{s-1}} \cdot 1 \le  8S (f(x_{s-1}) - f(x_\star)) \\
        |\hat{X}_s|^2 &\le  \lVert \operatorname{grad} f(x_{s-1}) \rVert^2_{x_{s-1}} \cdot 1 \le  8S (f(x_{s-1}) - f(x_\star)).
    \end{align}
    Thus we have that,
    \begin{align}
        |X_t| \le \sqrt{\sum_{s=1}^{t} |X_s|^2} \le \sqrt{8S}  \sqrt{\sum_{s=0}^{t-1} \left[f(x_s) - f(x_\star)\right]} =: c \\
        |\hat{X}_t| \le \sqrt{\sum_{s=1}^{t} |\hat{X}_s|^2} \le \sqrt{8S} \sqrt{\sum_{s=0}^{t-1} \left[f(x_s) - f(x_\star)\right]} =: c .\\
    \end{align}
    Therefore, applying \cref{lemma:martingale_bound} yields,
    \begin{align}
        \mathbb{P} \left( \exists t \le T: \left\lvert \sum_{s=0}^{t-1} \frac{\bar{r}_s}{\zeta_\kappa(\bar{r}_s)} \langle \Delta_s, \exp_{x_s}^{-1}(x_\star) \rangle_{x_s} \right\rvert \ge 8 \frac{\bar{r}_{t-1}}{\sqrt{\zeta_\kappa(\bar{r}_{t-1})}} \bar{d}_{t-1} \sqrt{\theta_{t, \delta} \sum_{s=1}^{t} \left( X_s - \hat{X}_s\right)^2 + c^2\theta_{t, \delta} ^2} \right) \le \delta.
    \end{align}
    Now, finally noting 
    \begin{align}
    \sum_{s=1}^{t} (X_s - \hat{X}_s)^2 \le \sum_{s=0}^{t-1}  \rVert \operatorname{grad} f(x_s) \rVert^2_{x_s} \le 2S\sum_{s=0}^{t-1} (f(x_s) - f(x_\star)),
    \end{align}
    yields the result.
   \end{proof}
\end{lemma}

\begin{theorem}
    \label{thm:opt_gap_dog_unknown_uniform}
    For all $\delta \in (0, 1)$ and $S>0$, if  \label{thm:dowg_smooth_known_curvature} \cref{assumption:g_convex}, \ref{assumption:Hadamard}, and \ref{assumption:bounded_stochastic_lsmooth} hold, then with probability at least $1 - \delta - \mathbb{P}(\bar{s}_{T} > S)$, for all $t \le T$, the optimality gap of the uniformly averaged iterates $f(\hat{x}_T) - f(x_\star)$ of RDoG (\cref{alg:rdog}) satisfy:
    \begin{align} 
            O\left(\frac{(d_0 \log_{+}\frac{\bar{r}_T}{\epsilon} + \bar{r}_T)^2\zeta_{\kappa}(d_0 + \bar{r}_T) \theta^{2}_{T, \delta} S}{T} \right).
    \end{align}
    
    \begin{proof}
        Similar to the non-smooth setting, define the times $\tau_s = \min \left\{ \min \{k | \bar{r}_k \ge 2 \bar{r}_{\tau_{k-1}} \}, T \right\}$, with $\tau_0 \coloneqq 0$. Moreover, let $K$ be the first index such that $\tau_{K} = T$ and note that $K \le 1 + \log_2 \frac{\bar{r}_T}{\epsilon}$ by construction. Now using the argument of \cref{lemma:weighted_regret_bound_unknown_curvature}, we have that for $k \le K$
        \begin{align}
            \sum_{t=\tau_{k-1}}^{\tau_{k}-1} \frac{\bar{r}_t}{\sqrt{\zeta_\kappa(\bar{r}_t)}} \langle -g_t, \exp_{x_t}^{-1} (x_{\star}) \rangle &\le \bar{r}_{\tau_k} \left(2 \bar{d}_{\tau_k} + \frac{\bar{r}_{\tau_k}}{\zeta_\kappa(\bar{r}_{\tau_k})}\zeta_\kappa(\bar{d}_{\tau_k}) \right) \sqrt{G_{\tau_k - 1}} \\ &= O\left(\bar{r}_{\tau_k}(d_0 + \bar{r}_{\tau_k})  \sqrt{G_{T-1}}\right) \\
            &\le O\left(\bar{r}_{\tau_k}(d_0 + \bar{r}_{\tau_k}) \sqrt{2S \sum_{t=0}^{T-1} [f(x_t) - f(x_\star)]}\right).
        \end{align}
        Where in the final inequality we have applied \cref{lemma:lsmooth_bound}. Now, by  \cref{lemma:uniform_dog_smooth_prob_bound_unknown_curvature}
        we have for all $k \le K$ with probability at least $1 - \delta - \mathbb{P}(\bar{s}_T > L)$, 
        \begin{align}
            \left\lvert \sum_{t=\tau_{k-1}}^{\tau_k-1} \frac{\bar{r}_t}{\sqrt{\zeta_\kappa(\bar{r}_t)}} \langle \Delta_t, \exp_{x_t}^{-1}(x_\star) \rangle_{x_t} \right\rvert &\le \left\lvert \sum_{t=0}^{\tau_k-1} \frac{\bar{r}_t}{\sqrt{\zeta_\kappa(\bar{r}_t)}} \langle \Delta_t, \exp_{x_t}^{-1}(x_\star)  \rangle_{x_t} \right\rvert + \left\lvert \sum_{t=0}^{\tau_{k-1}-1} \frac{\bar{r}_t}{\sqrt{\zeta_\kappa(\bar{r}_t)}}\langle \Delta_t, \exp_{x_t}^{-1}(x_\star)  \rangle_{x_t} \right\rvert \\ &\le 16 \frac{\bar{r}_{\tau_{k}-1}}{\sqrt{\zeta_\kappa(\bar{r}_{\tau_{k}-1})}} \bar{d}_{\tau_{k}-1} \sqrt{2S \theta_{T, \delta} + 8S \theta_{T, \delta}^2} \sqrt{\sum_{t=0}^{T-1} \left[f(x_t) - f(x_\star)\right]}.
        \end{align}
        Now combining these two bounds, we have
        \begin{align}
            \sum_{t=\tau_{k-1}}^{\tau_k-1}f(x_t) - f(x_\star) &\le \frac{1}{\bar{r}_{\tau_{k-1}} / \sqrt{\zeta_\kappa(\bar{r}_{\tau_{k-1}})}} \sum_{t=\tau_{k-1}}^{\tau_k-1} \frac{\bar{r}_t}{\sqrt{\zeta_\kappa(\bar{r}_t)}} [f(x_t) - f(x_\star)] \\ &\le \frac{1}{\bar{r}_{\tau_{k-1}} / \sqrt{\zeta_\kappa(\bar{r}_{\tau_{k-1}})}} \sum_{t=\tau_{k-1}}^{\tau_k-1} \frac{\bar{r}_t}{\sqrt{\zeta_\kappa(\bar{r}_t)}} \langle - \operatorname{grad}f(x_t), \exp_{x_t}^{-1} (x_\star) \rangle_{x_t} \\
            & = \frac{1}{\bar{r}_{\tau_{k-1}} / \sqrt{\zeta_\kappa(\bar{r}_{\tau_{k-1}})}} \sum_{t=\tau_{k-1}}^{\tau_k-1} \frac{\bar{r}_t}{\sqrt{\zeta_\kappa(\bar{r}_t)}} \left[ \langle - g_t, \exp_{x_t}^{-1} (x_\star) \rangle_{x_t} +  \langle \Delta_t, \exp_{x_t}^{-1}  (x_\star) \rangle_{x_t} \right]\\
            &= O \left(\frac{\bar{r}_{\tau_{k}} /  \sqrt{\zeta_\kappa(\bar{r}_{\tau_{k}})}}{\bar{r}_{\tau_{k-1}} / \sqrt{\zeta_\kappa(\bar{r}_{\tau_{k-1}})}} (d_0 + \bar{r}_{\tau_{k}} ) \sqrt{\zeta_{\kappa}(d_0 + \bar{r}_{\tau_k})} \sqrt{S \theta_{T, \delta}^2} \sqrt{\sum_{t=0}^{T-1} \left[f(x_t) - f(x_\star)\right]} \right) \\
            & = O \left((d_0 + \bar{r}_{\tau_{k}} ) \sqrt{\zeta_{\kappa}(d_0 + \bar{r}_{\tau_k})} \sqrt{S \theta_{T, \delta}^2} \sqrt{\sum_{t=0}^{T-1}\left[f(x_t) - f(x_\star)\right]}  \right).
        \end{align}
        Where final reduction holds since $d \mapsto \frac{d}{\zeta_{\kappa}(d)}$,and for any $t$,
        \begin{align}
            \bar{r}_{t+1} \le \bar{r}_{t} + d(x_{t+1}, x_t) = \bar{r}_{t}\left(1 + \frac{\lVert g_t \rVert_{x_t}}{\sqrt{G_t}} \right) \le 2 \bar{r}_t.
        \end{align}
        Now summing over $k$ from $1$ to $K$ we have
        \begin{align}
            \sum_{t=0}^{T-1} [f(x_t) - f(x_\star)] &= \sum_{k=1}^{K} \sum_{t=\tau_{k-1}}^{\tau_k-1} [f(x_t) - f(x_\star)] \\ &= O \left(\sum_{k=1}^{K}(d_0 + \bar{r}_{\tau_{k}} ) \sqrt{\zeta_{\kappa}(d_0 + \bar{r}_{\tau_k})} \sqrt{S \theta_{T, \delta}^2} \sqrt{\sum_{t=0}^{T-1} \left[f(x_t) - f(x_\star)\right]}  \right) \\
            &\le O \left(\sum_{k=1}^{K}(d_0 + \bar{r}_{\tau_{k}} ) \sqrt{\zeta_{\kappa}\left( d_0 + \sum_{k=1}^{K} \bar{r}_{\tau_k}\right)} \sqrt{S \theta_{T, \delta}^2} \sqrt{\sum_{t=0}^{T-1} \left[f(x_t) - f(x_\star)\right]}  \right).
        \end{align}
    Where the final reduction holds since $d \mapsto \zeta_\kappa(d)$ is an increasing function. Now, recall that $K = O\left(  \log_{+} \frac{\bar{r}_T}{\epsilon} \right)$ and note that $\sum_{k=1}^{K} \bar{r}_{\tau_k} = O(\bar{r}_T)$ since $\frac{\bar{r}_{\tau_s}}{\bar{r}_{\tau_{K-1}}} \le 2^{s-(K-1)}$ for all $s \le K-1$. We then divide both sides through by $\sqrt{\sum_{t=0}^{T-1} \left[f(x_t) - f(x_\star)\right]}$, to yield
    \begin{align}
            \sqrt{\sum_{t=0}^{T-1} [f(x_t) - f(x_\star)]} 
            &\le O \left((d_0{\bar{r}_T}/{\epsilon} + \bar{r}_{T} ) \sqrt{\zeta_{\kappa}\left( d_0 + \bar{r}_{T}\right)} \sqrt{S \theta_{T, \delta}^2} \right),
        \end{align}
    squaring both sides, the proof is complete noting that via \cref{lemma:weighted_average} we deduce $f(\hat{x}_T) \le \frac{1}{T} \sum_{t=0}^{T-1} f(x_t)$.
    \end{proof}
\end{theorem}

 \subsection{Iterate Stability Bound}
 \label{sec:rdog-stability-analysis}
 
We introduce \emph{Tamed Riemannian Distance over Gradients} (T-RDoG), a dampened version of RDoG (\cref{alg:rdog}) whose iterates are guaranteed to remain bounded with high probability. T-RDoG has the following step size scheme
\begin{align}
    \eta_t = \frac{\bar{r}_t}{\sqrt{\zeta_{\kappa}(\bar{r}_t)G_t^\prime}}, \quad G_t^\prime = 8^4 \theta_{T, \delta}^2 \log_{+}^2\left(\frac{(1+t)\bar{\ell}_t^2}{\bar{\ell}_0^2}\right)(G_{t-1} + 16 \bar{\ell}^2_t),
\end{align}
using $G_{-1} \coloneqq 0$ and recalling $\bar{\ell}_t \coloneqq \max_{s\le t} \ell(x_s)$ for a function $\ell$ satisfying \cref{assumption:bounded_stochastic_grads}. To show iterate boundedness in the stochastic setting, we consider the stopping time
\begin{align}
    \mathcal{T}_{out} = \min \{t \ge 0 : \bar{r}_t > 3d_0\},
\end{align}
so that the event $\{\bar{r}_T \le 3 d_0\}$ is the same as $\{\mathcal{T}_{out}  > T\}$.  We also define the following truncated step size sequence,
\begin{align}
    \tilde{\eta}_k \coloneqq \eta_k \mathbb{I}_{\{k < \mathcal{T}_{out}\}}.
\end{align}
Truncating as such allows us to handle the possibility that $\bar{r}_T$ exceeds $3d_0$. In particular, the following holds for $\{\tilde{\eta}_k\}$ but not for $\{{\eta}_k\}$.
\begin{lemma} For all $t \le T$, if \cref{assumption:g_convex}, \ref{assumption:Hadamard}, and \ref{assumption:bounded_stochastic_grads} hold, under the truncated T-RDoG step size sequence $\{\tilde{\eta}_t\}$, the iterates satisfy,
        \begin{align}
            \tilde{\eta}_t &\in \sigma(\mathcal{G}(x_0), \dotsc, \mathcal{G}(x_{t-1})), \\
            \lvert \tilde{\eta}_t \langle \gamma, \exp_{x_t}^{-1} (x_\star)\rangle_{x_t} \rvert &\le \frac{6 d_0^2}{8^2 \sqrt{\zeta_\kappa(3 d_0)} \theta_{T, \delta}}  \ \text{for} \ \gamma \in \{g_t, \operatorname{grad} f(x_t), \Delta_t \}, \\
            \sum_{k=0}^{t} \tilde{\eta}_k^2 \zeta_\kappa(d_k) \lVert g_k \rVert^2_{x_k} &\le \frac{12 d_0^2}{8^4 \theta_{T, \delta}}
            , \ \text{and} \\
            \sum_{k=0}^{t} (\tilde{\eta}_k \langle g_k, \exp_{x_k}^{-1}(x_\star) \rangle_{x_k})^2 &\le \frac{3 \cdot 4^3 d_0^4}{8^4 \theta_{T, \delta}}.
        \end{align}
    \begin{proof}
The first line holds directly by definition of the truncated T-RDoG iterates. For bound in the second line, note (recalling $\Delta_t = g_t - \operatorname{grad} f(x_t)$ we have $\lVert \Delta_t \rVert_{x_t} \le \lVert g_t \rVert_{x_t} + \lVert\operatorname{grad} f(x_t) \rVert_{x_t}  \le 2 \ell(x_t)$. Since $G^\prime_t \ge 4^28^4\ell^2(x_t) \theta^2_{T, \delta}$ for all $t$, the Cauchy-Schwartz inequality gives,
\begin{align}
    \lvert \tilde{\eta}_t \langle \Delta_t, \exp_{x_t}^{-1} (x_\star)\rangle_{x_t} \rvert \le \frac{\bar{r}_t}{\sqrt{\zeta_\kappa(\bar{r}_t) G_t^\prime}} \lVert \Delta_t \lVert_{x_t} d_t \le \frac{1}{2 \cdot 8^2 \theta_{T, \delta}} \frac{\bar{r}_T}{\sqrt{\zeta_\kappa(\bar{r}_T)}} d_t \le \frac{6 d_0^2}{8^2 \sqrt{\zeta_\kappa(3 d_0)} \theta_{T, \delta}},
\end{align}
where we have used the fact that $d \mapsto \frac{d}{\sqrt{\zeta_\kappa(d)}}$ is an increasing function, and that $d_t \le d_0 + \bar{r}_t$ and $\bar{r}_t \le 3 d_0$ (or else $\tilde{\eta_t} = 0)$. Bounds for $\gamma \in \{ g_t, \operatorname{grad} f(x_t) \}$ hold in a similar fashion.

Now for the third line, we have
\begin{align}
    \sum_{k=0}^{t} \tilde{\eta}_k^2 \zeta_\kappa(d_k) \lVert g_k \rVert^2_{x_k} &\le \sum_{k=0}^{\mathcal{T}_{out} -1} {\eta}_k^2 \zeta_\kappa(d_k) \lVert g_k \rVert^2_{x_k}\\ & =  \sum_{k=0}^{\mathcal{T}_{out} -1} \frac{\bar{r}_k^2  \zeta_\kappa(d_k) (G_k -G_{k-1})}{\zeta_\kappa(\bar{r}_k) G_k^\prime}.
\end{align}
Now $d \mapsto \zeta_\kappa(d)$ is increasing, thus $\zeta_\kappa(d_k) \le \zeta_\kappa(d_0 +\bar{r}_{\mathcal{T}_{out} -1})$ as $d_k \le \bar{d}_k \le \bar{d}_{\mathcal{T}_{out} -1} \le d_0 + \bar{r}_{\mathcal{T}_{out} -1} $. Additionally, since  $d \mapsto  \frac{d}{\zeta_\kappa(d)} $  is increasing, we have $\frac{\bar{r}_k^2}{\zeta_\kappa(\bar{r}_k)} = \bar{r}_k \frac{\bar{r}_k}{\zeta_\kappa(\bar{r}_k)} \le  \bar{r}_{\mathcal{T}_{out} -1} \frac{\bar{r}_{\mathcal{T}_{out} -1}}{\zeta_\kappa(\bar{r}_{\mathcal{T}_{out} -1})} \le \bar{r}_{\mathcal{T}_{out} -1} \frac{(\bar{r}_{\mathcal{T}_{out} -1} + d_0)}{\zeta_\kappa(\bar{r}_{\mathcal{T}_{out} -1} + d_0)}$. Thus, we have
\begin{align}
   \sum_{k=0}^{\mathcal{T}_{out} -1} \frac{\bar{r}_k^2  \zeta_\kappa(d_k) (G_k -G_{k-1})}{\zeta_\kappa(\bar{r}_k) G_k^\prime}
    &\le \bar{r}_{\mathcal{T}_{out} -1}(\bar{r}_{\mathcal{T}_{out} -1} + d_0) \sum_{k=0}^{\mathcal{T}_{out} -1} \frac{G_k - G_{k-1}}{G_k^\prime}\\
    &\overset{(i)}{\le} \frac{\bar{r}_{\mathcal{T}_{out} -1}(\bar{r}_{\mathcal{T}_{out} -1} + d_0)}{8^4 \theta_{T, \delta}} \sum_{k=0}^{\mathcal{T}_{out} -1} \frac{G_k - G_{k-1}}{(G_k + \bar{\ell}_k^2) \log_{+}^2\left( \frac{G_k + \bar{\ell}^2_k}{\bar{\ell}^2_0}\right)}  \\ &\overset{(ii)}{\le} \frac{12 d_0^2}{8^4 \theta_{T, \delta}}.
\end{align}
Where we have used in $(i)$ that
\begin{align}
    G_k^\prime \ge 8^4 \theta_{T, \delta} (G_{k-1} + \lVert g_k \rVert_{x_k}^2 + \ell_k^2)\log_{+}^2 \left(\frac{\sum_{s=0}^{k} \bar{\ell}_s^2 + \bar{\ell}_k^2 }{\bar{\ell}_0^2} \right) \ge 8^4 \theta_{T, \delta} (G_k + \bar{\ell}_k^2) \log_{+}^2 \left( \frac{G_k + \bar{\ell}_k^2}{\bar{\ell}_0^2} \right),
\end{align}
holding since $\lVert g_k \rVert_{x_k} \le \ell_k$. Additionally, in $(ii)$  we have used \cref{lemma:sum_bounded_by_one} with $a_k = G_{k} + \bar{\ell}_k^2$ and $\bar{r}_{\mathcal{T}_{out} -1} \le 3d_0$. 

The final line holds from the previous noting that,
\begin{align}
     \sum_{k=0}^{t} (\tilde{\eta}_k \langle g_k, \exp_{x_k}^{-1}(x_\star) \rangle_{x_k})^2 \le \sum_{k=0}^{t} \tilde{\eta}_k^2 \zeta_\kappa(d_k) \lVert g_k \rVert^2_{x_k} d_k^2 \le (4d_0)^2 \sum_{k=0}^{t} \tilde{\eta}_k^2 \zeta_\kappa(d_k) \lVert g_k \rVert^2_{x_k},
\end{align}
where the first inequality follows from the Cauchy-Schwartz inequality and the second inequality holds from the fact that only terms with $k < \mathcal{T}_{out}$ contribute to the sum.
    \end{proof}
\end{lemma}

Using the above lemma, we can establish the following concentration bound.

\begin{lemma}
    If \cref{assumption:g_convex}, \ref{assumption:Hadamard}, and \ref{assumption:bounded_stochastic_grads} hold, under the truncated T-RDoG step size sequence $\{\tilde{\eta}_t\}$, the iterates satisfy,
    \begin{align}
        \mathbb{P} \left( \exists t \le T : \sum_{k=0}^{t-1} \tilde{\eta}_k \langle \Delta_k, \exp_{x_k}^{-1}(x_\star) \rangle_{x_k} > d_{0}^2 \right) \le \delta.
    \end{align}

    \begin{proof}
        Consider the filtration $\mathcal{F}_t = \sigma(\mathcal{G}(x_0) \dotsc, \mathcal{G}(x_t))$ and define $X_t = \tilde{\eta}_t \langle \Delta_t, \exp_{x_t}^{-1}(x_\star) \rangle_{x_t}$ and $\hat{X}_t = -\tilde{\eta}_t \langle \operatorname{grad} f(x_t), \exp_{x_t}^{-1}(x_\star) \rangle_{x_t}$. Then we have that $X_t$ is a martingale difference sequence adapted to $\mathcal{F}_t$ and $\hat{X}_t \in \mathcal{F}_{t-1}$. Moreover, we have $\max \{ |X_t|, |\hat{X}_t| \} \le c$ almost surely for $c =  \frac{24 d_0^2}{8^4 \theta_{T, \delta}}$. Substituting into \cref{lemma:martingale_bound}, we have
        \begin{align}
            \mathbb{P} \left( \exists t \le T : \left\lvert\sum_{k=0}^{t-1} X_k \right\rvert \ge 4 \sqrt{\theta_{t, \delta} \sum_{k=0}^{t-1}(X_k - \hat{X}_k)^2 + c^2 \theta^2_{t, \delta}}\right) \le \delta.
        \end{align}
    Noting that $X_t - \hat{X}_t = \tilde{\eta}_t \langle g_t, \exp_{x_t}^{-1}(x_\star)\rangle_{x_t}$ and substituting the definition of $c$ and the bound gives for every $t < T$,
    \begin{align}
        4 \sqrt{\theta_{t, \delta} \sum_{k=0}^{t-1}(X_k - \hat{X}_k)^2 + c^2 \theta^2_{t, \delta}} \le 4 \sqrt{\theta_{t, \delta} \frac{3 \cdot 4^3 d_0^4}{8^4 \theta_{T, \delta}} + \left(\frac{6 \theta_{t, \delta} d_0^2}{8^2 \sqrt{\zeta_\kappa(3 d_0)} \theta_{T, \delta}}\right)^2} \le d_0^2.
    \end{align}
    \end{proof}
\end{lemma}

Finally, we show that the event defined in the previous lemma implies the desired distance bound.

\begin{lemma}
    Suppose \cref{assumption:g_convex}, \ref{assumption:Hadamard}, and \ref{assumption:bounded_stochastic_grads} hold. If $\sum_{k=0}^{t-1} \tilde{\eta}_k \langle \Delta_k, \exp_{x_k}^{-1}(x_\star) \rangle_{x_k} \le d_0^2$ for all $t \le T$ then $\mathcal{T}_{out} > T$ i.e., $\bar{r}_t \le 3 d_0$.
    \begin{proof}
        To condense notation, let $B_t \coloneqq \max_{t' \le t} \sum_{k=0}^{t'-1} \tilde{\eta}_k \langle \Delta_k, \exp_{x_k}^{-1}(x_\star)\rangle_{x_k}$, so the claim becomes $B_t \le d_0^2$ implies $\mathcal{T}_{out} > t$ for all $t \le T$. We prove the claim by induction on $t$. The basis of the induction is that $\mathcal{T}_{out} > 0$ always hold since $\bar{r}_0 = \epsilon \le 3 d_0$ by assumption. For the induction step, we assume that $B_{t-1}$ implies $\mathcal{T}_{out} \ge t$ and show that $B_t \le d_0^2$ implies $\mathcal{T}_{out} > t$. To that end, we use $\langle \operatorname{grad} f(x_t), \exp_{x_t}^{-1}(x_\star) \rangle_{x_t} \ge f(x_t) - f(x_\star) \ge 0$ to rearrange \cref{cor:law_of_cosines} as
        \begin{align}
            d_{k+1}^2 - d_k^2 \le \eta_k^2 \zeta_\kappa(d_k) \lVert g_k \rVert_{x_k}^2 + 2 \eta_k \langle \Delta_k, \exp_{x_k}^{-1}(x_\star) \rangle_{x_k}
        \end{align}
        for all $k$. Summing from $0 \le k \le t-1$, we have
        \begin{align}
            d_{t}^2 - d_0^2 &\le \sum_{k=0}^{t-1} \eta_k^2 \zeta_\kappa(d_k) \lVert g_k \rVert_{x_k}^2 + 2 \sum_{k=0}^{t-1}  \eta_k \langle \Delta_k, \exp_{x_k}^{-1}(x_\star) \rangle_{x_k} \\  &= \sum_{k=0}^{t-1} \tilde{\eta}_k^2 \zeta_\kappa(d_k) \lVert g_k \rVert_{x_k}^2 + 2 \sum_{k=0}^{t-1}  \tilde{\eta}_k \langle \Delta_k, \exp_{x_k}^{-1}(x_\star) \rangle_{x_k}.
        \end{align}
        where the equality holds since $\mathcal{T}_{out} > t-1$ and therefore $\eta_k = \tilde{\eta}_k$ for all $0 \le k \le t-1$. Now, by previous lemma we have $\sum_{k=0}^{t-1} \tilde{\eta}_k^2 \zeta_\kappa(d_k) \lVert g_k \rVert_{x_k}^2 \le \frac{12 d_0^2}{8^4 \theta_{T, \delta}} \le d_0^2$. Moreover, by assumption we have $\sum_{k=0}^{t-1}  \tilde{\eta}_k \langle \Delta_k, \exp_{x_k}^{-1}(x_\star) \rangle_{x_k} \le B_t \le d_0^2$, from which we conclude, $d_{t}^2 \le 4 d_0^2$ and hence $r_{t} \le d_0 + d_{t} \le 3 d_0$. Finally, since $\bar{r}_{t} = \max\{\bar{r}_{t-1}, {r}_{t}\}$ and $\bar{r}_{t-1} \le 3 d_0$ by the induction assumption, we have that $\bar{r}_{t} \le 3 d_0$.
    \end{proof}
\end{lemma}

\begin{theorem}
Suppose $\epsilon \le 3 d_0$ and  \cref{assumption:g_convex}, \ref{assumption:Hadamard}, and \ref{assumption:bounded_stochastic_grads} hold. Then for any $\delta \in (0, 1)$ and $t \in \mathbb{N}$, under the T-RDoG step size sequence $\{{\eta}_t\}$, the iterates satisfy $\mathbb{P}(\bar{r}_t > 3 d_0) \le \delta$.
\end{theorem}
\begin{proof}
    A consequence of combining the previous two lemmas.
\end{proof}

\begin{corollary}
    Suppose that \cref{assumption:g_convex}, \ref{assumption:Hadamard}, and \ref{assumption:bounded_stochastic_grads} hold. For any $\delta \in (0, 1/2)$, $t \in \mathbb{N}$, consider T iterations of T-RDoG, with an initial step size of $\epsilon \le 3d_0$. Then for $\tau \in \arg\max_{t \le T} \sum_{s=0}^{t-1}\frac{\bar{r}_s/\sqrt{\zeta_\kappa(\bar{r}_s)}}{\bar{r}_t/\sqrt{\zeta_\kappa(\bar{r}_t)}}$ we have, with probability at least $1-2\delta$, that
    \begin{align}
        f(\tilde{x}_{\tau}) - f(x_\star) = O\left(c_{\delta, \epsilon, T} \frac{d_0 \sqrt{\zeta_\kappa(d_0) (G_{\tau-1}+L_{\star}^2)}}{T} \right) = O\left(c_{\delta, \epsilon, T} \frac{d_0 \sqrt{\zeta_\kappa(d_0)} L_\star}{\sqrt{T}} \right). 
    \end{align}
    where $L_\star \coloneqq \max_{x\in \mathcal{M}: d(x, x_0) \le 3 d(x_\star, x_0)} \ell (x)$ and $c_{\delta, \epsilon, T} = \log_{+}(T \frac{d_0 L_\star}{f(x_0) - f(x_\star)})\log_{+}(\frac{d_0}{\epsilon})\log(\frac{\log_{+}(T)}{\delta})$.
    \begin{proof}
        Here we adapt \cref{thm:opt_gap_dog_known}. Using that $\bar{r}_t \le 3 d_0$, we have $\zeta_\kappa(d_0 + \bar{r}_t) \le \zeta_\kappa(4 d_0) \le \frac{\sqrt{\kappa} 4 d_0}{\tanh( \sqrt{\kappa} 4 d_0)} \le \frac{\sqrt{\kappa} 4 d_0}{\tanh( \sqrt{\kappa} d_0)} = O(\zeta_\kappa(d_0)) $ for $\kappa>0$, otherwise $\zeta_\kappa(d_0 + \bar{r}_t) = 1 = \zeta_\kappa(d_0) = O(\zeta_\kappa(d_0))$ in the case $\kappa=0$. Now, by \cref{assumption:bounded_stochastic_grads} we have $\ell_0 \ge \lVert \operatorname{grad} f(x_0) \rVert_{x_0} \ge (f(x_0) - f(x_\star))/d_0$, while $\bar{r}_T \le 3 d_0$ gives $\bar{\ell}_T \le L_\star$. Therefore, $\log_{+}\left(1 + \frac{T \bar{\ell}^2_{T}}{\bar{\ell}^2_0}\right) = O\left(\log_{+} \left(T \frac{d_0 L_\star}{f(x_0) - f(x_\star)}\right)\right)$.
    \end{proof}
\end{corollary}

\subsection{Omitting Geometric Curvature Term Analysis}
\label{sec:theoretical-results-RDoG-no-curvature}

We analyze omitting the geometric curvature term from the denominator RDoG (\cref{alg:rdog}). Thus we consider step sizes of the form
\begin{align}
     \eta_t = \frac{\bar{r}_t}{\sqrt{\sum_{s=0}^{t} \lVert g_s \rVert_{x_s}^2}}.
\end{align}
We term this algorithm \emph{Curvature Omitted Riemannian Distance over Gradients} (CO-RDoG).

We consider bounding the error of the weighted average sequence,
\begin{align*}
        \tilde{x}_{t+1} = \operatorname{exp}_{\tilde{x}_{t}}\left(\frac{\bar{r}_t}{\sum_{s=0}^{t} \bar{r}_s} \operatorname{exp}_{\tilde{x}_{t}}^{-1}\left( x_{t} \right) \right), \quad \tilde{x}_{1} = x_{0}.
\end{align*}
For a geodesically convex function $f \colon \mathcal{M} \to \mathbb{R}$, we have by Jensens inequality (\cref{lemma:weighted_average}) that $\tilde{x}_t$ satisfies,
\begin{align}
    f(\tilde{x}_t) - f(x_\star) \le \frac{1}{\sum_{s=0}^{t-1} \bar{r}_s}\sum_{s=0}^{t-1} \bar{r}_s \langle  -\operatorname{grad} f(x_s), \exp_{x_s}^{-1}(x_\star)\rangle_{x_s}.
\end{align}
Recalling $g_s$ is the stochastic oracle evaluation, $\mathcal{G}(x_s)$, the numerator decomposes into two components:
\begin{align}
    \underset{\text{weighted regret}}{\underbrace{\sum_{s=0}^{t-1} \bar{r}_s \langle -g_s, \exp_{x_s}^{-1}(x_\star)\rangle_{x_s}}} +  \underset{\text{noise}}{\underbrace{\sum_{s=0}^{t-1} \bar{r}_s \langle \Delta_s, \exp_{x_s}^{-1}(x_\star)\rangle_{x_s}}},
\end{align}
with $\Delta_s \coloneqq g_s - \operatorname{grad} f(x_s)$.

We give deterministic bounds for the weighted regret (\cref{lemma:weighted_regret_bound_unknown_curvature}) and high probability bounds for the noise term (\cref{lemma:noise_bound_unknown_curvature}).

\begin{lemma} \label{lemma:weighted_regret_bound_unknown_curvature}
   Under \cref{assumption:g_convex} and \ref{assumption:Hadamard}, the iterates of CO-RDoG, satisfy
   \begin{align}
       \sum_{s=0}^{t-1} \bar{r}_s \langle  -g_s, \exp_{x_s}^{-1}(x_\star)\rangle_{x_s} \le \bar{r}_t \left( 2 \bar{d}_t + \bar{r}_t \zeta_{\kappa}(\bar{d}_t) \right) \sqrt{G_{t-1}}.
   \end{align}
    \begin{proof}
   Applying \cref{cor:law_of_cosines}, we can bound the weighted average as
   \begin{align}
      \sum_{s=0}^{t-1} \bar{r}_s \langle  -g_s, \exp_{x_s}^{-1}(x_\star)\rangle_{x_s} \le \frac{1}{2} \underset{(A)}{\underbrace{\sum_{s=0}^{t-1} \frac{\bar{r}_s}{\eta_s}\left(d_s^2 - d_{s+1}^2\right)}} + \frac{1}{2} \underset{(B)}{\underbrace{\sum_{s=0}^{t-1} \bar{r}_s \eta_s \zeta_\kappa(d_s) \lVert g_s \rVert_{x_s}^2}}.
   \end{align}
We bound the terms $(A)$ and $(B)$ in turn, beginning with the former:
    \begin{align}
       (A) &= \sum_{s=0}^{t-1} \sqrt{G_s} \left(d_s^2 - d_{s+1}^2 \right) = d_0^2 \sqrt{G_0} - d_t^2 \sqrt{G_{t-1}} + \sum_{s=1}^{t-1} d_s^2 \left( \sqrt{G_s} - \sqrt{G_{s-1}} \right) \\ 
       &\overset{(i)}{\le} {\bar{d}_t}^2 \sqrt{G_0} - d_t^2 \sqrt{G_{t-1}} +  \bar{d}_t^2 \sum_{s=1}^{t-1} \left( \sqrt{G_s} -  \sqrt{G_{s-1}} \right) = \sqrt{G}_{t-1} (\bar{d}_t^2 - d_t^2)
       \overset{(ii)}{\le} 4 \bar{r}_t \bar{d}_t \sqrt{G_{t-1}}.
    \end{align}
Inequality $(i)$ uses $d_s \le \bar{d}_t$ and that $G_t$ is nondecreasing. Inequality $(ii)$ use that for $k \in \arg \max_{s \le t} d_s$, we have $\bar{d}_t^2 - d_t^2 = d_k^2 - d_t^2 = (d_k - d_t) (d_k + d_t)  \le d(x_k, x_t) (d_k + d_t) \le (\bar{r}_k + \bar{r}_t)(d_k + d_t) \le 4 \bar{r}_t \bar{d}_t$.

Bounding the second term $(B)$, using $d \mapsto \zeta_\kappa(d)$ is an increasing function, we have:
    \begin{align}
       (B) = \sum_{s=0}^{t-1} \frac{\bar{r}_s^2 \zeta_\kappa(d_s) \lVert g_s \rVert_{x_s}^2}{\sqrt{G_s}} \le \sum_{s=0}^{t-1} \frac{\bar{r}_s^2 \zeta_\kappa(\bar{d}_s) \lVert g_s \rVert_{x_s}^2}{\sqrt{G_s}} \le  \bar{r}_t^2 \zeta_\kappa(\bar{d}_t) \sum_{s=0}^{t-1} \frac{ \lVert g_s \rVert_{x_s}^2}{\sqrt{G_s}} \le 2 \bar{r}_t^2 \zeta_\kappa(\bar{d}_t) \sqrt{G_{t-1}}.
\end{align}
Thus, combining $(A)$ and $(B)$ together, gives the result.
\end{proof}
\end{lemma}

\begin{lemma} \label{lemma:noise_bound_unknown_curvature}
   For all $\delta \in (0,1), T \in \mathbb{N}$ and $L > 0$, if \cref{assumption:g_convex}, \ref{assumption:Hadamard}, and \ref{assumption:bounded_stochastic_grads} hold, the iterates of CO-RDoG satisfy
    \begin{align}
       \mathbb{P} \left( \exists t \le T : \left\lvert  \sum_{s=0}^{t-1} \bar{r}_s \langle \Delta_s, \exp_{x_s}^{-1}(x_\star)\rangle_{x_s} \right\rvert \ge b_t \right) \le \delta + \mathbb{P}(\bar{\ell}_{T} > L),
   \end{align}
   where $b_t =  8 \bar{r}_{t-1} \bar{d}_{t-1} \sqrt{\theta_{t, \delta} G_{t-1} + \theta^{2}_{t, \delta} L^2}$ and $\bar{\ell}_{T} \coloneqq \max_{s \le T} \ell(x_s)$.
   \begin{proof}
       For $1\le s \le T$ define the random variables
\begin{align}
    Y_s \coloneqq \bar{r}_{s-1} \bar{d}_{s-1}, \quad X_s \coloneqq \left\langle \Delta_{s-1}, \frac{\exp_{x_{s-1}}^{-1}(x_\star)}{\bar{d}_{s-1}}\right\rangle_{x_{s-1}}, \quad \hat{X}_s \coloneqq \left\langle -\operatorname{grad}f(x_{s-1}), \frac{\exp_{x_{s-1}}^{-1}(x_\star)}{\bar{d}_{s-1}}\right\rangle_{x_{s-1}}.
\end{align}
By the Cauchy-Schwartz inequality and \cref{assumption:bounded_stochastic_grads} we have each $\lvert X_s \rvert \le \ell(x)$, and each $\lvert\hat{X}_s\rvert \le \ell(x)$ with probability 1. Moreover, for any $t\le T$, we have 
\begin{align}
    \sum_{s=1}^{t} Y_s X_s = \sum_{s=0}^{t-1} \bar{r}_s \langle \Delta_s, \exp_{x_s}^{-1}(x_\star) \rangle_{x_s}.
\end{align}
Therefore, applying \cref{lemma:martingale_bound} yields the result as required.
   \end{proof}
\end{lemma}

Combining the above results, we obtain the following.
\begin{theorem}
    \label{thm:opt_gap_dog_unknown}
    For all $\delta \in (0, 1)$ and $L>0$, if  \cref{assumption:g_convex}, \ref{assumption:Hadamard}, and \ref{assumption:bounded_stochastic_grads} hold, then with probability at least $1 - \delta - \mathbb{P}(\bar{\ell}_{T} > L)$, for all $t \le T$, the optimality gap on the weighted iterates $f(\tilde{x}_t) - f(x_\star)$ of CO-RDoG satisfy
    \begin{align} 
            O\left(\frac{(d_0 + \bar{r}_t)\zeta_{\kappa}(d_0 + \bar{r}_t) \sqrt{G_{t-1} + \theta_{t, \delta} G_{t-1} + \theta^{2}_{t, \delta} L^2}}{\sum_{s=0}^{t-1} \bar{r}_s/\bar{r}_t} \right)
    \end{align}
    with probability at least $1 - \delta - \mathbb{P}(\bar{\ell}_{T} > L)$.
    
    \begin{proof}
        Combining \cref{lemma:weighted_regret_bound_unknown_curvature} and  \cref{lemma:noise_bound_unknown_curvature} and utilizing $\bar{d}_t \le d_0 + \bar{r}_t$ and that $d \mapsto \zeta_\kappa(d)$ is an increasing function yields the result as required.
    \end{proof}
\end{theorem}
Thus in comparison to standard RDoG, we pay an additional cost of $O\left(\sqrt{\zeta_{\kappa}(d_0 + \bar{r}_t)}\right)$ for omitting the geometric curvature term with CO-RDoG.

We then have a useful result when the manifold is bounded but its exact diameter is unknown.
\begin{corollary}
    Under \cref{assumption:g_convex}, \ref{assumption:Hadamard}, and \ref{assumption:bounded_stochastic_grads}, for any $D \ge d_0$ let $L_{D} \coloneqq \max_{x\in \mathcal{M} : d(x, x_0) \le D} \ell(x)$. Then, for all $\delta \in (0, 1)$ and for $\tau \in \arg \max_{t \le T} \sum_{s=0}^{t-1} \frac{\bar{r}_s /\sqrt{\zeta_\kappa(\bar{r}_s)}}{\bar{r}_t /\sqrt{\zeta_\kappa(\bar{r}_t)}}$, with probability at least $1 - \delta - \mathbb{P}(\bar{\ell}_{T} > L)$, iterates of CO-RDoG satisfy the optimality gap bound
    \begin{align}
        f(\tilde{x}_{\tau}) - f(x_\star) = O\left( \frac{D\zeta_{\kappa}(D) \sqrt{G_{\tau-1} \theta_{\tau, \delta} + L^2_D \theta_{\tau, \delta}^2}}{T} \log_{+}(D/\epsilon) \right).
    \end{align}
    \begin{proof}
        Apply \cref{lemma:fraction_sum} to the denominator term of  \cref{thm:opt_gap_dog_unknown}.
    \end{proof}
 \end{corollary}

Thus in comparison to standard RDoG, we pay an additional cost of $O\left(\sqrt{\zeta_{\kappa}(D)}\right)$ for omitting the curvature term with CO-RDoG.

\section{RDoWG Theoretical Analysis}

\subsection{Overview}
In this section, we analyze RDoWG (\cref{alg:rdowg}). Thus we consider RSGD with step sizes given by,
\begin{align}
    \eta_t = \frac{\bar{r}_t^2}{\zeta_\kappa(\bar{r}_t)\sqrt{v_t}}, \quad v_t = v_{t-1} + \frac{\bar{r}_t^2}{\zeta_\kappa(\bar{r}_t)} \lVert g_t \rVert_{x_t}^2, \quad v_{-1} = 0.
\end{align}
We consider the bounding the error of the weighted average sequence,
\begin{align*}
        \tilde{x}_{t+1} = \operatorname{exp}_{\tilde{x}_{t}}\left(\frac{\bar{r}_t^2 / \zeta_\kappa(\bar{r}_t)}{\sum_{s=0}^{t} \bar{r}_s^2 / \zeta_\kappa(\bar{r}_s)} \operatorname{exp}_{\tilde{x}_{t}}^{-1}\left( x_{t} \right) \right), \quad \tilde{x}_{1} = x_{0}.
\end{align*}
For a geodesically convex function $f \colon \mathcal{M} \to \mathbb{R}$, we have by \cref{lemma:weighted_average} that $\tilde{x}_t$ satisfies,
\begin{align}
    f(\tilde{x}_t) - f(x_\star) \le \frac{1}{\sum_{s=0}^{t-1} ({\bar{r}_s^2 }/{\zeta_\kappa(\bar{r}_s)})}\sum_{s=0}^{t-1} ({\bar{r}_s^2 }/{\zeta_\kappa(\bar{r}_s)})\langle  -\operatorname{grad} f(x_s), \exp_{x_s}^{-1}(x_\star)\rangle_{x_s}.
\end{align}
Recalling that $g_s$ represents the stochastic oracle evaluation at $x_s$, denoted as $\mathcal{G}(x_s)$, we can decompose the numerator into two components:
\begin{align}
    \underset{\text{weighted regret}}{\underbrace{\sum_{s=0}^{t-1} (\bar{r}_s^2 /{\zeta_\kappa(\bar{r}_s)}) \langle -g_s, \exp_{x_s}^{-1}(x_\star)\rangle_{x_s}}} + \underset{\text{noise}}{\underbrace{\sum_{s=0}^{t-1} (\bar{r}_s^2 /{\zeta_\kappa(\bar{r}_s)}) \langle \Delta_s, \exp_{x_s}^{-1}(x_\star)\rangle_{x_s}}},
\end{align}
with $\Delta_s \coloneqq g_s - \operatorname{grad} f(x_s)$.

\subsection{Supporting Analysis}
Our first result gives deterministic bounds for the weighted regret (\cref{lemma:dowg_supporting_curvature}).
\begin{lemma}
    \label{lemma:dowg_supporting_curvature}
     Under \cref{assumption:g_convex} and \ref{assumption:Hadamard}, we have that the iterates of RDoWG (\cref{alg:rdowg}) satisfy
    \begin{align}
       \sum_{s=0}^{t-1} (\bar{r}_s^2 /{\zeta_\kappa(\bar{r}_s)}) \langle  -g_s, \exp_{x_s}^{-1}(x_\star)\rangle_{x_s} \le \bar{r}_t \left(2 \bar{d}_t + \frac{\bar{r}_t}{\zeta_{\kappa}(\bar{r}_t)} \zeta_{\kappa}(\bar{d}_t)\right)  \sqrt{v_{t-1}}.
   \end{align}
    \begin{proof}
           Follow same argument as \cref{lemma:weighted_regret_bound} but with weights $\frac{\bar{r}_s^2}{\zeta_\kappa(\bar{r}_s)}$ replacing $\frac{\bar{r}_s}{\sqrt{\zeta_\kappa(\bar{r}_s)}}$ and weighted gradient sum $v_s$ replacing the standard gradient sum $G_s$.
    \end{proof}
\end{lemma}

\subsection{Non-Smooth Analysis
\label{sec:rdowg-non-smooth-analysis}}
We give deterministic bounds for the weighted regret (\cref{lemma:dowg_nonsmooth_weighted_regret_curvature}) and high probability bounds for the noise term (\cref{lemma:dowg_nonsmooth_prob_curvature}) for the non-smooth setting.

\begin{lemma}
\label{lemma:dowg_nonsmooth_weighted_regret_curvature}
    \label{lemma:dowg_supporting_curvature}
     Under \cref{assumption:g_convex} and \ref{assumption:Hadamard}, we have that the iterates of RDoWG (\cref{alg:rdowg}) satisfy
    \begin{align}
       \sum_{s=0}^{t-1} (\bar{r}_s^2 /{\zeta_\kappa(\bar{r}_s)}) \langle  -g_s, \exp_{x_s}^{-1}(x_\star)\rangle_{x_s} \le \frac{\bar{r}_t^2}{\sqrt{\zeta_{\kappa}(\bar{r}_t)}} \left( 2 \bar{d}_t + \frac{\bar{r}_t}{\zeta_{\kappa}(\bar{r}_t)} \zeta_{\kappa}(\bar{d}_t) \right) \sqrt{G_{t-1}}.
   \end{align}
    \begin{proof}
        Using the bound of \ref{lemma:dowg_supporting_curvature} and that $\sqrt{v_{t-1}} \le \frac{\bar{r}_{t-1}}{\sqrt{\zeta_{\kappa}(\bar{r}_{t-1})}} \sqrt{G_{{t-1}}} \le \frac{\bar{r}_{t}}{\sqrt{\zeta_{\kappa}(\bar{r}_{t})}} \sqrt{G_{{t-1}}}$ gives the result.
    \end{proof}
\end{lemma}

\begin{lemma}
    \label{lemma:dowg_nonsmooth_prob_curvature}
  For all $\delta \in (0,1), T \in \mathbb{N}$ and $L > 0$, if \cref{assumption:g_convex}, \ref{assumption:Hadamard}, and \ref{assumption:bounded_stochastic_grads} hold, the iterates of RDoWG (\cref{alg:rdowg}) satisfy
    \begin{align}
       \mathbb{P} \left( \exists t \le T : \left\lvert  \sum_{s=0}^{t-1} \frac{\bar{r}_s^2}{\zeta_\kappa(\bar{r}_s)} \langle \Delta_s, \exp_{x_s}^{-1}(x_\star)\rangle_{x_s} \right\rvert \ge b_t \right) \le \delta + \mathbb{P}(\bar{\ell}_{T} > L),
   \end{align}
   where $b_t =  8 \frac{\bar{r}_{t-1}^2}{\zeta_\kappa(\bar{r}_{t-1})} \bar{d}_{t-1} \sqrt{\theta_{t, \delta} G_{t-1} + \theta^{2}_{t, \delta} L^2}$ and $\bar{\ell}_{T} \coloneqq \max_{s \le T} \ell(x_s)$.
   \begin{proof}
       Following the argument of \cref{lemma:noise_bound}.
   \end{proof}
\end{lemma}

Combining the above results, we obtain the following.

\begin{theorem}
  \label{thm:dowg_nonsmooth_curvature}
   For all $\delta \in (0, 1)$ and $L>0$, if \cref{assumption:g_convex}, \ref{assumption:Hadamard}, and \ref{assumption:bounded_stochastic_grads} hold, then with probability at least $1 - \delta - \mathbb{P}(\bar{\ell}_{T} > L)$, for all $t \le T$, the optimality gap on the weighted iterates $f(\tilde{x}_t) - f(x_\star)$ of RDoWG (\cref{alg:rdowg}) satisfy
    \begin{align} 
            O\left(\frac{(d_0 + \bar{r}_t)\sqrt{\zeta_{\kappa}(d_0 + \bar{r}_t)} \sqrt{G_{t-1} + \theta_{t, \delta} G_{t-1} + \theta^{2}_{t, \delta} L^2}}{\sum_{s=0}^{t-1} \frac{ \bar{r}_s^2 / \zeta_\kappa(\bar{r}_s) }{\bar{r}_t^2 / \zeta_\kappa(\bar{r}_t)}} \right).
    \end{align}
  \begin{proof}
  Using \cref{lemma:dowg_nonsmooth_weighted_regret_curvature} and \cref{lemma:dowg_nonsmooth_prob_curvature} we have
  \begin{align}
      f(\tilde{x}_t) - f(x_\star) &\le \frac{\left( 2 \bar{d}_t \sqrt{\zeta_\kappa(\bar{r}_t)} + \frac{\bar{r}_t}{\sqrt{\zeta_\kappa(\bar{r}_t)}} \zeta_{\kappa}(\bar{d}_t) \right) \sqrt{G_{t-1}} + 8 \bar{d}_{t} \sqrt{\theta_{t, \delta} G_{t-1} + \theta^{2}_{t, \delta} L^2}}{\sum_{s=0}^{t-1} \frac{ \bar{r}_s^2 / \zeta_\kappa(\bar{r}_s) }{\bar{r}_t^2 / \zeta_\kappa(\bar{r}_t)}}.
  \end{align}
  Now using the fact $\bar{d}_t \le d_0 + \bar{r}_t$ and that $d \mapsto \zeta_\kappa(d)$ and $d \mapsto \frac{d}{\sqrt{\zeta_\kappa(d)}}$ are increasing functions gives the result.
\end{proof}
\end{theorem}

\begin{corollary}
\label{cor:dowg_nonsmooth_curvature}
    Suppose \cref{assumption:g_convex}, \ref{assumption:Hadamard}, and \ref{assumption:bounded_stochastic_grads} hold, and for any $D \ge d_0$ let $L_{D} \coloneqq \max_{x\in \mathcal{M} : d(x, x_0) \le D} \ell(x)$. Then, for all $\delta \in (0, 1)$ and for $\tau \in \arg \max_{t \le T} \sum_{s=0}^{t-1} \frac{\bar{r}_s /\sqrt{\zeta_\kappa(\bar{r}_s)}}{\bar{r}_t /\sqrt{\zeta_\kappa(\bar{r}_t)}}$, with probability at least $1 - \delta - \mathbb{P}(\bar{\ell}_{T} > L)$, iterates of RDoWG (\cref{alg:rdowg}) satisfy the optimality gap bound
    \begin{align}
        f(\tilde{x}_{\tau}) - f(x_\star) = O\left( \frac{D\sqrt{\zeta_{\kappa}(D)} \sqrt{G_{\tau-1} \theta_{\tau, \delta} + L^2_D \theta_{\tau, \delta}^2}}{T} \log_{+}\left(\frac{D/\sqrt{\zeta_{\kappa}(D)}}{\epsilon/\sqrt{\zeta_{\kappa}(\epsilon)}}\right)\right).
    \end{align}
    \begin{proof}
        Apply \cref{lemma:fraction_sum} to the denominator term of \cref{thm:dowg_nonsmooth_curvature}.
    \end{proof}
 \end{corollary}

\subsection{Smooth Analysis}
\label{sec:rdowg-smooth-analysis}

\begin{lemma}
    \label{lemma:dowg_smooth_weighted_regret_curvature}
    Suppose $f$ is $S$-smooth and assume \cref{assumption:g_convex} and \ref{assumption:Hadamard} hold. Then we have that the iterates of RDoWG (\cref{alg:rdowg}) satisfy
    \begin{align}
       \sum_{s=0}^{t-1} (\bar{r}_s^2 /{\zeta_\kappa(\bar{r}_s)}) \langle  -g_s, \exp_{x_s}^{-1}(x_\star)\rangle_{x_s} \le  \bar{r}_t \left( 2 \bar{d}_t + \frac{\bar{r}_t}{\zeta_\kappa(\bar{r}_s)} \zeta_{\kappa}(\bar{d}_t) \right) \sqrt{2S \sum_{s=0}^{t-1} \frac{\bar{r}_s^2}{\zeta_\kappa(\bar{r}_s)} (f(x_s) - f(x_\star))}.
   \end{align}

    \begin{proof}
        By smoothness we can use \cref{lemma:lsmooth_bound}  to deduce $\lVert \operatorname{grad}f(x) \rVert_x^2 \le 2S(f(x) - f(x_\star))$ for all $x \in \mathcal{M}$. Therefore
        \begin{align}
            v_t = \sum_{s=0}^{t-1} \frac{\bar{r}_s^2}{\zeta_\kappa(\bar{r}_s)} \lVert \operatorname{grad}f(x_s) \rVert_{x_s}^2 \le 2S \sum_{s=0}^{t-1} \frac{\bar{r}_s^2}{\zeta_\kappa(\bar{r}_s)} (f(x_s) - f(x_\star)).
        \end{align}
        Taking square roots and substituting this into \cref{lemma:dowg_supporting_curvature} gives the result.
    \end{proof}
\end{lemma}

\begin{lemma}
\label{lemma:dowg_smooth_prob_bound_curvature}
    Suppose  \cref{assumption:g_convex}, \ref{assumption:Hadamard}, and  \ref{assumption:bounded_stochastic_lsmooth} hold. Then for all $\delta \in (0,1), T \in \mathbb{N}$ and $S > 0$,  Then  we have that the iterates of RDoWG (\cref{alg:rdowg}) satisfy
    \begin{align}
       \mathbb{P} \left( \exists t \le T : \left\lvert  \sum_{s=0}^{t-1} \frac{\bar{r}_s^2}{\zeta_\kappa(\bar{r}_s)} \langle \Delta_s, \exp_{x_s}^{-1}(x_\star)\rangle_{x_s} \right\rvert \ge b_t \right) \le \delta + \mathbb{P}(\bar{s}_{T} > S),
   \end{align}
   where $b_t =  8\frac{\bar{r}_{t-1}}{\sqrt{\zeta_\kappa(\bar{r}_{t-1})}} \bar{d}_{t-1} \sqrt{2S \theta_{t, \delta} + 8S \theta_{t, \delta}^2} \sqrt{\sum_{s=0}^{t-1} \frac{\bar{r}_s^2}{\zeta_\kappa(\bar{r}_s)} \left[f(x_s) - f(x_\star)\right]}$ and $\bar{s}_{T} \coloneqq \max_{t \le T} s(x_t)$.
   
   \begin{proof}
        Define for $1 \le s \le T$ the following random variables as
    \begin{align}
        Y_s \coloneqq \frac{\bar{r}_{s-1}}{\sqrt{\zeta_\kappa(\bar{r}_{s-1})}} \bar{d}_{s-1}, \quad X_{s} \coloneqq \frac{\bar{r}_{s-1}}{\sqrt{\zeta_\kappa(\bar{r}_{s-1})}} \left\langle \Delta_{s-1}, \frac{\exp_{x_{s-1}}^{-1}(x_\star)}{\bar{d}_{s-1}}\right\rangle_{x_{s-1}}, \quad \\ \hat{X}_s \coloneqq \frac{\bar{r}_{s-1}}{\sqrt{\zeta_\kappa(\bar{r}_{s-1})}} \left\langle -\operatorname{grad} f(x_{s-1}), \frac{\exp_{x_{s-1}}^{-1}(x_\star)}{\bar{d}_{s-1}}\right\rangle_{x_{s-1}},
    \end{align}
    and follow similar argument to \cref{lemma:uniform_dog_smooth_prob_bound_unknown_curvature}.
   \end{proof}
\end{lemma}

Combining the above results, we obtain the following.
\begin{theorem}
    \label{thm:dowg_smooth_known_curvature}
    Suppose
 \cref{assumption:g_convex}, \ref{assumption:Hadamard}, and \ref{assumption:bounded_stochastic_lsmooth} hold. Then for all $\delta \in (0, 1)$ and $S>0$, with probability at least $1 - \delta - \mathbb{P}(\bar{s}_{T} > S)$, for all $t \le T$, the optimality gap on the weighted iterates $f(\tilde{x}_t) - f(x_\star)$ of RDoWG (\cref{alg:rdowg}) satisfy
    \begin{align}
        O \left(\frac{ (d_0 + \bar{r}_t)^2 \zeta_{\kappa}(d_0 + \bar{r}_t) (S \theta_{t, \delta}^2) }{\sum_{s=0}^{t-1} \frac{ \bar{r}_s^2 / \zeta_\kappa(\bar{r}_s)}{\bar{r}_t^2 / \zeta_\kappa(\bar{r}_t)} } \right).
    \end{align}
    
\begin{proof}
Using \cref{lemma:dowg_smooth_weighted_regret_unknown_curvature} and  \cref{lemma:dowg_smooth_prob_bound_unknown_curvature} above, we have with the relevant probabilistic conditions,  
\begin{align*}
    \sum_{s=0}^{t-1} \frac{\bar{r}_s^2}{\zeta_\kappa(\bar{r}_s)} [f(x_s) - f(x_\star)] &\le \left (\sqrt{2S} \bar{r}_t \left( 2 \bar{d}_t + \frac{\bar{r}_t}{\zeta_\kappa(\bar{r}_t)} \zeta_{\kappa}(\bar{d}_t) \right) +  8\frac{\bar{r}_t}{\sqrt{\zeta_\kappa(\bar{r}_t)}} \bar{d}_t \sqrt{2S \theta_{t, \delta} + 8S \theta_{t, \delta}^2} \right) \\ &\times \sqrt{\sum_{s=0}^{t-1} \frac{\bar{r}_s^2}{\zeta_\kappa(\bar{r}_s)} \left[f(x_s) - f(x_\star)\right]}.
\end{align*}
Now if $f(x_s) - f(x_\star) = 0$ for some iterate, then the statement is trivial. Otherwise diving by sides by the square root term, we have
\begin{align}
\sqrt{\sum_{s=0}^{t-1} \frac{\bar{r}_s^2}{\zeta_\kappa(\bar{r}_s)} \left[f(x_s) - f(x_\star)\right]} \le  \left (\sqrt{2S} \bar{r}_t \left( 2 \bar{d}_t + \frac{\bar{r}_t}{\zeta_\kappa(\bar{r}_t)} \zeta_{\kappa}(\bar{d}_t) \right) + 8\frac{\bar{r}_t}{\sqrt{\zeta_\kappa(\bar{r}_t)}} \bar{d}_t \sqrt{2S \theta_{t, \delta} + 8S \theta_{t, \delta}^2} \right).
\end{align}
We square both sides and divide through by $\frac{\bar{r}_s^2}{\zeta_\kappa(\bar{r}_s)}$. Finally  using the fact, $\bar{d}_t \le d_0 + \bar{r}_t$, in the above bound gives the result since $d \mapsto \zeta_\kappa(d)$ and $d \mapsto \frac{d}{\sqrt{\zeta_\kappa(d)}}$ are increasing functions.
\end{proof}
\end{theorem}

We then have a useful result when the manifold is bounded but its exact diameter is unknown.

\begin{corollary}
\label{cor:dowg_smooth_known_curvature}
    Under \cref{assumption:g_convex}, \ref{assumption:Hadamard},  and \ref{assumption:bounded_stochastic_lsmooth} hold, for any $D \ge d_0$ let $S_{D} \coloneqq \max_{x\in \mathcal{M} : d(x, x_0) \le D} s(x)$. Then, for all $\delta \in (0, 1)$ and for $\tau \in \arg \max_{t \le T} \sum_{s=0}^{t-1} \frac{\bar{r}_s^2 / \zeta_\kappa(\bar{r}_s)}{\bar{r}_t^2 / \zeta_\kappa(\bar{r}_t)}$, with probability at least $1 - \delta - \mathbb{P}(\bar{s}_{T} > S)$, iterates of \cref{alg:rdog} satisfy the optimally gap on the weighted iterates $f(\tilde{x}_{\tau}) - f(x_\star)$ of RDoWG (\cref{alg:rdowg}) satisfy
    \begin{align}
         O\left(\frac{D^2 \zeta_{\kappa}(D) S_D \theta_{\tau, \delta}^2}{T} \log_{+}\left(\frac{D/\sqrt{\zeta_{\kappa}(D)}}{\epsilon/\sqrt{\zeta_{\kappa}(\epsilon)}}\right)\right).
    \end{align}
    \begin{proof}
        Apply \cref{lemma:fraction_sum} to the denominator term of  \cref{thm:dowg_smooth_known_curvature}.
    \end{proof}
 \end{corollary}

\subsection{Iterate Stability Bound}
\label{sec:rdowg-stability-analysis}

We introduce \emph{Tamed Riemannian Distance over Weighted Gradients} (T-RDoWG), a dampened version of RDoWG (\cref{alg:rdowg}) whose iterates are guaranteed to remain bounded with high probability. T-RDoWG has the following step size scheme
\begin{align}
     v_t &= v_{t-1} + \frac{\bar{r}_t^2}{\zeta_\kappa(\bar{r}_t)} \lVert g_t \rVert_{x_t}^2, \quad v_{-1} = 0,\\
    \eta_t &= \frac{\bar{r}_t^2}{\zeta_{\kappa}(\bar{r}_t)\sqrt{v_t^\prime}}, \quad v_t^\prime = 8^4 \theta_{T, \delta}^2 \log_{+}^2\left(\frac{(1+t)\bar{r}_t^2\bar{\ell}_t^2/\zeta_\kappa(\bar{r}_t)}{\bar{r}_0^2\bar{\ell}_0^2/\zeta_\kappa(\bar{r}_0)}\right)(v_{t-1} + 16 \frac{\bar{r}_t^2}{\zeta_\kappa(\bar{r}_t)}\bar{\ell}^2_t).
\end{align}

To show iterate boundedness in the stochastic setting, we consider the stopping time
\begin{align}
    \mathcal{T}_{out} = \min \{t \ge 0 : \bar{r}_t > 3d_0\},
\end{align}
so that the event $\{\bar{r}_T \le 3 d_0\}$ is the same as $\{\mathcal{T}_{out}  > T\}$.  We also define the following truncated step size sequence,
\begin{align}
    \tilde{\eta}_k \coloneqq \eta_k \mathbb{I}_{\{k < \mathcal{T}_{out}\}}.
\end{align}
Truncating as such allows us to handle the possibility that $\bar{r}_T$ exceeds $3d_0$. In particular, the following holds for $\{\tilde{\eta}_k\}$ but not for $\{{\eta}_k\}$.

\begin{lemma}
    For all $t \le T$, if \cref{assumption:g_convex}, \ref{assumption:Hadamard}, and \ref{assumption:bounded_stochastic_grads} hold, under the truncated T-RDoWG step size sequence $\{\tilde{\eta}_t\}$, the iterates satisfy
        \begin{align}
            \tilde{\eta}_t &\in \sigma(\mathcal{G}(x_0), \dotsc, \mathcal{G}(x_{t-1})), \\
            \lvert \tilde{\eta}_t \langle \gamma, \exp_{x_t}^{-1} (x_\star)\rangle_{x_t} \rvert &\le \frac{6 d_0^2}{8^2 \sqrt{\zeta_\kappa(3 d_0)} \theta_{T, \delta}}  \ \text{for} \ \gamma \in \{g_t, \operatorname{grad} f(x_t), \Delta_t \}, \\
            \sum_{k=0}^{t} \tilde{\eta}_k^2 \zeta_\kappa(d_k) \lVert g_k \rVert^2_{x_k} &\le \frac{12 d_0^2}{8^4 \theta_{T, \delta}}
            , \ \text{and} \\
            \sum_{k=0}^{t} (\tilde{\eta}_k \langle g_k, \exp_{x_k}^{-1}(x_\star) \rangle_{x_k})^2 &\le \frac{3 \cdot 4^3 d_0^4}{8^4 \theta_{T, \delta}}.
        \end{align}
    \begin{proof}
The first line holds directly by definition of the truncated T-RDoWG iterates. For bound in the second line, note (recalling $\Delta_t = g_t - \operatorname{grad} f(x_t)$ we have $\lVert \Delta_t \rVert_{x_t} \le \lVert g_t \rVert_{x_t} + \lVert\operatorname{grad} f(x_t) \rVert_{x_t}  \le 2 \ell(x_t)$. Since $v^\prime_t \ge 4^28^4\frac{\bar{r}_t^2}{\zeta_\kappa(\bar{r}_t)}\ell^2(x_t) \theta_{T, \delta}^2$ for all $t$, the Cauchy-Schwartz inequality gives,
\begin{align}
    \lvert \tilde{\eta}_t \langle \Delta_t, \exp_{x_t}^{-1} (x_\star)\rangle_{x_t} \rvert \le \frac{\bar{r}_t^2}{\zeta_\kappa(\bar{r}_t) \sqrt{v_t^\prime}} \lVert \Delta_t \lVert_{x_t} d_t \le \frac{1}{2 \cdot 8^2 \theta_{T, \delta}} \frac{\bar{r}_T}{\sqrt{\zeta_\kappa(\bar{r}_T)}} d_t \le \frac{6 d_0^2}{8^2 \sqrt{\zeta_\kappa(3 d_0)} \theta_{T, \delta}},
\end{align}
where we have used the fact that $d \mapsto \frac{d}{\sqrt{\zeta_\kappa(d)}}$ is an increasing function, and that $d_t \le d_0 + \bar{r}_t$ and $\bar{r}_t \le 3 d_0$ (or else $\tilde{\eta_t} = 0)$. Bounds for $\gamma \in \{ g_t, \operatorname{grad} f(x_t) \}$ hold in a similar fashion.

Now for the third line, we have
\begin{align}
    \sum_{k=0}^{t} \tilde{\eta}_k^2 \zeta_\kappa(d_k) \lVert g_k \rVert^2_{x_k} &\le \sum_{k=0}^{ \mathcal{T}_{out} -1} {\eta}_k^2 \zeta_\kappa(d_k) \lVert g_k \rVert^2_{x_k}\\ & =  \sum_{k=0}^{ \mathcal{T}_{out} -1} \frac{\bar{r}_k^4  \zeta_\kappa(d_k) \lVert g_k \rVert^2_{x_k}}{\zeta_\kappa(\bar{r}_k)^2 v_k^\prime} \\
    &= \sum_{k=0}^{ \mathcal{T}_{out} -1} \frac{\bar{r}_k^2  \zeta_\kappa(d_k) (v_{k} - v_{k-1})}{\zeta_\kappa(\bar{r}_k) v_k^\prime}.
\end{align}
Now $d \mapsto \zeta_\kappa(d)$ is increasing, thus $\zeta_\kappa(d_k) \le \zeta_\kappa(d_0 +\bar{r}_{\mathcal{T}_{out} -1})$ as $d_k \le \bar{d}_k \le \bar{d}_{\mathcal{T}_{out} -1} \le d_0 + \bar{r}_{\mathcal{T}_{out} -1} $. Additionally, since  $d \mapsto  \frac{d}{\zeta_\kappa(d)} $  is increasing, we have $\frac{\bar{r}_k^2}{\zeta_\kappa(\bar{r}_k)} = \bar{r}_k \frac{\bar{r}_k}{\zeta_\kappa(\bar{r}_k)} \le  \bar{r}_{\mathcal{T}_{out} -1} \frac{\bar{r}_{\mathcal{T}_{out} -1}}{\zeta_\kappa(\bar{r}_{\mathcal{T}_{out} -1})} \le \bar{r}_{\mathcal{T}_{out} -1} \frac{(\bar{r}_{\mathcal{T}_{out} -1} + d_0)}{\zeta_\kappa(\bar{r}_{\mathcal{T}_{out} -1} + d_0)}$. Thus, we have
\begin{align}
   \sum_{k=0}^{ \mathcal{T}_{out} -1} \frac{\bar{r}_k^2  \zeta_\kappa(d_k) (G_k -G_{k-1})}{\zeta_\kappa(\bar{r}_k) G_k^\prime}
    &\le \bar{r}_{\mathcal{T}_{out} -1}(\bar{r}_{\mathcal{T}_{out} -1} + d_0) \sum_{k=0}^{ \mathcal{T}_{out} -1} \frac{v_k - v_{k-1}}{v_k^\prime}\\
    &\overset{(i)}{\le} \frac{\bar{r}_{\mathcal{T}_{out} -1}(\bar{r}_{\mathcal{T}_{out} -1} + d_0)}{8^4 \theta_{T, \delta}} \sum_{k=0}^{ \mathcal{T}_{out} -1} \frac{v_k - v_{k-1}}{(v_k + \frac{\bar{r}_k^2}{\zeta_\kappa(\bar{r}_k)}\bar{\ell}_k^2) \log_{+}^2\left( \frac{v_k + \frac{\bar{r}_k^2}{\zeta_\kappa(\bar{r}_k)}\bar{\ell}^2_k}{\frac{\bar{r}_0^2}{\zeta_\kappa(\bar{r}_0)} \bar{\ell}^2_0}\right)}  \\ &\overset{(ii)}{\le} \frac{12 d_0^2}{8^4 \theta_{T, \delta}}.
\end{align}
Where we have used in $(i)$ that
\begin{align}
    v_k^\prime &\ge 8^4 \theta_{T, \delta} \left(v_{k-1} + \frac{\bar{r}_k^2}{\zeta_\kappa(\bar{r}_k)}\lVert g_k \rVert_{x_k}^2 + \frac{\bar{r}_k^2}{\zeta_\kappa(\bar{r}_k)}\ell_k^2\right)\log_{+}^2 \left(\frac{\sum_{s=0}^{k} \frac{\bar{r}_s^2}{\zeta_\kappa(\bar{r}_s)} \bar{\ell}_s^2 + \frac{\bar{r}_k^2}{\zeta_\kappa(\bar{r}_k)} \bar{\ell}_k^2}{\frac{\bar{r}_0^2}{\zeta_\kappa(\bar{r}_0)}\bar{\ell}_0^2} \right) \\ &\ge 8^4 \theta_{T, \delta} \left(v_k + \frac{\bar{r}_k^2}{\zeta_\kappa(\bar{r}_k)}\bar{\ell}_k^2\right) \log_{+}^2 \left(\frac{v_k + \frac{\bar{r}_k^2}{\zeta_\kappa(\bar{r}_k)}\bar{\ell}_k^2}{\frac{\bar{r}_0^2}{\zeta_\kappa(\bar{r}_0)}\bar{\ell}_0^2}\right),
\end{align}
holding since $\lVert g_k \rVert_{x_k} \le \ell_k$. Additionally, in $(ii)$  we have used \cref{lemma:sum_bounded_by_one} with $a_k = v_{k} + \frac{\bar{r}_k^2}{\zeta_\kappa(\bar{r}_k)} \bar{\ell}_k^2$ and $\bar{r}_{\mathcal{T}_{out} -1} \le 3d_0$. 

The final line holds from the previous noting that,
\begin{align}
     \sum_{k=0}^{t} (\tilde{\eta}_k \langle g_k, \exp_{x_k}^{-1}(x_\star) \rangle_{x_k})^2 \le \sum_{k=0}^{t} \tilde{\eta}_k^2 \zeta_\kappa(d_k) \lVert g_k \rVert^2_{x_k} d_k^2 \le (4d_0)^2 \sum_{k=0}^{t} \tilde{\eta}_k^2 \zeta_\kappa(d_k) \lVert g_k \rVert^2_{x_k},
\end{align}
where the first inequality follows from the Cauchy-Schwartz inequality and the second inequality holds from the fact that only terms with $k < \mathcal{T}_{out}$ contribute to the sum.
    \end{proof}
\end{lemma}

Using the above lemma, we can establish the following concentration bound.

\begin{lemma}
     If \cref{assumption:g_convex}, \ref{assumption:Hadamard}, and \ref{assumption:bounded_stochastic_grads} hold, under the truncated T-RDoWG step size sequence $\{\tilde{\eta}_t\}$, the iterates satisfy,
    \begin{align}
        \mathbb{P} \left( \exists t \le T : \sum_{k=0}^{t-1} \tilde{\eta}_k \langle \Delta_k, \exp_{x_k}^{-1}(x_\star) \rangle_{x_k} > d_{0}^2 \right) \le \delta.
    \end{align}

    \begin{proof}
        Consider the filtration $\mathcal{F}_t = \sigma(\mathcal{G}(x_0) \dotsc, \mathcal{G}(x_t))$ and define $X_t = \tilde{\eta}_t \langle \Delta_t, \exp_{x_t}^{-1}(x_\star) \rangle_{x_t}$ and $\hat{X}_t = -\tilde{\eta}_t \langle \operatorname{grad} f(x_t), \exp_{x_t}^{-1}(x_\star) \rangle_{x_t}$. Then we have that $X_t$ is a martingale difference sequence adapted to $\mathcal{F}_t$ and $\hat{X}_t \in \mathcal{F}_{t-1}$. Moreover, we have $\max \{ |X_t|, |\hat{X}_t| \} \le c$ almost surely for $c =  \frac{24 d_0^2}{8^4 \theta_{T, \delta}}$. Substituting into \cref{lemma:martingale_bound}, we have
        \begin{align}
            \mathbb{P} \left( \exists t \le T : \left\lvert\sum_{k=0}^{t-1} X_k \right\rvert \ge 4 \sqrt{\theta_{t, \delta} \sum_{k=0}^{t-1}(X_k - \hat{X}_k)^2 + c^2 \theta^2_{t, \delta}}\right) \le \delta.
        \end{align}
    Noting that $X_t - \hat{X}_t = \tilde{\eta}_t \langle g_t, \exp_{x_t}^{-1}(x_\star)\rangle_{x_t}$ and substituting the definition of $c$ and the bound gives for every $t < T$,
    \begin{align}
        4 \sqrt{\theta_{t, \delta} \sum_{k=0}^{t-1}(X_k - \hat{X}_k)^2 + c^2 \theta^2_{t, \delta}} \le 4 \sqrt{\theta_{t, \delta} \frac{3 \cdot 4^3 d_0^4}{8^4 \theta_{T, \delta}} + \left(\frac{6 \theta_{t, \delta} d_0^2}{8^2 \sqrt{\zeta_\kappa(3 d_0)} \theta_{T, \delta}}\right)^2} \le d_0^2.
    \end{align}
    \end{proof}
\end{lemma}

\begin{lemma}
     Suppose \cref{assumption:g_convex}, \ref{assumption:Hadamard}, and \ref{assumption:bounded_stochastic_grads} hold. If $\sum_{k=0}^{t-1} \tilde{\eta}_k \langle \Delta_k, \exp_{x_k}^{-1}(x_\star) \rangle_{x_k} \le d_0^2$ for all $t \le T$ then $\mathcal{T}_{out} > T$ i.e., $\bar{r}_t \le 3 d_0$.
    \begin{proof}
        To condense notation, let $B_t \coloneqq \max_{t' \le t} \sum_{k=0}^{t' -1} \tilde{\eta}_k \langle \Delta_k, \exp_{x_k}^{-1}(x_\star)\rangle_{x_k}$, so the claim becomes $B_t \le d_0^2$ implies $\mathcal{T}_{out} > t$ for all $t \le T$. We prove the claim by induction on $t$. The basis of the induction is that $\mathcal{T}_{out} > 0$ always hold since $\bar{r}_0 = \epsilon \le 3 d_0$ by assumption. For the induction step, we assume that $B_{t}$ implies $\mathcal{T}_{out} \ge t$ and show that $B_t \le d_0^2$ implies $\mathcal{T}_{out} > t$. To that end, we use $\langle \operatorname{grad} f(x_t), \exp_{x_t}^{-1}(x_\star) \rangle_{x_t} \ge f(x_t) - f(x_\star) \ge 0$ to rearrange \cref{cor:law_of_cosines} as
        \begin{align}
            d_{k+1}^2 - d_k^2 \le \eta_k^2 \zeta_\kappa(d_k) \lVert g_k \rVert_{x_k}^2 + 2 \eta_k \langle \Delta_k, \exp^{-1}_{x_k}(x_\star) \rangle_{x_k}
        \end{align}
        for all $k$. Summing from $0 \le k \le t-1$, we have
        \begin{align}
            d_{t}^2 - d_0^2 &\le \sum_{k=0}^{t-1} \eta_k^2 \zeta_\kappa(d_k) \lVert g_k \rVert_{x_k}^2 + 2 \sum_{k=0}^{t-1}  \eta_k \langle \Delta_k, \exp_{x_k}^{-1}(x_\star) \rangle_{x_k} \\  &= \sum_{k=0}^{t-1} \tilde{\eta}_k^2 \zeta_\kappa(d_k) \lVert g_k \rVert_{x_k}^2 + 2 \sum_{k=0}^{t-1}  \tilde{\eta}_k \langle \Delta_k, \exp_{x_k}^{-1}(x_\star) \rangle_{x_k}.
        \end{align}
        where the equality holds since $\mathcal{T}_{out} \ge t$ and therefore $\eta_k = \tilde{\eta}_k$ for all $0 \le k \le t-1$. Now, by previous lemma we have $\sum_{k=0}^{t-1} \tilde{\eta}_k^2 \zeta_\kappa(d_k) \lVert g_k \rVert_{x_k}^2 \le \frac{12 d_0^2}{8^4 \theta_{T, \delta}} \le d_0^2$. Moreover, by assumption we have $\sum_{k=0}^{t-1}  \tilde{\eta}_k \langle \Delta_k, \exp_{x_k}^{-1}(x_\star) \rangle_{x_k} \le B_t \le d_0^2$, from which we conclude, $d_{t}^2 \le 4 d_0^2$ and hence $r_{t} \le d_0 + d_{t} \le 3 d_0$. Finally, since $\bar{r}_{t} = \max\{\bar{r}_{t-1}, {r}_{t}\}$ and $\bar{r}_{t-1} \le 3 d_0$ by the induction assumption, we have that $\bar{r}_{t} \le 3 d_0$.
    \end{proof}
\end{lemma}

\begin{theorem}
\label{prop:stability-analysis}
Suppose \cref{assumption:g_convex}, \ref{assumption:Hadamard}, and \ref{assumption:bounded_stochastic_grads} hold, and $\epsilon \le 3 d_0$. Then for any $\delta \in (0, 1)$ and $t \in \mathbb{N}$, under the T-RDoWG step size sequence $\{\eta_k\}$, the iterates satisfy $\mathbb{P}(\bar{r}_t > 3 d_0) \le \delta$.
\end{theorem}
\begin{proof}
    A consequence of combining the previous two lemmas.
\end{proof}

\begin{corollary}
\label{cor:stability-analysis}
    Suppose that \cref{assumption:g_convex}, \ref{assumption:Hadamard}, and \ref{assumption:bounded_stochastic_grads} hold. For any $\delta \in (0, 1/2)$, $t \in \mathbb{N}$, consider T iterations of $\{\eta_k\}$, with initial step size of $\epsilon \le 3d_0$. Then for $\tau \in \arg\max_{t \le T} \sum_{s=0}^{t-1} \frac{\bar{r}_s^2/\zeta_\kappa(\bar{r}_s)}{\bar{r}_t^2/\zeta_\kappa(\bar{r}_t)}$ we have, with probability at least $1-2\delta$, that
    \begin{align}
        f(\tilde{x}_{\tau}) - f(x_\star) = O\left(c_{\delta, \epsilon, T} \frac{d_0 \sqrt{\zeta_\kappa(d_0) (G_{\tau}+L_{\star}^2)}}{T} \right) = O\left(c_{\delta, \epsilon, T} \frac{d_0 \sqrt{\zeta_\kappa(d_0)} L_\star}{\sqrt{T}} \right). 
    \end{align}
    where $L_\star \coloneqq \max_{x\in \mathcal{M}: d(x, x_0) \le 3 d(x_\star, x_0)} \ell (x)$ and $c_{\delta, \epsilon, T} = \log_{+}(T \frac{d_0 L_\star}{f(x_0) - f(x_\star)})\log_{+}(\frac{d_0}{\epsilon})\log(\frac{\log_{+}(T)}{\delta})$.
    \begin{proof}
        Here we adapt theorem \cref{thm:dowg_nonsmooth_curvature}. Using that $\bar{r}_t \le 3 d_0$, we have $\zeta_\kappa(d_0 + \bar{r}_t) \le \zeta_\kappa(4 d_0) \le \frac{\sqrt{\kappa} 4 d_0}{\tanh( \sqrt{\kappa} 4 d_0)} \le \frac{\sqrt{\kappa} 4 d_0}{\tanh( \sqrt{\kappa} d_0)} = O(\zeta_\kappa(d_0)) $ for $\kappa>0$, otherwise $\zeta_\kappa(d_0 + \bar{r}_t) = 1 = \zeta_\kappa(d_0) = O(\zeta_\kappa(d_0))$ in the case $\kappa=0$. Now, by \cref{assumption:bounded_stochastic_grads} we have $\ell_0 \ge \lVert \operatorname{grad} f(x_0) \rVert_{x_0} \ge (f(x_0) - f(x_\star))/d_0$, while $\bar{r}_T \le 3 d_0$ gives $\bar{\ell}_T \le L_\star$. Therefore, $\log_{+}\left(1 + \frac{T \bar{\ell}^2_{T}}{\bar{\ell}^2_0}\right) = O\left(\log_{+} \left(T \frac{d_0 L_\star}{f(x_0) - f(x_\star)}\right)\right)$.
    \end{proof}
\end{corollary}

\subsection{Omitting Geometric Curvature Term Analysis}
\label{sec:theoretical-results-RDoWG-no-curvature}
We analyze omitting the geometric curvature term from the denominator RDoWG (\cref{alg:rdowg}). Thus we consider step sizes of the form
\begin{align}
    \eta_t = \frac{\bar{r}_t^2}{\sqrt{v_t}}, \quad v_t = v_{t-1} + \bar{r}_t^2 \lVert g_s \rVert_{x_s}^2, \quad v_{-1} = 0.
\end{align}
We term this algorithm \emph{Curvature Omitted Riemannian Distance over Weighted Gradients} (CO-RDoWG).

We consider the bound the error of the weighted average sequence,
\begin{align*}
        \tilde{x}_{t+1} = \operatorname{exp}_{\tilde{x}_{t}}\left(\frac{\bar{r}_t^2}{\sum_{s=0}^{t} \bar{r}_s^2} \operatorname{exp}_{\tilde{x}_{t}}^{-1}\left( x_{t} \right) \right), \quad \tilde{x}_{1} = x_{0}.
\end{align*}
For a geodesically convex function $f \colon \mathcal{M} \to \mathbb{R}$, we have by \cref{lemma:weighted_average} that $\tilde{x}_t$ satisfies,
\begin{align}
    f(\tilde{x}_t) - f(x_\star) \le \frac{1}{\sum_{s=0}^{t-1} \bar{r}_s^2}\sum_{s=0}^{t-1}\bar{r}_s^2 \langle  -\operatorname{grad} f(x_s), \exp_{x_s}^{-1}(x_\star)\rangle_{x_s}.
\end{align}
Recalling $g_s$ is the stochastic oracle evaluation, $\mathcal{G}(x_s)$, the numerator decomposes into two components:
\begin{align}
     \underset{\text{weighted regret}}{\underbrace{\sum_{s=0}^{t-1} \bar{r}_s^2 \langle -g_s, \exp_{x_s}^{-1}(x_\star)\rangle_{x_s}}} +  \underset{\text{noise}}{\underbrace{\sum_{s=0}^{t-1} \bar{r}_s^2 \langle \Delta_s, \exp_{x_s}^{-1}(x_\star)\rangle_{x_s}}},
\end{align}
with $\Delta_s \coloneqq g_s - \operatorname{grad} f(x_s)$.

\subsubsection*{Supporting Analysis}
We give a deterministic bound for the weighted regret (\cref{lemma:dowg_supporting}).
\begin{lemma}
    \label{lemma:dowg_supporting}
    Under \cref{assumption:g_convex} and \ref{assumption:Hadamard}, the iterates of CO-RDoWG, satisfy
    \begin{align}
       \sum_{s=0}^{t-1} \bar{r}_s^2 \langle  -g_s, \exp_{x_s}^{-1}(x_\star)\rangle_{x_s} \le \bar{r}_t \left( 2 \bar{d}_t + \bar{r}_t \zeta_{\kappa}(\bar{d}_t) \right) \sqrt{v_{t-1}}.
   \end{align}
    \begin{proof}
           Follow the same argument as \cref{lemma:weighted_regret_bound_unknown_curvature} but with weights $\bar{r}_s^2$ replacing $\bar{r}_s$ and weighted gradient sum $v_s$ replacing the standard gradient sum $G_s$.
    \end{proof}
\end{lemma}

\subsubsection*{Non-Smooth Analysis}
We give deterministic bounds for the weighted regret (\cref{lemma:dowg_nonsmooth_weighted_regret_unknown_curvature}) and high probability bounds for the noise term (\cref{lemma:dowg_nonsmooth_prob_unknown_curvature}) in the non-smooth setting.

\begin{lemma}
\label{lemma:dowg_nonsmooth_weighted_regret_unknown_curvature}
    Under \cref{assumption:g_convex} and \ref{assumption:Hadamard}, the iterates of CO-RDoWG, satisfy
    \begin{align}
       \sum_{s=0}^{t-1} \bar{r}_s^2 \langle  -g_s, \exp_{x_s}^{-1}(x_\star)\rangle_{x_s} \le \bar{r}_t^2 \left( 2 \bar{d}_t + \bar{r}_t \zeta_{\kappa}(\bar{d}_t) \right) \sqrt{G_{t-1}}.
   \end{align}
    \begin{proof}
        Using the bound of \ref{lemma:dowg_supporting} and that $\sqrt{v_{t-1}} \le \bar{r}_{t-1} \sqrt{G_{{t-1}}} \le \bar{r}_{t} \sqrt{G_{{t-1}}}$ gives the result.
    \end{proof}
\end{lemma}

\begin{lemma}
\label{lemma:dowg_nonsmooth_prob_unknown_curvature}
   For all $\delta \in (0,1), T \in \mathbb{N}$ and $L > 0$, if \cref{assumption:g_convex}, \ref{assumption:Hadamard}, and \ref{assumption:bounded_stochastic_grads} hold, the iterates of CO-RDoWG satisfy
    \begin{align}
       \mathbb{P} \left( \exists t \le T : \left\lvert  \sum_{s=0}^{t-1} \bar{r}_s^2 \langle \Delta_s, \exp_{x_s}^{-1}(x_\star)\rangle_{x_s} \right\rvert \ge b_t \right) \le \delta + \mathbb{P}(\bar{\ell}_{T} > L),
   \end{align}
   where $b_t =  8 \bar{r}_{t-1}^2 \bar{d}_{t-1} \sqrt{\theta_{t, \delta} G_{t-1} + \theta^{2}_{t, \delta} L^2}$ and $\bar{\ell}_{T} \coloneqq \max_{s \le T} \ell(x_s)$.
   \begin{proof}
       Follow argument of \cref{lemma:noise_bound_unknown_curvature}.
   \end{proof}
\end{lemma}

Combining the above results, we obtain the following. 

\begin{theorem}
  \label{thm:dowg_nonsmooth_unknown_curvature}
  For all $\delta \in (0, 1)$ and $L>0$, if  \cref{assumption:g_convex}, \ref{assumption:Hadamard}, and \ref{assumption:bounded_stochastic_grads} hold, then with probability at least $1 - \delta - \mathbb{P}(\bar{\ell}_{T} > L)$, for all $t \le T$, the optimality gap on the weighted iterates $f(\tilde{x}_t) - f(x_\star)$ of CO-RDoWG satisfy
    \begin{align} 
            O\left(\frac{(d_0 + \bar{r}_t)\zeta_{\kappa}(d_0 + \bar{r}_t) \sqrt{G_{t-1} + \theta_{t, \delta} G_{t-1} + \theta^{2}_{t, \delta} L^2}}{\sum_{s=0}^{t-1} \bar{r}_s^2/\bar{r}_t^2} \right).
    \end{align}
  \begin{proof}
  Combine \cref{lemma:dowg_nonsmooth_weighted_regret_unknown_curvature} and \cref{lemma:dowg_nonsmooth_prob_unknown_curvature} and use the fact $\bar{d}_t \le d_0 + \bar{r}_t$.
\end{proof}
\end{theorem}
Thus in comparison to standard RDoWG, we pay an additional cost of $O\left(\sqrt{\zeta_{\kappa}(d_0 + \bar{r}_t)}\right)$ for omitting the geometric curvature term with CO-RDoWG.

We then have a useful result when the manifold is bounded but its exact diameter is unknown.
\begin{corollary}
    Under \cref{assumption:g_convex}, \ref{assumption:Hadamard}, and \ref{assumption:bounded_stochastic_grads}, for any $D \ge d_0$ let $L_{D} \coloneqq \max_{x\in \mathcal{M} : d(x, x_0) \le D} \ell(x)$. Then, for all $\delta \in (0, 1)$ and for $\tau \in \arg \max_{t \le T} \sum_{s=0}^{t-1} \frac{\bar{r}_s^2}{\bar{r}_t^2}$, with probability at least $1 - \delta - \mathbb{P}(\bar{\ell}_{T} > L)$, iterates of CO-DoWG satisfy the optimality gap bound
    \begin{align}
        f(\tilde{x}_{\tau}) - f(x_\star) = O\left( \frac{D\zeta_{\kappa}(D) \sqrt{G_{\tau-1} \theta_{\tau, \delta} + L^2_D \theta_{\tau, \delta}^2}}{T} \log_{+}(D/\epsilon) \right).
    \end{align}
    \begin{proof}
        Apply \cref{lemma:fraction_sum} to the denominator term of \cref{thm:dowg_nonsmooth_unknown_curvature}.
    \end{proof}
 \end{corollary}
Thus in comparison to standard RDoWG, we pay an additional cost of $O\left(\sqrt{\zeta_{\kappa}(D)}\right)$ for omitting the curvature term with CO-RDoWG.

\subsubsection*{Smooth Analysis}

\begin{lemma}
    \label{lemma:dowg_smooth_weighted_regret_unknown_curvature}
    Suppose $f$ is $S$-smooth and assume \cref{assumption:g_convex} and \ref{assumption:Hadamard} hold. Then we have that the iterates of CO-RDoWG satisfy
    \begin{align}
       \sum_{s=0}^{t-1} \bar{r}_s^2 \langle  -g_s, \exp_{x_s}^{-1}(x_\star)\rangle_{x_s} \le \bar{r}_t \left( 2 \bar{d}_t + \bar{r}_t \zeta_{\kappa}(\bar{d}_t) \right) \sqrt{2S \sum_{s=0}^{t-1} \bar{r}_s^2 (f(x_s) - f(x_\star))}.
   \end{align}

    \begin{proof}
        By smoothness we can use \cref{lemma:lsmooth_bound}  to deduce $\lVert \operatorname{grad}f(x) \rVert_x^2 \le 2S(f(x) - f(x_\star))$ for all $x \in \mathcal{M}$. Therefore
        \begin{align}
            v_t = \sum_{s=0}^{t-1} \bar{r}_s^2 \lVert \operatorname{grad}f(x_s) \rVert_{x_s}^2 \le 2S \sum_{s=0}^{t-1} \bar{r}_s^2 (f(x_s) - f(x_\star)).
        \end{align}
        Taking square roots and substituting this into \cref{lemma:lsmooth_bound} gives the result.
    \end{proof}
\end{lemma}

\begin{lemma}
   \label{lemma:dowg_smooth_prob_bound_unknown_curvature}
   Suppose  \cref{assumption:g_convex}, \ref{assumption:Hadamard}, and  \ref{assumption:bounded_stochastic_lsmooth} hold. Then for all $\delta \in (0,1), T \in \mathbb{N}$ and $S > 0$,  Then  we have that the iterates of CO-RDoWG satisfy
    \begin{align}
       \mathbb{P} \left( \exists t \le T : \left\lvert  \sum_{s=0}^{t-1} \bar{r}_s^2 \langle \Delta_s, \exp_{x_s}^{-1}(x_\star)\rangle_{x_s} \right\rvert \ge b_t \right) \le \delta + \mathbb{P}(\bar{s}_{T} > S),
   \end{align}
   where $b_t =  8\bar{r}_{t-1} \bar{d}_{t-1} \sqrt{2S \theta_{t, \delta} + 8S \theta_{t, \delta}^2} \sqrt{\sum_{s=0}^{t-1} \bar{r}_s^2 \left[f(x_s) - f(x_\star)\right]}$ and $\bar{s}_{T} \coloneqq \max_{t \le T} s(x_t)$.
   
   \begin{proof}
       Follow argument of \cref{lemma:dowg_smooth_prob_bound_curvature}.
   \end{proof}
\end{lemma}

Combining the above results, we obtain the following. 
\begin{theorem}
    \label{thm:dowg_smooth_unknown_curvature}
    For all $\delta \in (0, 1)$ and $S>0$, if  \cref{assumption:g_convex}, \ref{assumption:Hadamard}, and \ref{assumption:bounded_stochastic_lsmooth}  hold, then with probability at least $1 - \delta - \mathbb{P}(\bar{s}_{T} > S)$, for all $t \le T$,  CO-RDoWG satisfies the optimality gap $f(\tilde{x}_t) - f(x_\star)$ of
    \begin{align}
        O \left(\frac{ \left( (d_0 + \bar{r}_t) \zeta_{\kappa}(d_0 + \bar{r}_t) \right)^2 (S \theta_{t, \delta}^2) }{\sum_{s=0}^{t-1} \frac{ \bar{r}_s^2}{\bar{r}_t^2}} \right).
    \end{align}
    
\begin{proof}
Using \cref{lemma:dowg_smooth_weighted_regret_unknown_curvature} and  \cref{lemma:dowg_smooth_prob_bound_unknown_curvature} above, we have with the relevant probabilistic conditions,  
\begin{align}
    \sum_{s=0}^{t-1} \bar{r}_s^2 [f(x_s) - f(x_\star)] \le \left (\sqrt{2S} \bar{r}_t \left( 2 \bar{d}_t + \bar{r}_t \zeta_{\kappa}(\bar{d}_t) \right) + 8\bar{r}_t \bar{d}_t \sqrt{2S \theta_{t, \delta} + 8S \theta_{t, \delta}^2} \right) \sqrt{\sum_{s=0}^{t-1} \bar{r}_s^2 \left[f(x_s) - f(x_\star)\right]}.
\end{align}
Now if $f(x_s) - f(x_\star) = 0$ for some iterate, then the statement is trivial. Otherwise diving by sides by the square root term, we have
\begin{align}
\sqrt{\sum_{s=0}^{t-1} \bar{r}_s^2 \left[f(x_s) - f(x_\star)\right]} \le  \left (\sqrt{2S} \bar{r}_t \left( 2 \bar{d}_t + \bar{r}_t \zeta_{\kappa}(\bar{d}_t) \right) + 8 \bar{r}_t \bar{d}_t \sqrt{2S \theta_{t, \delta} + 8S \theta_{t, \delta}^2} \right).
\end{align}
We square both sides and divide through by $\sum_{s=0}^{t-1} \bar{r}_s^2$,
\begin{align}
\frac{1}{\sum_{s=0}^{t-1} \bar{r}_s^2} \sum_{s=0}^{t-1} \bar{r}_s^2 \left[f(x_s) - f(x_\star)\right] \le O \left(\frac{ \left( 2 \bar{d}_t + \bar{r}_t \zeta_{\kappa}(\bar{d}_t) \right)^2 (S \theta_{t, \delta}^2) }{\sum_{s=0}^{t-1} \frac{ \bar{r}_s^2}{\bar{r}_t^2}} \right).
\end{align}
Now using the fact, $\bar{d}_t \le d_0 + \bar{r}_t$, in the above bound gives the result.
\end{proof}
\end{theorem}

Thus in comparison to standard RDoWG, we pay an additional cost of $O\left(\sqrt{\zeta_{\kappa}(d_0 + \bar{r}_t)}\right)$ for omitting the geometric curvature term with CO-RDoWG.

We then have a useful result when the manifold is bounded but its exact diameter is unknown.
\begin{corollary}
    Under \label{thm:dowg_smooth_known_curvature}
    \cref{assumption:g_convex}, \ref{assumption:Hadamard}, and \ref{assumption:bounded_stochastic_lsmooth}, for any $D \ge d_0$  let $S_{D} \coloneqq \max_{x\in \mathcal{M} : d(x, x_0) \le D} s(x)$. Then, for all $\delta \in (0, 1)$ and for $\tau \in \arg \max_{t \le T} \sum_{s=0}^{t-1} \frac{\bar{r}_s^2}{\bar{r}_t^2}$, with probability at least $1 - \delta - \mathbb{P}(\bar{s}_{T} > S)$, the iterates of CO-RDoWG satisfies the optimally gap $f(\tilde{x}_{\tau}) - f(x_\star)$ of
    \begin{align}
         O\left(\frac{D^2 \zeta_{\kappa}(D)^2 S_D \theta_{\tau, \delta}^2}{T} \log_{+}(D/\epsilon) \right).
    \end{align}
    \begin{proof}
        Apply \cref{lemma:fraction_sum} to the denominator term of  \cref{thm:dowg_smooth_unknown_curvature}.
    \end{proof}
 \end{corollary}
Thus in comparison to standard RDoWG, we pay an additional cost of $O\left(\sqrt{\zeta_{\kappa}(D)}\right)$ for omitting the curvature term with CO-RDoWG.

\section{NRDoG Overview}
\label{sec:nrdog}
Here we present a learning-rate-free schedule for NRSGD: \emph{Normalized Riemannian Distance over Gradients} (NRDoG).

\begin{algorithm}[!htbp]
   \caption{NRDoG}
   \label{alg:nrdog}
\begin{algorithmic}
   \STATE {\bfseries Input:} initial point $x_0$, initial estimate $\epsilon >0$, $G_{-1}=0$.
   \FOR{$t=0$ {\bfseries to} $T-1$}
    %\STATE Play $x_t$ and receive $f$.
    \STATE $g_{t} = \mathcal{G}(x_t)$
    \STATE $\bar{r}_t = \max\left(\epsilon, \max_{s\le t} d (x_{s}, x_{0})\right)$
    \STATE $\eta_t = \frac{\bar{r}_t} {\sqrt{ ({t+1}) \zeta_\kappa( \bar{r}_t)} }$
    \STATE $x_{t+1} = \exp_{x_{t}} \left( - \eta_t \frac{g_{t}}{\lVert g_{t} \rVert_{x_t}} \right)$
   \ENDFOR
\end{algorithmic}
\end{algorithm}

\section{Geometry of Specific Riemannian Manifolds}
\label{sec:geometry}

\subsection{Sphere Manifold}
The sphere manifold $\mathbb{S}^{d-1} \coloneqq \{x \in \mathbb{R}^d: \| x\| = 1 \}$ is an embedded submanifold of $\mathbb{R}^d$ with tangent space $\mathcal{T}_x \mathbb{S}^{d-1} = \{v \in \mathbb{R}^{d} : x^{T} v = 0\}$. The Riemannian metric is given by the Euclidean inner product $\langle \cdot, \cdot' \rangle_x = \langle \cdot, \cdot' \rangle$. The exponential map is given by $\exp_{x}(v) = \cos(\|v\|)x + \sin(\|v\|)\frac{v}{\|v\|}$ with inverse exponential map as $\exp^{-1}_{x}(y) = \arccos(x^{T}y)\frac{\operatorname{Proj}_{x}(y-x)}{\|\operatorname{Proj}_{x}(y-x)\|}$ where  $\operatorname{Proj}_{x}(v) = v - (x^{T} v)x$ is the orthogonal projection of any $v \in \mathbb{R}^{d}$ to the tangent space $\mathcal{T}_x \mathbb{S}^{d-1}$. Following the Pymanopt implementation \citep{Pymanopt}, parallel transport is approximated with the projection operation, i.e., $\Gamma_x^y v \approx \operatorname{Proj}_{x}(v)$.

\subsection{Grassmann Manifold}
The Grassmann manifold of dimension $d \times r$, denoted as $\mathbb{G}(d, r)$ is the set of all $r$ dimensional
subspaces in $\mathbb{R}^d$ $(d \ge r)$. Each point on the Grassmann manifold can be identified as a column of orthonormal matrices $x \in \mathbb{R}^{d \times r}$, $x^{T} x = \mathrm{I}$ and two points $x, y$ are equivalent if $x = y o$ for some $r \times r$ orthogonal
matrix $o$. For our implementation of the exponential map, inverse exponential map, and parallel transport, we directly translate the Pymanopt code \citep{Pymanopt} from NumPy to JAX.

\subsection{Poincar\'{e} Manifold}
The Poincar\'{e} manifold of dimension $d$ is given by the open $d$-dimensional unit ball $\mathbb{B}_d \coloneqq \{ x \in \mathbb{R}^d: \lVert x \rVert < 1 \}$ equipped with Riemannian metric  $\langle \cdot, \cdot' \rangle_x = 4/(1 - \lVert x \rVert^2)^2 \langle \cdot, \cdot' \rangle$. The \emph{Möbius addition} of $x$ and $y$ in $\mathbb{B}^d$ is defined as \citep{ungar2022gyrovector}
\begin{align*}
    x \oplus y \coloneqq \frac{(1+2 \langle x, y \rangle + \lVert y \rVert^2)x + (1-\lVert x \rVert^2) y}{1 + 2 \langle x, y \rangle + \lVert x \rVert^2 \lVert y \rVert^2}.
\end{align*}
Defining the \emph{conformal factor} as $\lambda_x \coloneqq 2 /(1-\lVert x \rVert^2)$, the exponential map is given by $\exp_{x}(v) = x \oplus \left(\tanh\left(\lambda_x \frac{\lVert x\rVert}{2}\right)\right) \frac{v}{\lVert v \rVert}$  and the inverse exponential map is given by $\exp^{-1}_{x}(y) = \frac{2}{\lambda_x} \tanh^{-1}(\lVert - x \oplus y \rVert) \frac{- x \oplus y}{\lVert - x \oplus y \rVert}$. Parallel transport can also be given in closed form \citep[see][for further details]{ungar2022gyrovector}.

\section{Additional Numerical Results}

\subsection{Rayleigh Quotient Maximization on the Sphere}
\label{additional:rayleigh}

In this section, we provide additional results for the Rayleigh quotient maximization discussed in \cref{experiments:rayleigh} with a consistent setup across $d=1000$ dimensions. The initial figures in \cref{ad_fig:rayleigh_regret_trace} emphasize the learning-rate-free adaptability and insensitivity to the choice of the initial distance estimate, $\epsilon \in [10^{-8}, 10^0]$, for RDoG, RDoWG, and NRDoG, particularly after a few hundred iterations. In contrast, we observe a notable impact on the performance of RSGD due to the choice of the learning rate, $\eta \in [10^{-8}, 10^0]$. This sensitivity in regret also influences solution quality, as illustrated in \cref{ad_fig:rayleigh_distance_trace}.

We proceed to evaluate the algorithms for various numbers of iterations $T\in\{100, 500, 1000, 2000\}$, showcasing regret in \cref{ad_fig:rayleigh_regret} and geodesic distance to a numerically computed optimum in \cref{ad_fig:rayleigh_distance} for different learning rates, $\eta \in [10^{-8}, 10^{6}]$, for RSGD and RADAM. Additionally, we explore different initial distance estimates, $\epsilon \in [10^{-8}, 10^{-1}]$, for RDoG, RDoWG, and NRDoG. Notably, we observe that for $T=100$ iterations, the initial distance estimate does impact the algorithms, but after $T=500$ iterations, the effect becomes insensitive over several orders of magnitude, mirroring \cref{ad_fig:rayleigh_regret} and \cref{ad_fig:rayleigh_distance}. Conversely, RADAM and RSGD exhibit a requirement for careful tuning.

\begin{figure*}[htb]
  \centering
  \subfigure[RSGD.]{\includegraphics[width=0.225\textwidth]{figs/sphere_rayleigh/sphere_sensitivity_trace_rsgd_regret.pdf}}
  \subfigure[RDoG.]{\includegraphics[width=0.225\textwidth]{figs/sphere_rayleigh/sphere_sensitivity_trace_rdog_regret.pdf}}
  \subfigure[RDoWG.]{\includegraphics[width=0.225\textwidth]{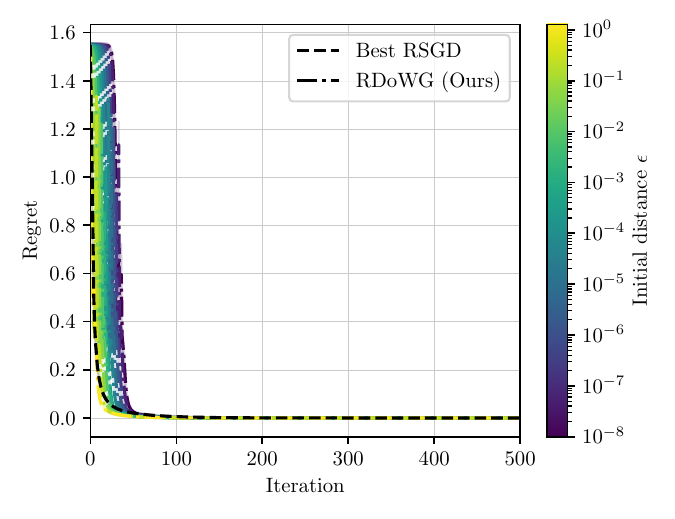}}
  \subfigure[NRDoG.]{\includegraphics[width=0.225\textwidth]{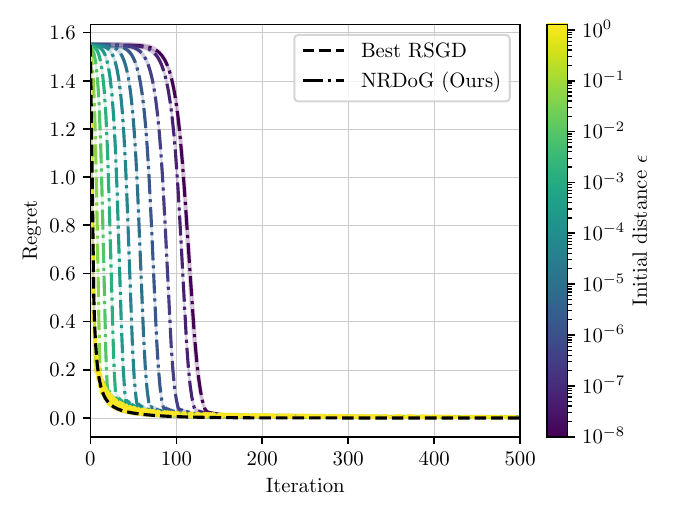}}
  \vspace{-3mm}
  \caption{\textbf{Supplementary results for Rayleigh quotient maximization on the sphere (\cref{experiments:rayleigh}).} The plots depict regret as a function of the iteration, considering various learning rates. Results are averaged over ten random replications. The optimal RSGD is chosen based on minimizing the regret after 5000 iterations. Note that (a) and (b) are equivalent to \cref{fig:rayleigh} (b) and (c) respectively.}
  \label{ad_fig:rayleigh_regret_trace}
\end{figure*}

\begin{figure*}[htb]
  \centering
  \subfigure[RSGD.]{\includegraphics[width=0.225\textwidth]{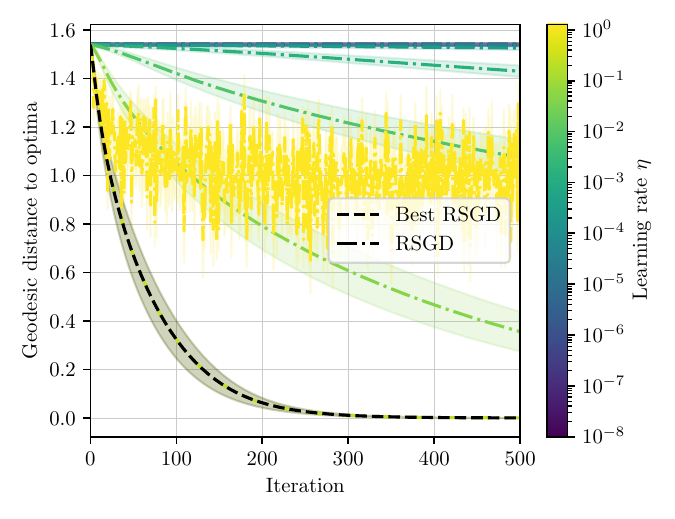}}
  \subfigure[RDoG.]{\includegraphics[width=0.225\textwidth]{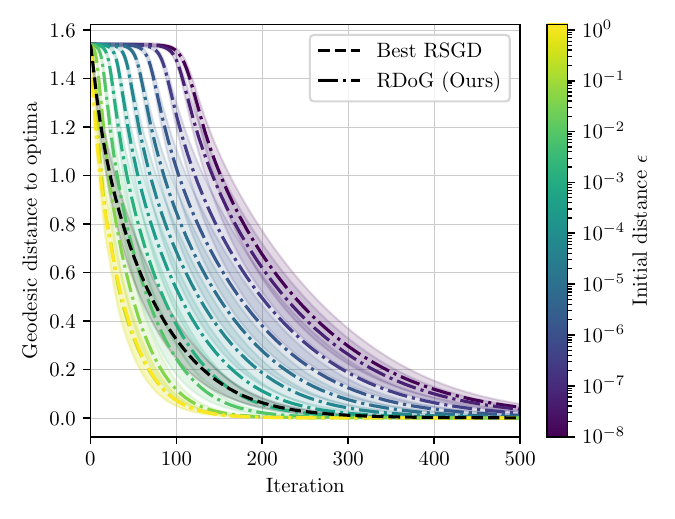}}
  \subfigure[RDoWG.]{\includegraphics[width=0.225\textwidth]{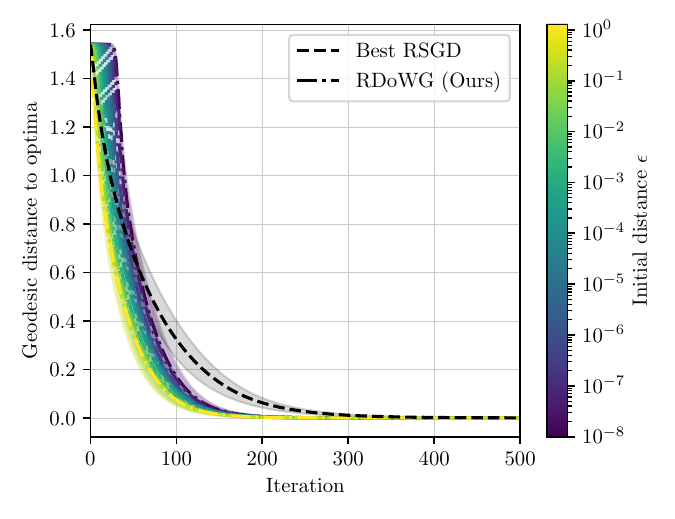}}
  \subfigure[NRDoG.]{\includegraphics[width=0.225\textwidth]{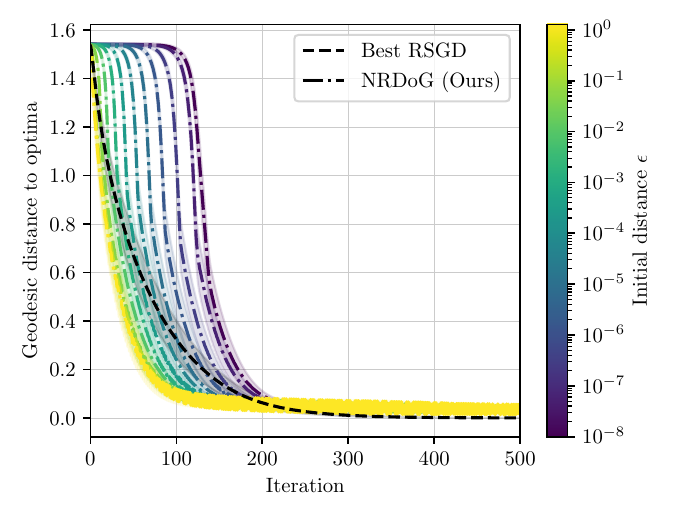}}
  \vspace{-3mm}
  \caption{\textbf{Supplementary results for Rayleigh quotient maximization on the sphere (\cref{experiments:rayleigh}).} The plots display the geodesic distance from an optimum as a function of the iteration, considering various learning rates. Results are averaged over ten replicates with different initial points. The optimal RSGD is selected based on minimizing the geodesic distance from the optimum after 5000 iterations.}
  \label{ad_fig:rayleigh_distance_trace}
\end{figure*}

\begin{figure*}[htb]
  \centering
  \subfigure[T = 100.]{\includegraphics[width=0.225\textwidth]{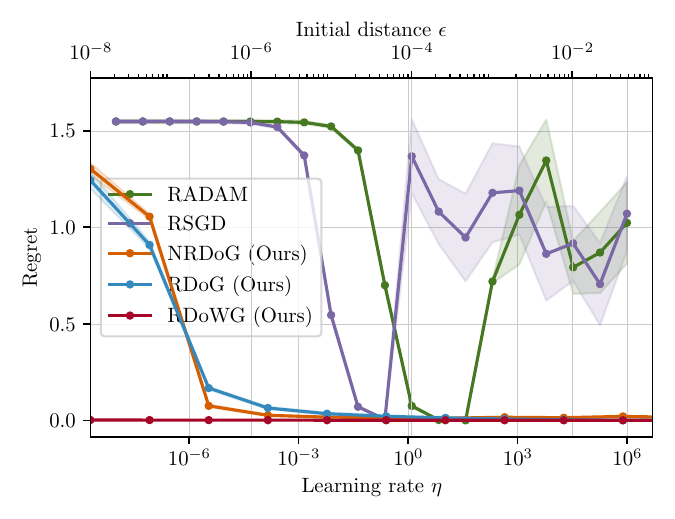}}
  \subfigure[T = 500.]{\includegraphics[width=0.225\textwidth]{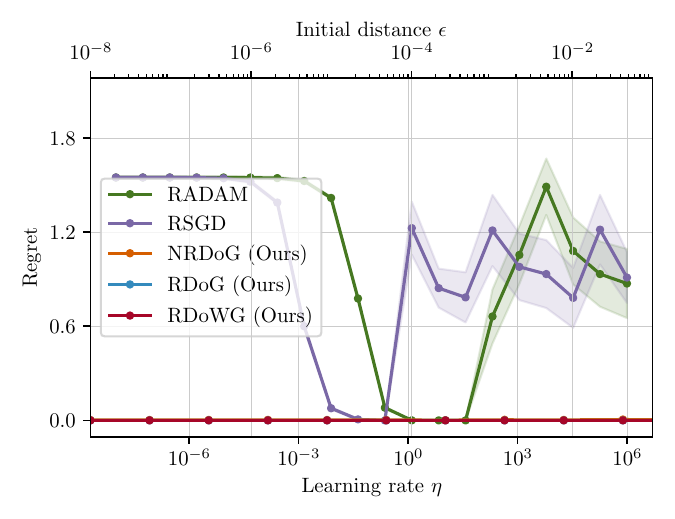}}
  \subfigure[T = 1000.]{\includegraphics[width=0.225\textwidth]{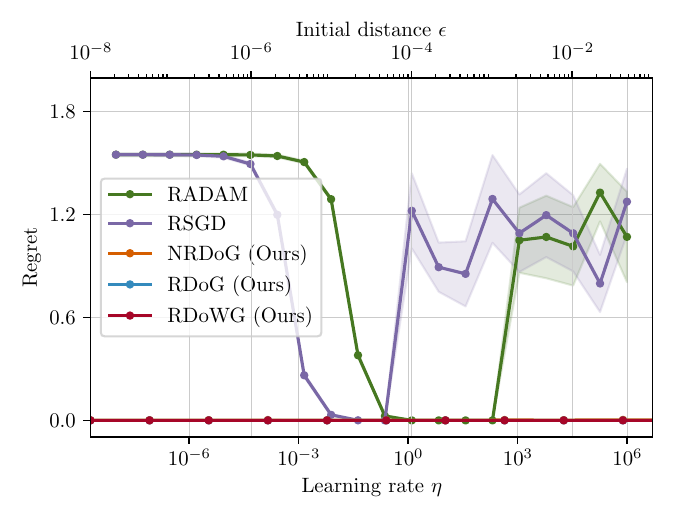}}
  \subfigure[T = 2000.]{\includegraphics[width=0.225\textwidth]{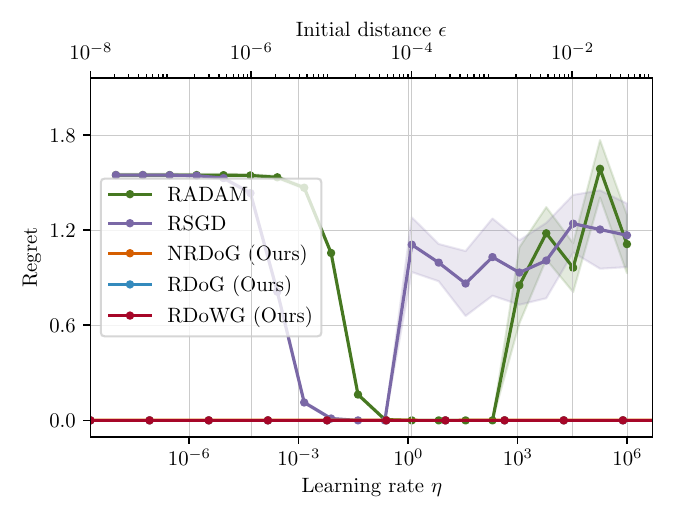}}
  \vspace{-3mm}
  \caption{\textbf{Supplementary Results for Rayleigh Quotient Maximization (\cref{experiments:rayleigh}).} Each plot illustrates the regret after the algorithm is halted for the specified number of iterations. Results are averaged over ten replicates with different initial points.}
  \label{ad_fig:rayleigh_regret}
\end{figure*}

\begin{figure*}[htb]
  \centering
  \subfigure[T = 100.]{\includegraphics[width=0.225\textwidth]{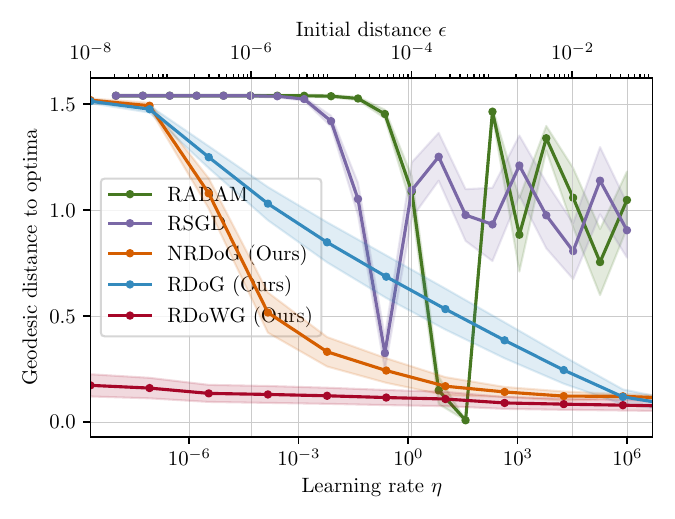}}
  \subfigure[T = 500.]{\includegraphics[width=0.225\textwidth]{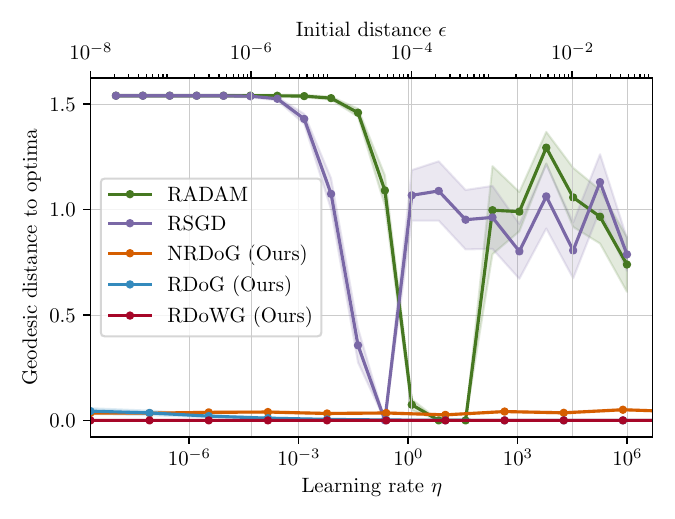}}
  \subfigure[T = 1000.]{\includegraphics[width=0.225\textwidth]{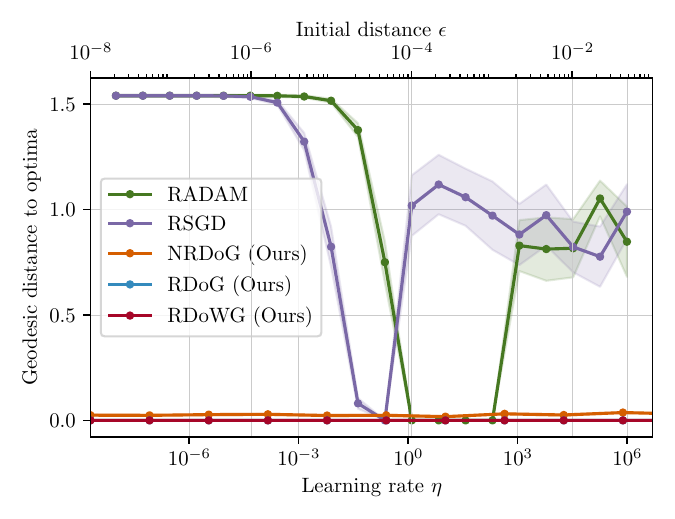}}
  \subfigure[T = 2000.]{\includegraphics[width=0.225\textwidth]{figs/sphere_rayleigh/sphere_distance_2000.pdf}}
  \vspace{-3mm}
  \caption{\textbf{Supplementary results for Rayleigh quotient maximization on the sphere (\cref{experiments:rayleigh}).} Each plot illustrates the geodesic distance to a numerically computed optimum after the algorithm is halted for the specified number of iterations. Results are averaged over ten replicates with different initial points.}
  \label{ad_fig:rayleigh_distance}
\end{figure*}

\subsection{PCA on the Grassmann Manifold}
\label{additional:pca}
In this section, we present additional results for the PCA on the Grassmann manifold discussed in \cref{experiments:pca}, maintaining a consistent experimental setup. In \cref{tab:pca_trace}, we observe that while RSGD exhibits sensitivity to the learning rate $\eta \in [10^{-8}, 10^{0}]$, RDoG, RDoWG, and NRDoG quickly adapt and achieve performance comparable to the best learning rate for RSGD within 500 iterations, irrespective of the chosen initial distance $\epsilon \in [10^{-8}, 10^{0}]$. This adaptability is further evident in \cref{tab:pca_distance}, where we consider halting the algorithms for $T\in\{100, 500, 1000, 2000\}$ iterations and comparing the geodesic distance of the output of the optimizer with the numerical solution. Discrepancies are noticeable for $T=100$, but these discrepancies diminish for $T=500$ iterations and beyond.

\begin{table}[htb]
  \centering
  \begin{tabular}{cccccc}
    & RSGD & RDoG & RDoWG & NRDoG \\
    \raisebox{1.0\height}{\rotatebox{90}{Wine}} & \includegraphics[width=0.225\textwidth]{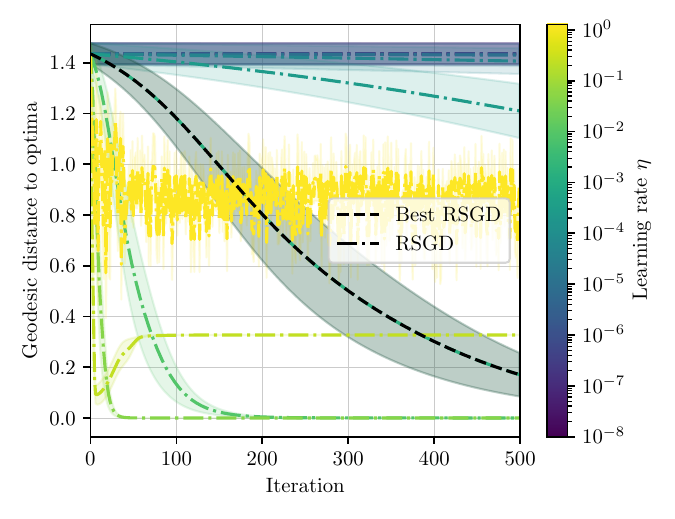} & \includegraphics[width=0.225\textwidth]{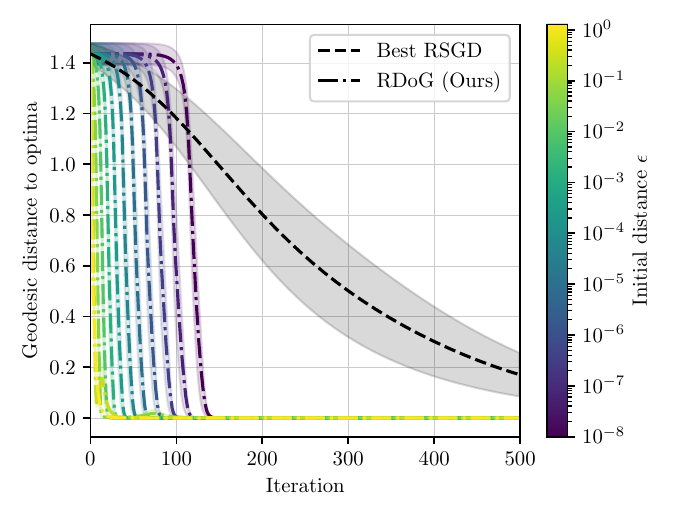} & \includegraphics[width=0.225\textwidth]{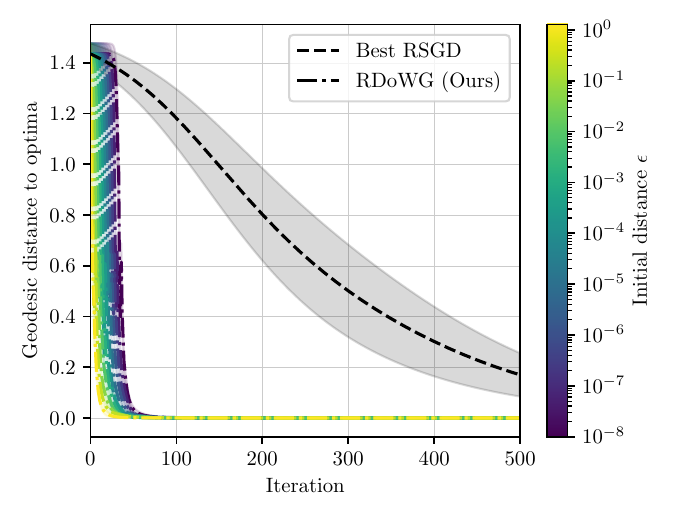} & \includegraphics[width=0.225\textwidth]{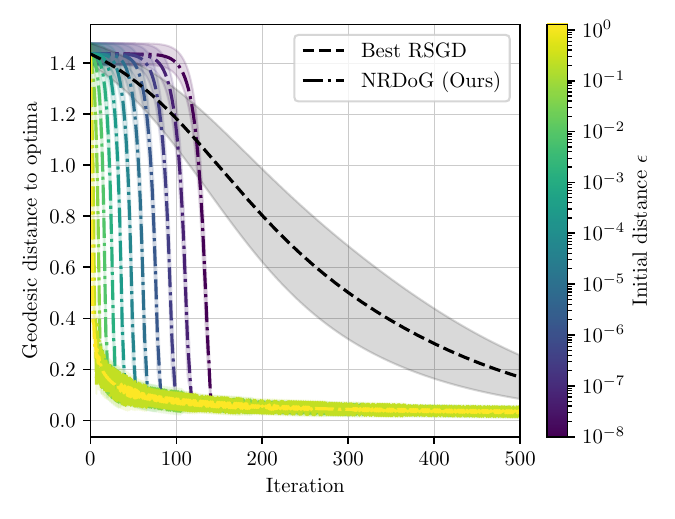} \\
   \raisebox{0.5\height}{\rotatebox{90}{Waveform}} & \includegraphics[width=0.225\textwidth]{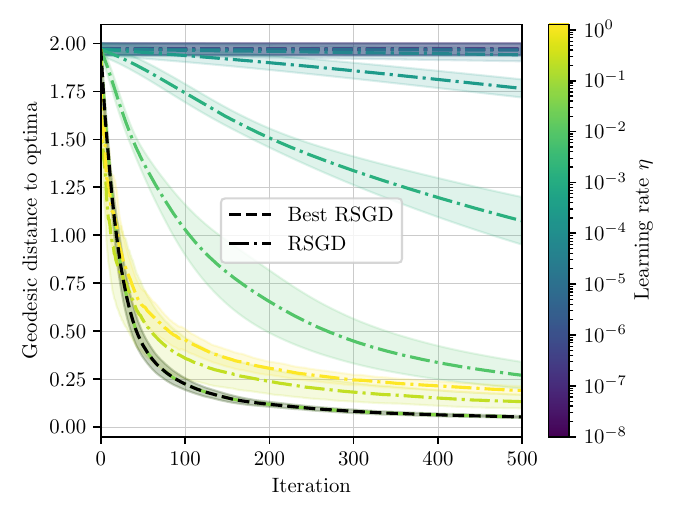} & \includegraphics[width=0.225\textwidth]{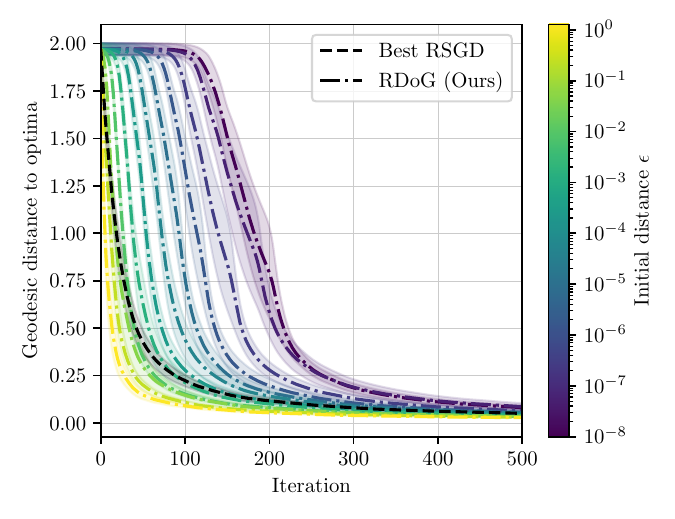} & \includegraphics[width=0.225\textwidth]{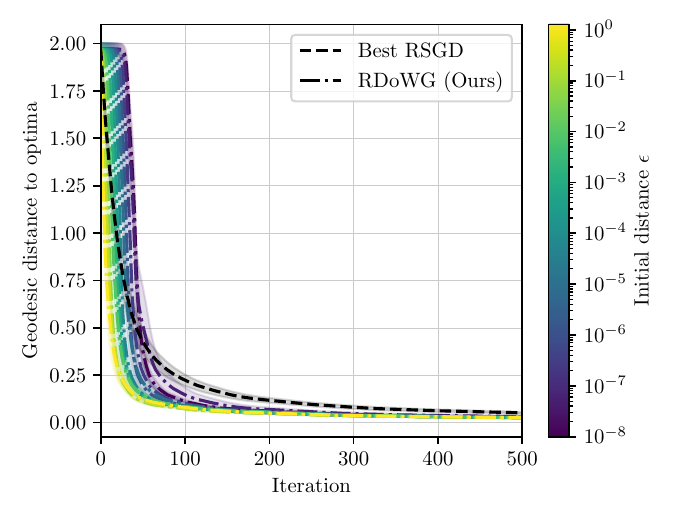} & \includegraphics[width=0.225\textwidth]{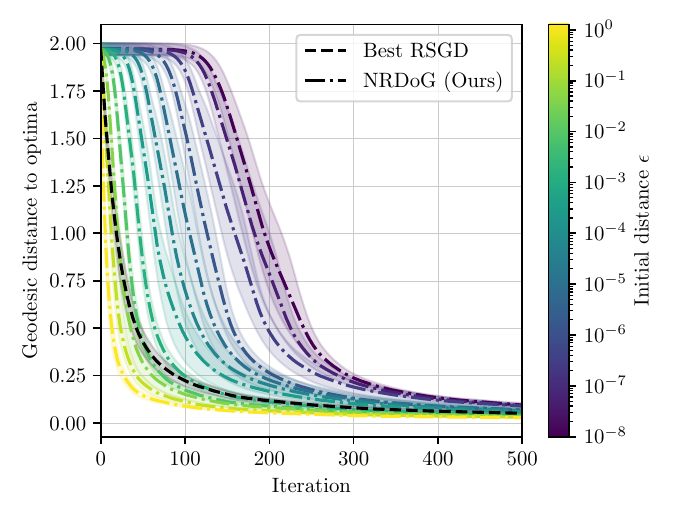} \\
   \raisebox{0.1\height}{\rotatebox{90}{Tiny ImageNet}} & \includegraphics[width=0.225\textwidth]{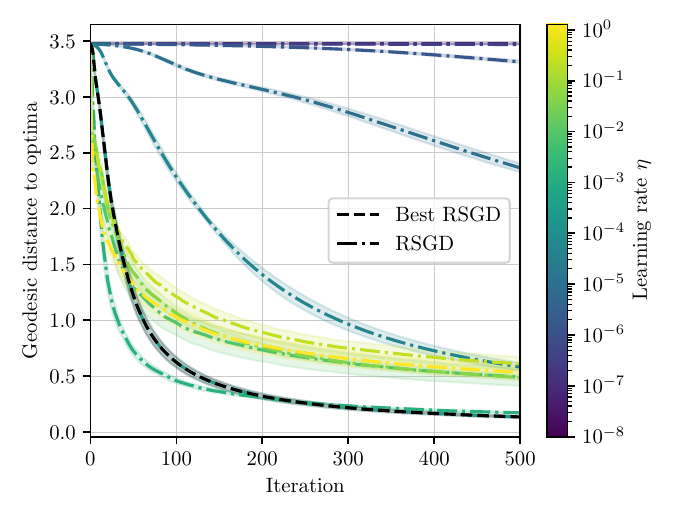} & \includegraphics[width=0.225\textwidth]{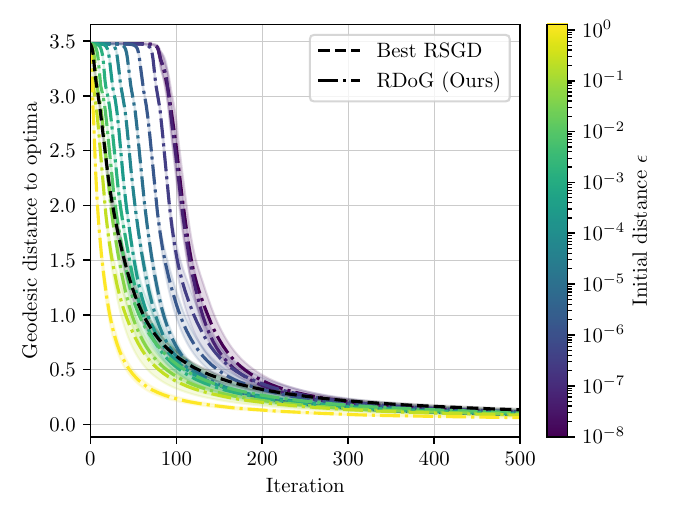} & \includegraphics[width=0.225\textwidth]{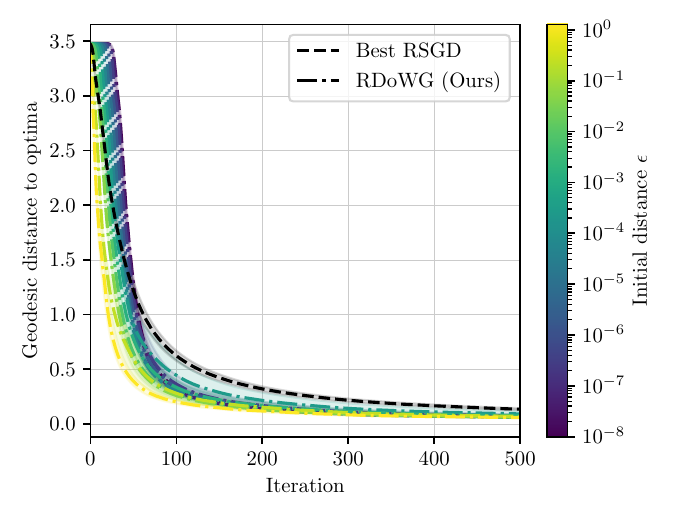} & 
   \includegraphics[width=0.225\textwidth]{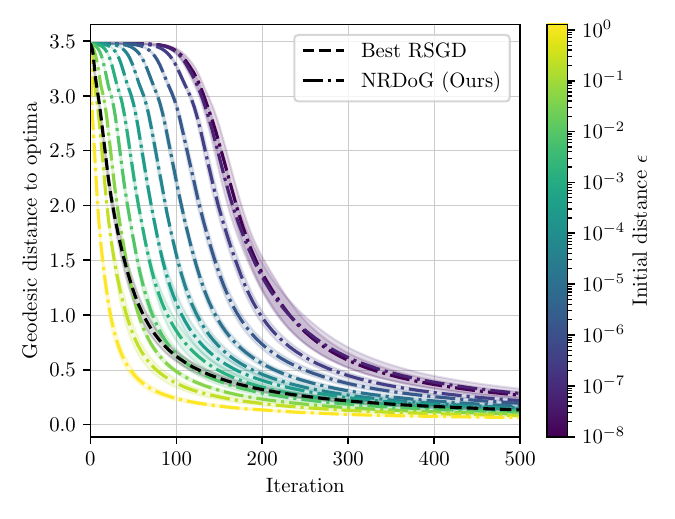}\\
  \end{tabular}
  \caption{\textbf{Supplementary results for PCA on the Grassmann manifold (\cref{experiments:pca}).} The plots display the geodesic distance from a numerically computed optimum as a function of the iteration, considering various learning rates. Results are averaged over five replicates with different initial points. The optimal RSGD is selected based on minimizing the geodesic distance from the optimum after 2000 iterations.}
  \label{tab:pca_trace}
\end{table}

\begin{table}[htb]
  \centering
  \begin{tabular}{ccccc}
    & T=100 & T=500 & T=1000 & T=2000 \\
    \raisebox{1.0\height}{\rotatebox{90}{Wine}}& \includegraphics[width=0.225\textwidth]{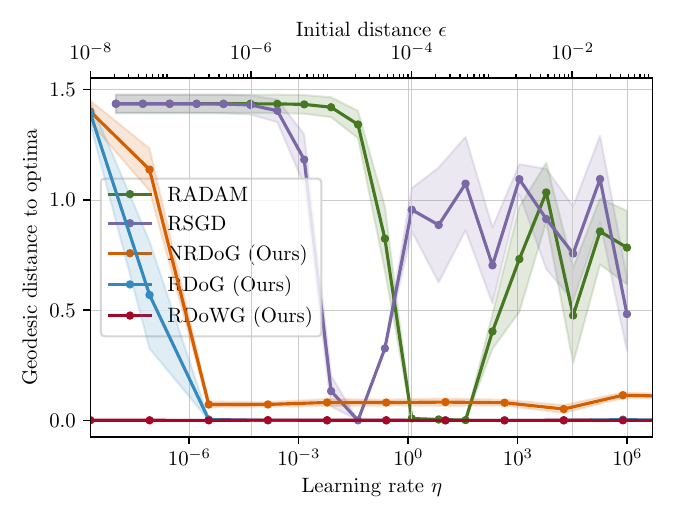} & \includegraphics[width=0.225\textwidth]{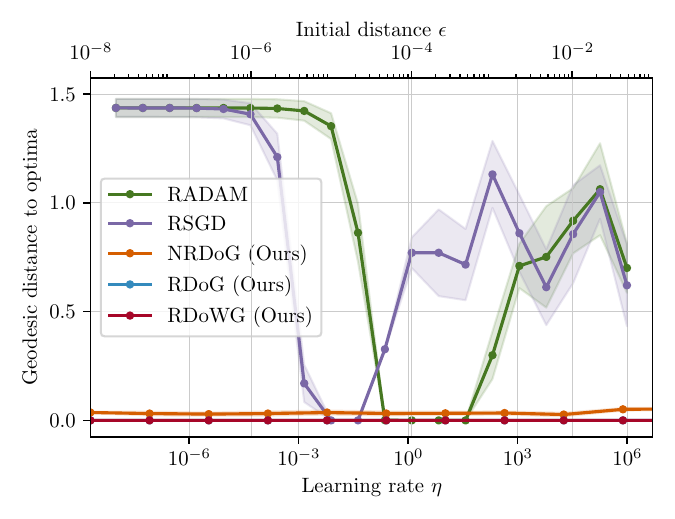} & \includegraphics[width=0.225\textwidth]{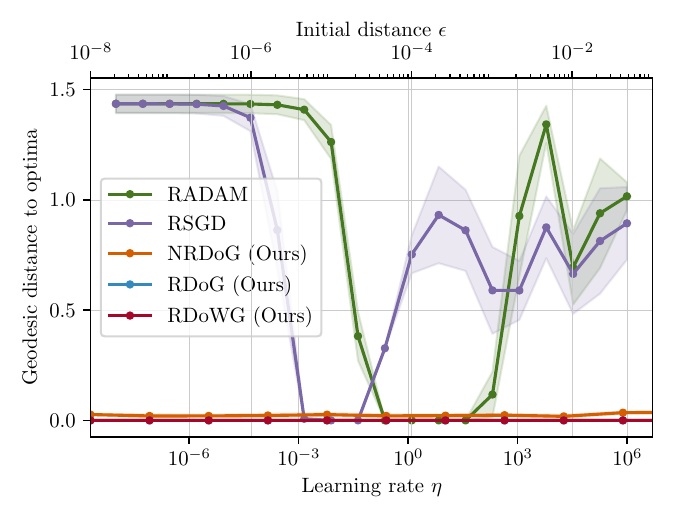} & \includegraphics[width=0.225\textwidth]{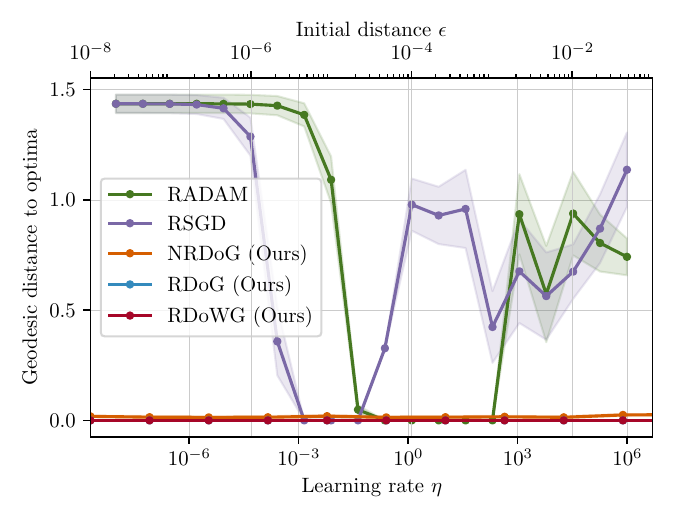} \\
   \raisebox{0.5\height}{\rotatebox{90}{Waveform}}  & \includegraphics[width=0.225\textwidth]{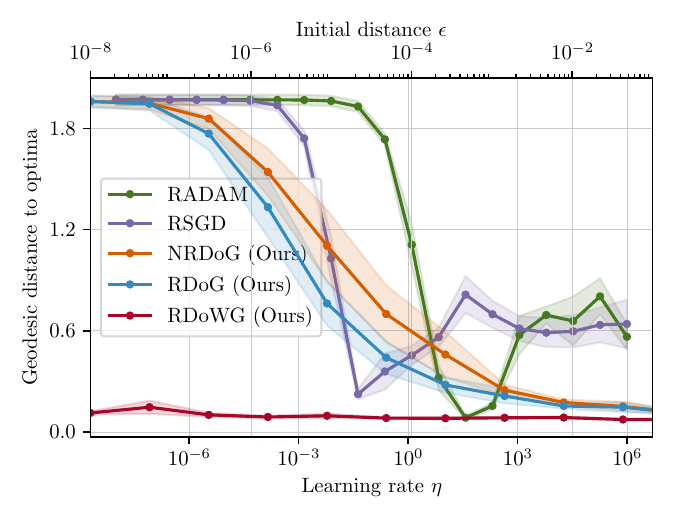} & \includegraphics[width=0.225\textwidth]{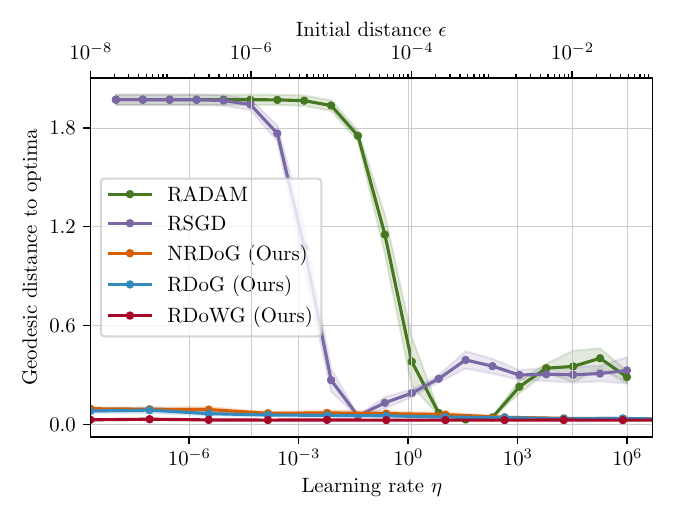} & \includegraphics[width=0.225\textwidth]{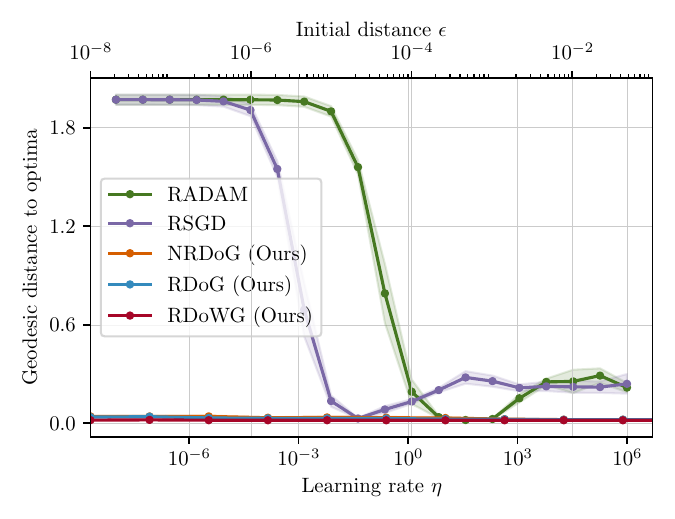} & \includegraphics[width=0.225\textwidth]{figs/waveform_pca/waveform_distance_2000.pdf} \\
   \raisebox{0.1\height}{\rotatebox{90}{Tiny ImageNet}}
    & \includegraphics[width=0.225\textwidth]{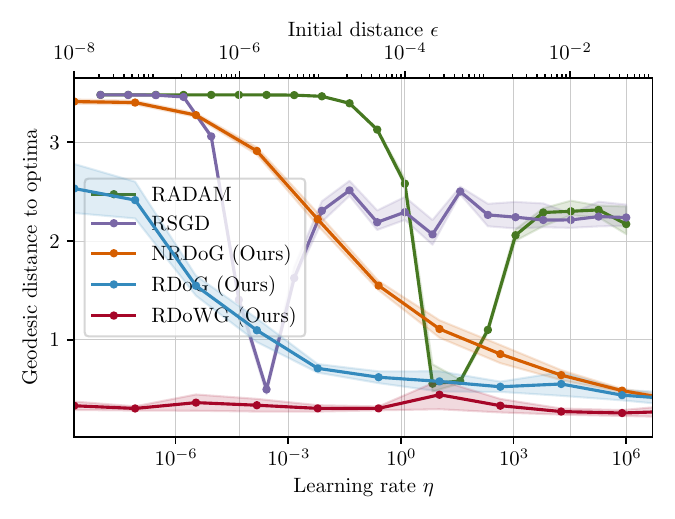} & \includegraphics[width=0.225\textwidth]{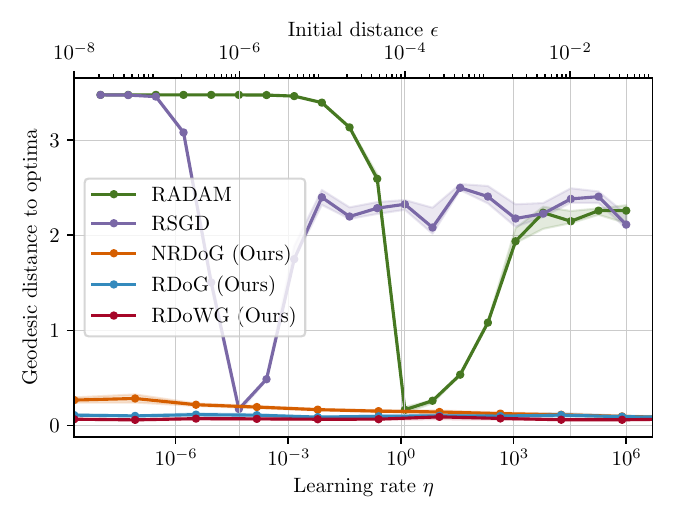} & \includegraphics[width=0.225\textwidth]{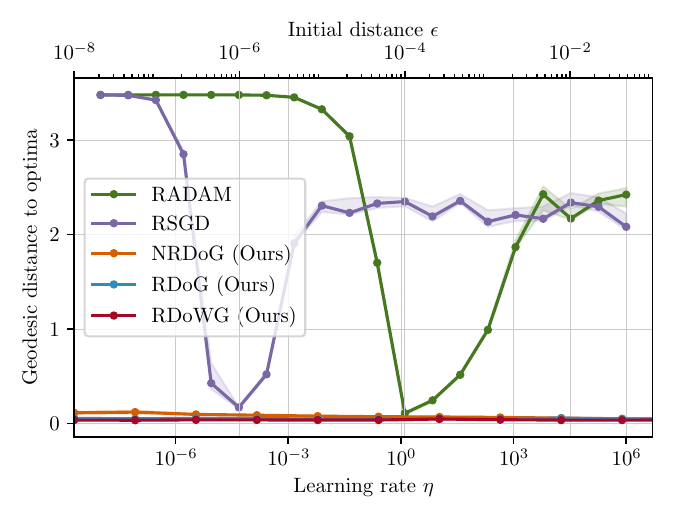} & \includegraphics[width=0.225\textwidth]{figs/tiny_pca/tiny_image_net_distance_2000.pdf} \\
  \end{tabular}
  \caption{\textbf{Supplementary results for PCA on the Grassmann manifold (\cref{experiments:pca}).} Results for different datasets and methods. Each plot illustrates the geodesic distance to a numerically computed optimum after the algorithm is halted for the specified number of iterations. Results are averaged over ten replicates with different initial points.}
  \label{tab:pca_distance}
\end{table}

\subsection{Embedding Graphs in the Poincar\'{e} Embeddings}
\label{additional:poincare}
In this section, we provide supplementary results concerning Poincar\'{e} embeddings, as detailed in \cref{experiments:poincare}. 

\Cref{tab:poincare5D} maintains a consistent experimental framework with the main paper, focusing on five-dimensional embeddings. The top-left section of the table corresponds to \cref{fig:map} presented in the main paper, serving as a reference for comparison. In the bottom-left segment, we explore the algorithmic performance of RADAM and RSGD without implementing the burn-in heuristic, which results in inferior performance. Notably, our optimizers demonstrate robustness, eliminating the need for such heuristics. On the left-hand column, we investigate the impact of omitting the curvature term from the learning rates. For RDoG and RDoWG, the curvature omission corresponds to CO-RDoG (\cref{sec:theoretical-results-RDoG-no-curvature}) and CO-RDoWG (\cref{sec:theoretical-results-RDoWG-no-curvature}). This omission leads to a performance decrease for NRDoG and RDoWG, while RDoG remains unaffected in performance.

\Cref{tab:poincare2D} adheres to a consistent experimental framework for two-dimensional embeddings outlined in the main paper. In the right-hand column, we discern meaningful groupings across various categories without resorting to burn-in heuristics for RDoG, RDoWG, and NRDoG. Conversely, in the left-hand column, we emphasize the pivotal role of geometric curvature in governing step sizes; its absence results in inferior groupings. This discrepancy is reflected in the mean average precision metric.

\begin{table}[htb]
  \centering
  \begin{tabular}{ccc}
    & With geometric curvature term & Without geometric curvature term \\
    \raisebox{0.5\height}{\rotatebox{90}{With burn-in period} }& \includegraphics[width=0.45\textwidth]{figs/mammals/mammals_sensitivity.pdf} & \includegraphics[width=0.45\textwidth]{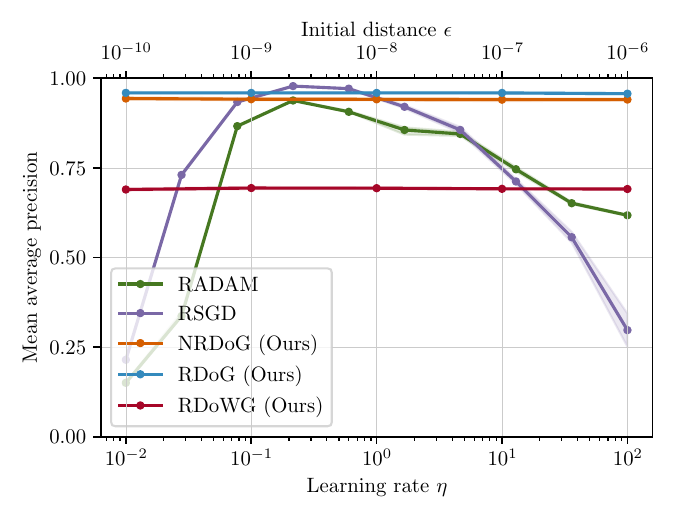} \\
    \raisebox{0.5\height}{\rotatebox{90}{Without burn-in period}}  & \includegraphics[width=0.45\textwidth]{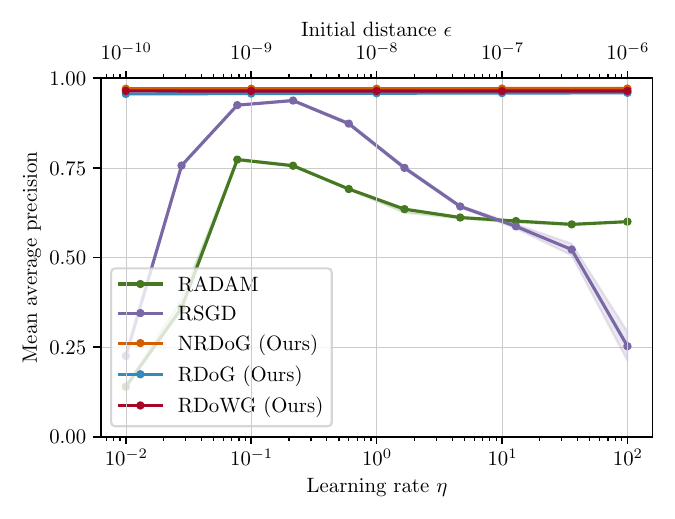} & \includegraphics[width=0.45\textwidth]{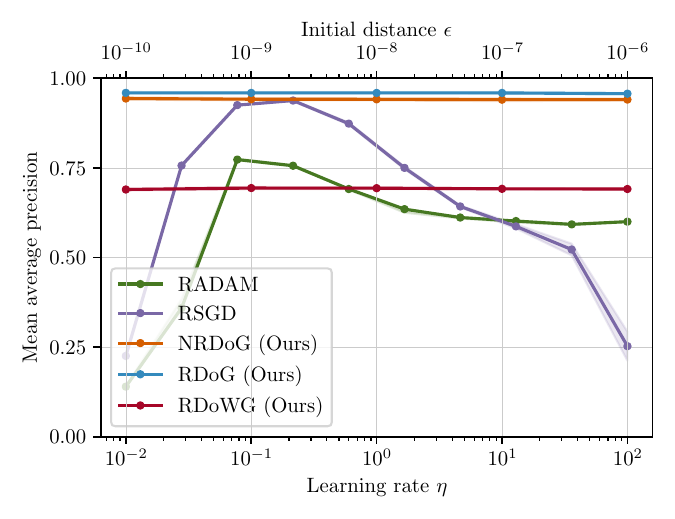} \\
  \end{tabular}
  \caption{\textbf{Supplementary results for the five-dimensional Poincar\'{e} word embeddings (\cref{experiments:poincare}).} We compute the mean average precision of the embeddings against the ground truth after 1000 training epochs. The reported results represent the average over five replications, with the dimension of the embeddings set to five. In the columns, ``with geometric curvature term" corresponds to learning schedulers for RDoG, RDoWG, and NRDoG that retain the geometric curvature term in the denominator, while ``without geometric curvature term" denotes the omission of this term. On the rows, ``with burn-in period" indicates running RADAM and RSGD with a burn-in heuristic. In this case, the algorithms are executed with learning rates divided by ten for the initial ten epochs before regular training. ``without burn-in period" signifies the absence of this heuristic. }
  \label{tab:poincare5D}
\end{table}

\begin{table}[htb]
  \centering
  \begin{tabular}{ccc}
    & With geometric curvature term & Without geometric curvature term \\
    \raisebox{0.2\height}{\rotatebox{90}{Mean average precision} }& \includegraphics[width=0.35\textwidth]{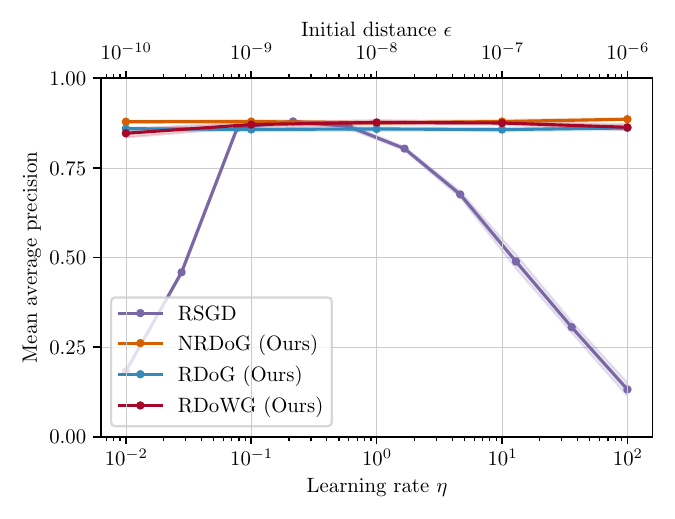} & \includegraphics[width=0.35\textwidth]{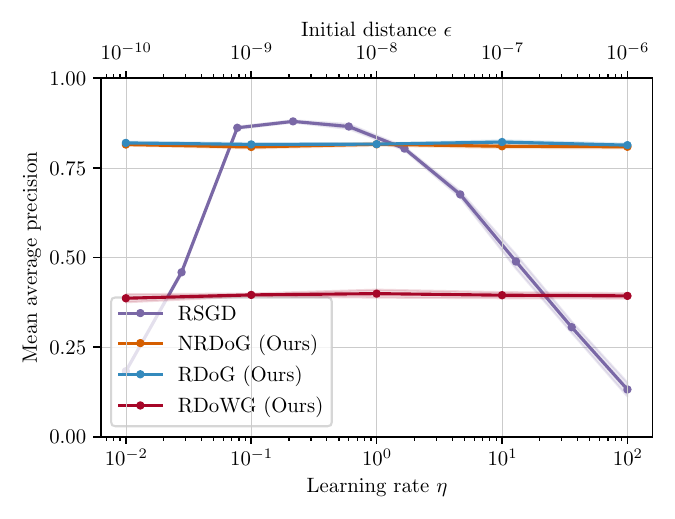} \\
    \raisebox{2.0\height}{\rotatebox{90}{RDoG}}  & \includegraphics[width=0.30\textwidth]{figs/mammals_2d/embeddings_rdog_sensitivity.pdf} & \includegraphics[width=0.30\textwidth]{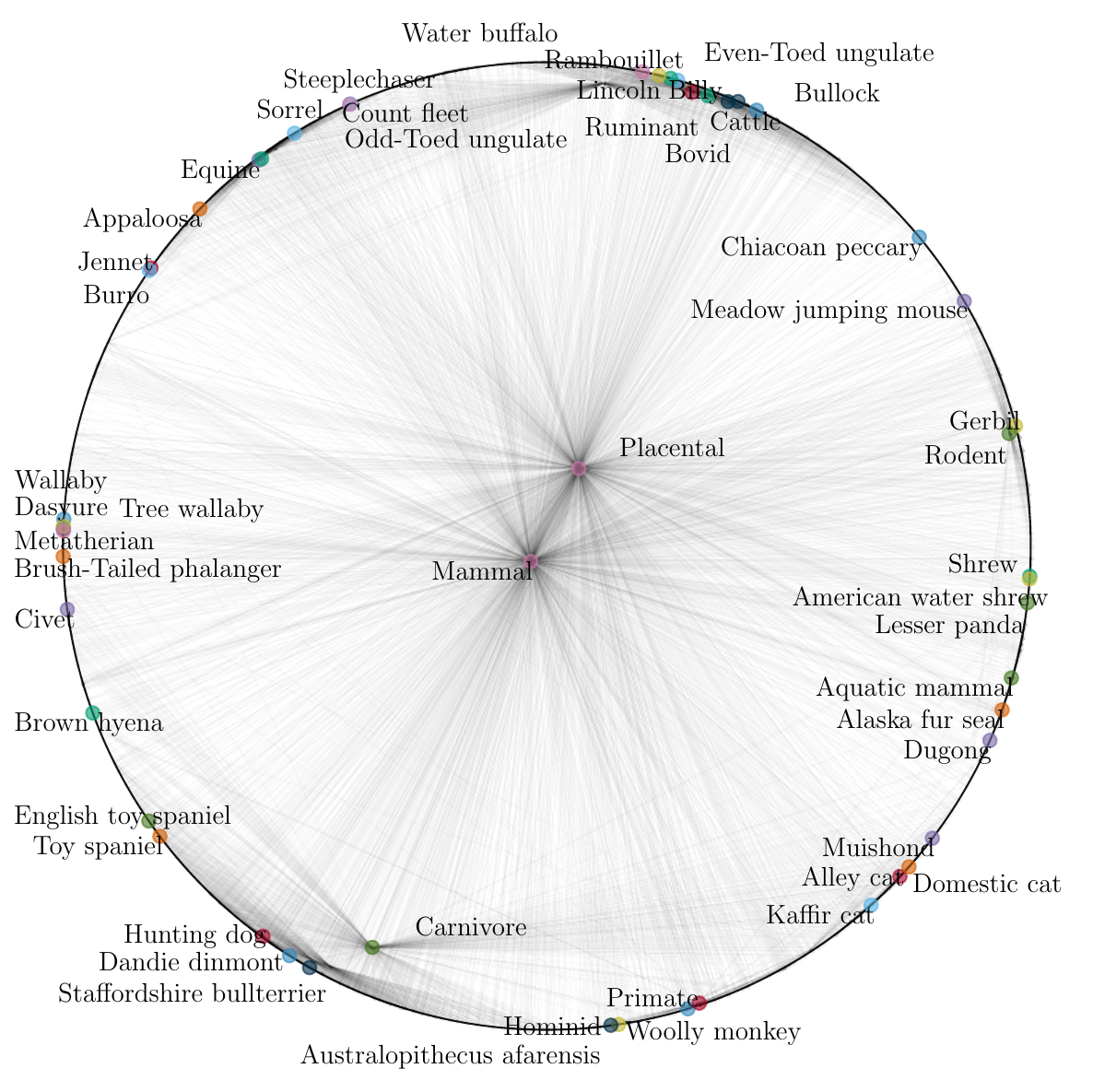} \\
    \raisebox{1.5\height}{\rotatebox{90}{NRDoG}}  & 
    \includegraphics[width=0.30\textwidth]{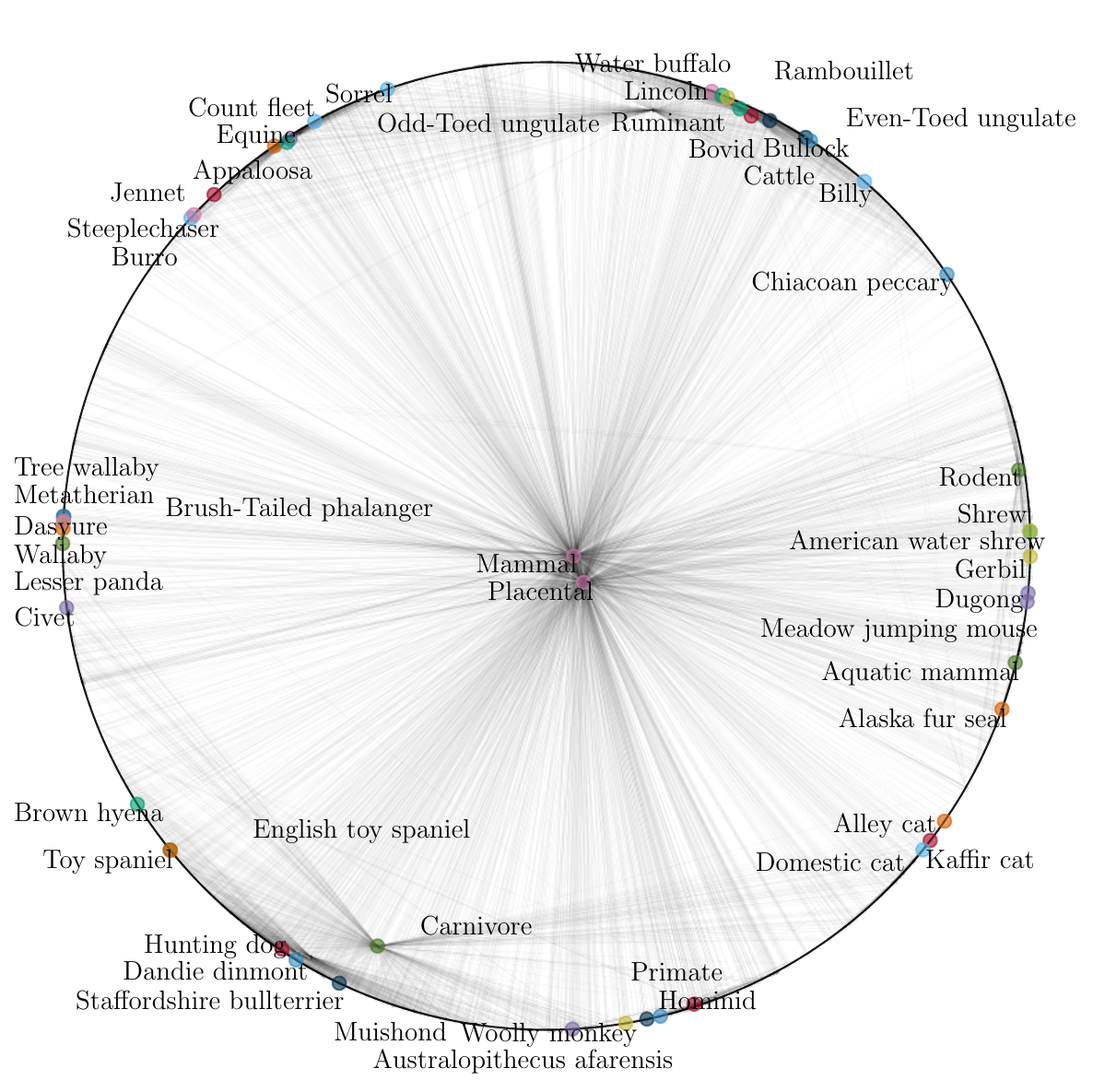} & \includegraphics[width=0.30\textwidth]{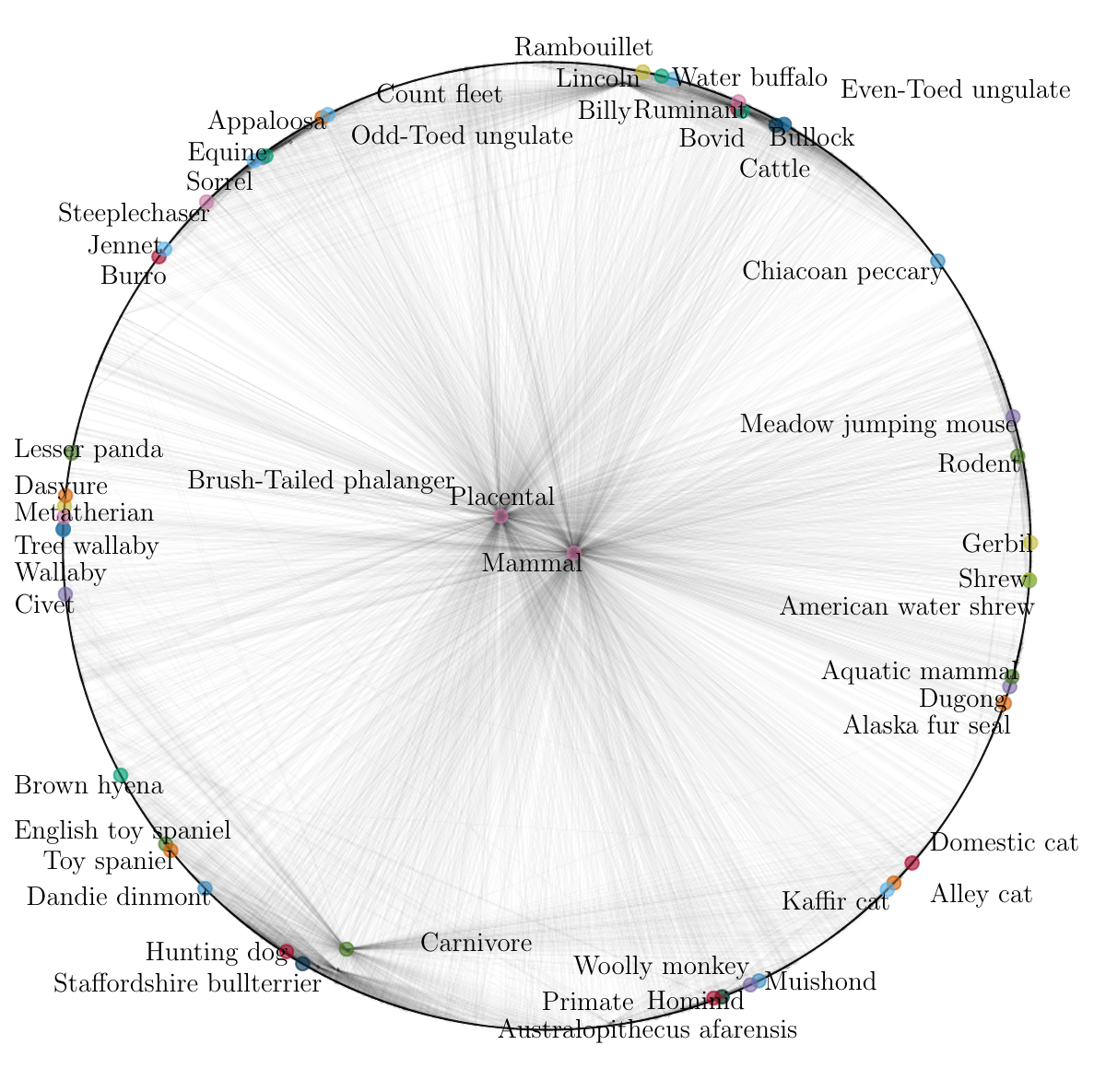} \\
    \raisebox{1.5\height}{\rotatebox{90}{RDoWG}}  & 
    \includegraphics[width=0.30\textwidth]{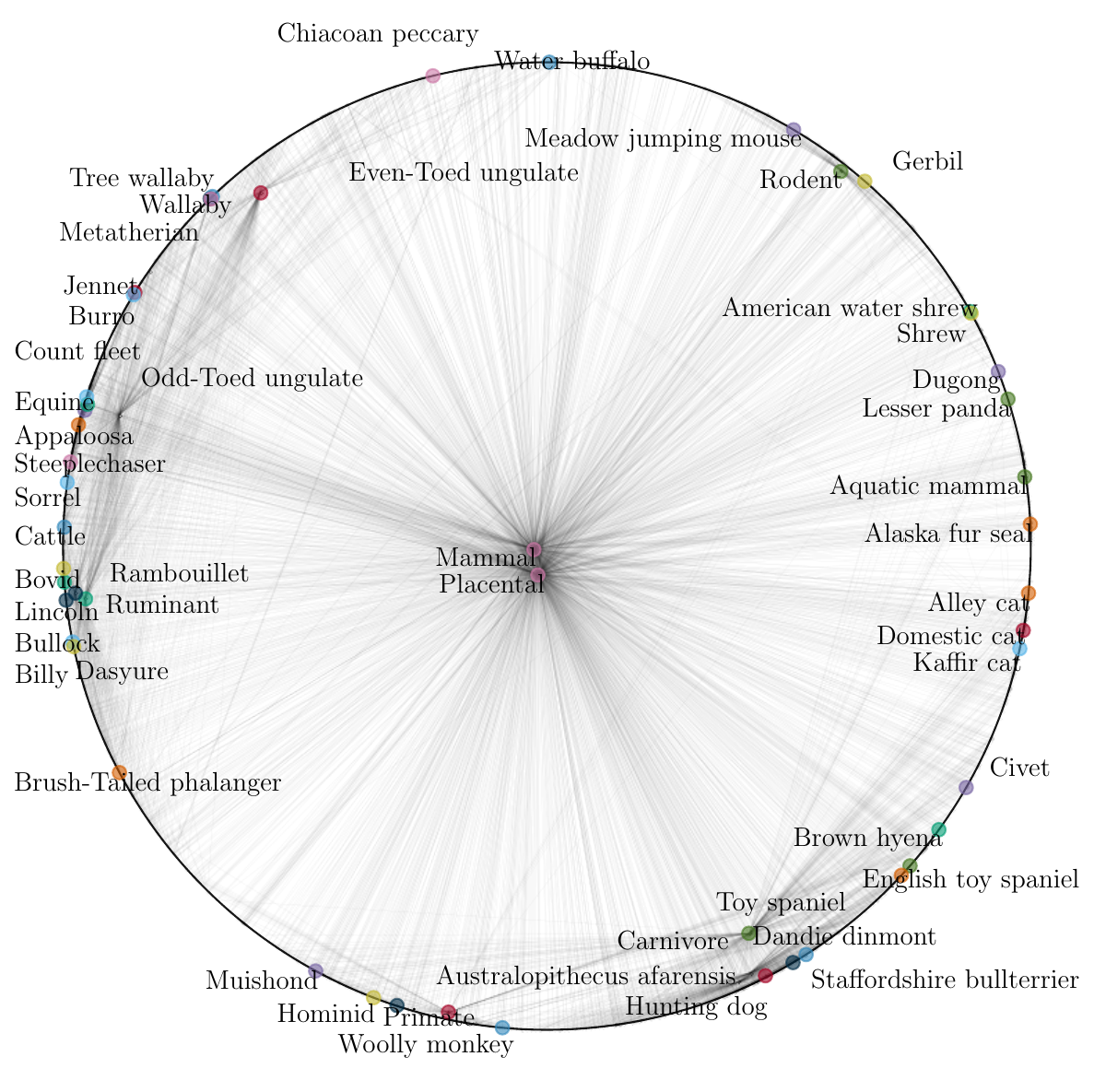} & \includegraphics[width=0.30\textwidth]{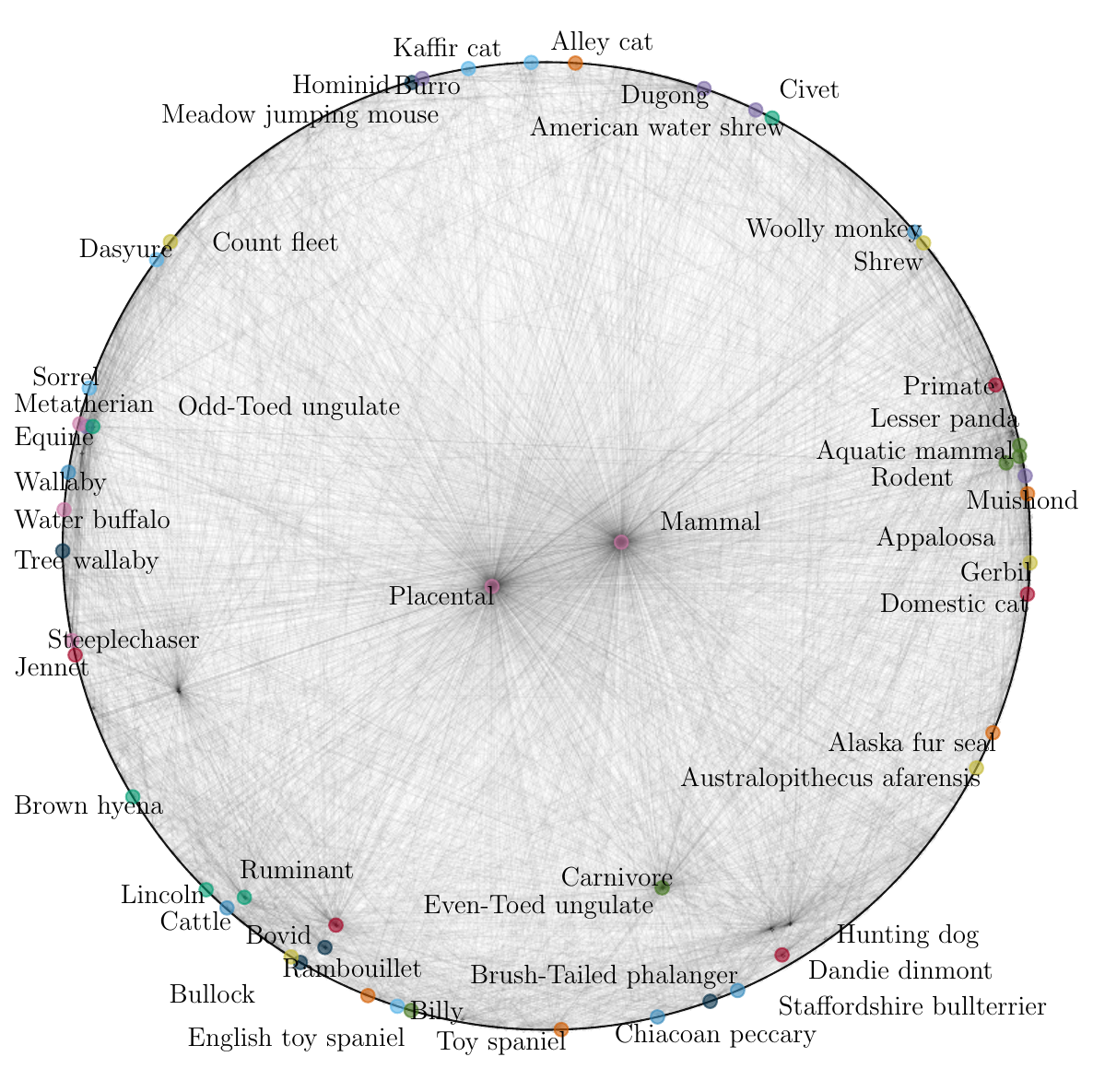} \\
  \end{tabular}
  
  \caption{\textbf{Supplementary results for the two-dimensional Poincar\'{e} word embeddings (\cref{experiments:poincare}).} We compute the mean average precision of the embeddings against the ground truth after 2000 training epochs. The reported results represent the average over five replications, with the dimension of the embeddings set to two. Plots of embeddings obtained under each optimizer are visualized and annotated for the first 50 nouns of the mammal's subtree. In the columns, ``with geometric curvature term" corresponds to learning schedulers for RDoG, RDoWG, and NRDoG that retain the geometric curvature term in the denominator, while ``without geometric curvature term" denotes the omission of this term.
  }
  \label{tab:poincare2D}
\end{table}

\end{document}